\documentclass[10pt]{article} 
\usepackage[left=1in,top=1in,right=1in,bottom=1in,letterpaper]{geometry}

\usepackage{multicol}
\usepackage{blindtext}
\usepackage{enumitem}
\usepackage{calc}
\usepackage{ifthen}
\usepackage{bbding}
\usepackage{amsmath,amsfonts, amsthm, amssymb}
\usepackage{color,graphicx,overpic}
\usepackage{hyperref}
\usepackage{bbm}

\usepackage[ansinew]{inputenc}
\usepackage{ upgreek }
\usepackage{graphicx} 
\usepackage{textcomp} 
\usepackage{amsfonts} 
\usepackage{amsmath}
\usepackage{amssymb}
\usepackage{yfonts}
\usepackage{mathtools}
\usepackage{mathrsfs}
\usepackage{latexsym}
\usepackage{array}
\usepackage{float}
\usepackage{calc}
\usepackage{perpage}
\usepackage{textcomp}
\usepackage{verbatim} 
\usepackage{hieroglf}
\usepackage{multirow} 
\usepackage{rotating} 
\usepackage{bm}
\usepackage{dsfont} 

\usepackage{url}
\usepackage{color}
\usepackage{tikz}
\usetikzlibrary{shapes,arrows,positioning,calc}
\usepackage{float}
\usepackage{soul}
\usepackage{enumitem}
\usepackage{booktabs}

\usepackage{xcolor}
\usepackage{algorithm}
\usepackage{algpseudocode}

\usepackage[normalem]{ulem}

\newcommand{\tri}[1]{{\left\vert\kern-0.25ex\left\vert\kern-0.25ex\left\vert #1 
    \right\vert\kern-0.25ex\right\vert\kern-0.25ex\right\vert}}

\newtheorem{theorem}{Theorem}[section]
\newtheorem{proposition}{Proposition}[]
\newtheorem{lemma}{Lemma}[section]
\newtheorem{corollary}{Corollary}[section]

\newtheorem{definition}{Definition}[section]
\newtheorem{assumption}{Assumption}[section]

\newcommand{\RN}[1]{%
  \textup{\uppercase\expandafter{\romannumeral#1}}%
}

\newcommand{\rnu}[1]{%
  \textup{\expandafter{\romannumeral#1}}%
}

\def\cD{\mathcal{D}}
\def\E{\mathbb{E}}

\def\cF{\mathcal{F}}

\def\cG{\mathcal{G}}
\def\cH{\mathcal{H}}

\def\cJ{\mathcal{J}}
\def\N{\mathbb{N}}
\def\cN{\mathcal{N}}

\def\R{\mathbb{R}}

\def\cV{\mathcal{V}}

\newcommand{\mpr}{\mathbb{P}}

\newcommand{\grad}{\nabla}

\makeatletter
\@addtoreset{equation}{section}
\makeatother

\hypersetup{
  colorlinks,
  linkcolor={red!50!black},
  citecolor={blue!50!black},
  urlcolor={blue!80!black}
}

\DeclarePairedDelimiter\ceil{\lceil}{\rceil}
\DeclarePairedDelimiter\floor{\lfloor}{\rfloor}
\DeclarePairedDelimiter\abs{\lvert}{\rvert}%
\DeclarePairedDelimiter\norm{\lVert}{\rVert}%

\newcommand*{\inner}[2]{\left  \langle #1,  #2 \right \rangle}

\newcommand*{\rom}[1]{\expandafter\@slowromancap\romannumeral #1@}
\def\Hek{H_{e_k}}

\newcommand*{\indic}[1]{\mathbbm{1}_{  #1  }   }

\newcommand*{\tm}{\vert_{\text{\scriptsize top($M$)}}}

\newcommand*{\pJ}{\vert_{\text{${ \cJ}$}}}
\newcommand*{\cJJ}{\vert_{\text{${ \scriptstyle \cJ \times \cJ }$}}}
\newcommand*{\cJcJ}{\vert_{\text{${ \scriptstyle \cJ^c \times \cJ }$}}}

\newcommand*{\wJ}{ \bs{w}_{\text{${ \scriptstyle \cJ}$}}}
\newcommand*{\sJ}{ \bs{s}_{\text{${ \scriptstyle \cJ}$}}}

\newcommand*\te{\tilde{\bs{e}}}
\newcommand*\supp{\text{supp}}
\newcommand\pspace{ \substack{ \bs{w} \in S_M^{d-1} \\ b \in \R }}
\newcommand\gt{\sigma^*}
\newcommand\tgt{\tilde{\sigma}^*}

\def\cgp{C_{\gt}}
\def\hmu{\hat{\bm{\mu}}}
\def\az{\bs{a}^{(0)}}
\def\azj{a^{(0)}_j}
\def\azp{a^{(0)}}
\def\ao{\bs{a}^{(1)}}

\def\bz{\bs{b}^{(0)}}
\def\bo{\bs{b}^{(1)}}
\def\bzj{b^{(0)}_j}
\def\bzl{b^{(0)}_l}
\def\bzp{b^{(0)}}
\def\boj{b^{(1)}_j}
\def\bop{b^{(1)}}
\def\Wz{\bs{W}^{(0)}}
\def\Wo{\bs{W}^{(1)}}
\newcommand*{\ngt}{ N^{\text{${  \tau}$}}}

\newcommand*{\ide}[1]{\bs{I}_{#1}}
\DeclareRobustCommand{\vect}[1]{\bm{#1}}
\newcommand\phiset{ \{ t,  \text{ReLU}(t) \}}
\newcommand\rhc{ \tilde{\gamma}}  
\newcommand\mrhc{\overline{\gamma}} 
\newcommand\ghc{\gamma} 
\newcommand\shc{\beta}  
\newcommand\bs[1]{\boldsymbol{#1}}  

\newcounter{relctr} 
\everydisplay\expandafter{\the\everydisplay\setcounter{relctr}{0}} 

\newcommand\labelrel[2]{%
  \begingroup
    \refstepcounter{relctr}%
    \stackrel{\tiny{(\alph{relctr})}}{\mathstrut{#1}}%
    \originallabel{#2}%
  \endgroup
}
\AtBeginDocument{\let\originallabel\label} 

\title{Pruning is Optimal for Learning Sparse Features  in High-Dimensions}

\author{
Nuri Mert Vural\thanks{Department of Computer Science at University of Toronto, and Vector Institute.  \texttt{vural@cs.toronto.edu}.}
\ \ \ \ \ \  
Murat A. Erdogdu\thanks{Department of Computer Science, and of Statistical Sciences at University of Toronto, and Vector Institute. \texttt{erdogdu@cs.toronto.edu}.}
}

\mathtoolsset{showonlyrefs}

\newcommand{\eq}[1]{\begin{align}#1\end{align}}

\begin{document}
\maketitle

\begin{abstract}%
While it is commonly observed in practice that pruning networks to a certain level of sparsity can improve the quality of the features, a theoretical explanation of this phenomenon remains elusive.   In this work, we investigate this by demonstrating that a broad class of statistical models  can be optimally learned using pruned neural networks trained with gradient descent, in high-dimensions. 

We consider learning  both single-index and multi-index models of the form $y = \sigma^*(\boldsymbol{V}^{\top} \boldsymbol{x}) + \epsilon$,  where $\sigma^*$  is a degree-$p$ polynomial,  and $\boldsymbol{V} \in \mathbbm{R}^{d \times r}$ with $r \ll d$, is the matrix containing relevant model directions.  We assume that $\boldsymbol{V}$ satisfies a certain $\ell_q$-sparsity condition for matrices and show that pruning neural networks proportional to the sparsity level of $\boldsymbol{V}$  improves their sample complexity compared to unpruned networks.   Furthermore, we establish  Correlational Statistical Query (CSQ) lower bounds in this setting, which take the sparsity level of $\boldsymbol{V}$ into account.   We show that if the sparsity level of $\boldsymbol{V}$ exceeds a certain threshold, training pruned networks with a gradient descent algorithm achieves  the sample complexity suggested by the CSQ lower bound.  In the same scenario, however,  our results imply that basis-independent methods such as models trained via standard gradient descent initialized with rotationally invariant random weights can provably achieve only suboptimal sample complexity.
\end{abstract}


\section{ Introduction}

Neural network pruning, 
a technique aimed at reducing the number of weights by selectively removing certain connections or neurons, 
has attracted significant attention in recent years as a means to improve efficiency and scalability in deep learning~\cite{LeCun1989OptimalBD, Hassibi1992SecondOD, Han2015LearningBW, Frankle2018TheLT}. 
Beyond the computational advantages offered by pruning, 
empirical observations demonstrate that this method can also substantially improve the generalization performance of neural networks~\cite{Bartoldson2019TheGT, Jin2022PruningsEO}.

Deep learning has challenged the classical learning theory and demonstrated that overparameterization will oftentimes improve generalization.
In stark contrast, however,
pruning overparametrized networks is also known to improve generalization, as observed in many empirical studies~\cite{LeCun1989OptimalBD, Hassibi1992SecondOD, Bartoldson2019TheGT, Jin2022PruningsEO}.
In this context, our understanding of the effect of pruning remains elusive. As such, we focus on the following question:



\begin{itemize}[leftmargin=0.3in,rightmargin=0.3in]
\setlength\itemsep{0.05em}
\item []\ \hspace{.5in} \emph{Does pruning improve the quality of trained features in neural networks?}
\end{itemize}
We answer this question in the affirmative. Indeed, we show that when the statistical model satisfies a certain sparsity condition,
pruned neural networks trained with gradient descent can achieve optimal sample complexity, and learn significantly more efficiently compared to unpruned networks.

Feature learning in neural networks has been the focus of many recent works.
A key characteristic in these models is their ability to learn low-dimensional latent features~\cite{Yehudai2019OnTP,Ghorbani2020WhenDN,MousaviHosseini2022NeuralNE}.   An apt scenario for studying this capability is the task of learning multi-index models~\cite{Damian2022NeuralNC,MousaviHosseini2022NeuralNE},  where the response $y \in \R$ depends on the input $\bs{x} \in \R^d$ via the relationship  $y = \gt(\bs{V}^\top \bs{x}) + \epsilon$. Here, $\gt: \R^r \to \R$ is the non-linear link function, and the matrix $\bs{V} \in \R^{d \times r}$ contains the relevant \emph{model directions}. Our main focus is the regime where there are few relevant directions when compared to the ambient input dimension, i.e. $r \ll d$. In the special case $r = 1$, this model also covers the single-index setting, which has been studied extensively; see e.g.~\cite{Ba2022HighdimensionalAO,MousaviHosseini2022NeuralNE,arous2021online,damian2023smoothing} and the references therein.
In the simplified single-index case, the sample complexity of
learning the model direction 
is determined by the \emph{information exponent} $k^\star$
of the link function $\gt$,
which is defined as
the smallest order nonzero Hermite coefficient of $\gt$. 
\cite{arous2021online} 
proved that SGD learns the direction in $n \geq O(d^{1 \vee k^\star -1})$ samples, which is also tight for this algorithm.
This, however, does not meet the corresponding Correlational Statistical Query (CSQ) lower
bound in this setting which, roughly states that $n\geq \Omega(d^{k^\star/2})$
samples are necessary.  Recently, \cite{damian2023smoothing}
showed that smoothing the
loss landscape can close this gap and attain the CSQ lower bound. 

It is important to highlight that the aforementioned studies consider single- or multi-index settings in their full generality,  without any structural assumptions on the model directions. In practice,  however, high-dimensional data often exhibits low-dimensional structures; thus, sparsity is a natural property to consider. It is reasonable to expect that with this additional structure, the corresponding CSQ lower bound 
would become smaller. However, it remains unclear whether the previously considered training methods can still achieve this lower bound in the sparse setting.

In this paper, 
we introduce the concept of \emph{soft sparsity} for the model directions 
$\bs{V}$ and derive a CSQ lower bound that depends on this sparsity level, which is always smaller than the lower bound in the general multi-index setting that only considers the worst-case sparsity scenario.  Next, we demonstrate that pruned neural networks trained with a gradient-based method can achieve the optimal sample complexity suggested by this CSQ lower bound. Since the additional sparsity structure reduces the
lower bound,  basis-independent training methods such as gradient descent initialized with a symmetric distribution have provably suboptimal sample complexity;  this implies a separation between pruning-based and existing training methods.
We summarize our contributions below.
\begin{itemize}[leftmargin=.2in,itemsep=.001in]
\item We consider learning multi-index models of the form $y = \gt(\bs{V}^\top \bs{x}) + \epsilon$ where the model directions $\bs{V}\in\mathbb{R}^{d\times r}$ satisfy a certain soft sparsity. In Theorem~\ref{thm:csqlbtext}, we prove a Correlational Statistical Query (CSQ) lower bound for this model, which also takes the inherent sparsity into account. The lower bound depends only on   the sparsity level beyond a certain threshold. In this regime, our result shows that basis-independent training methods are always
suboptimal.

\item In the single-index case where $r=1$,
we prove that
pruning the neural network with a sparsity level proportional to that of the model direction leads to a better sample complexity after training. Specifically, we consider polynomial link functions  and show in Theorem~\ref{thm:single} that the sample complexity achieved after pruning is optimal in the sense that, training after pruning can achieve the complexity suggested by the CSQ lower bound for any  information exponent $k^\star\geq 1$.

\item  Finally, we consider the multi-index case with $r>1$. Under an additional assumption implying that the information exponent is $k^\star =2$,
we prove in Theorem~\ref{thm:multi} that, pruned network trained with gradient descent can achieve the corresponding CSQ lower bound in this setting as well.
\end{itemize}

\vspace{-.1in}
\subsection{Related Work}
\vspace{-.05in}
\textbf{Pruning and generalization. }   Pruning techniques have a rich history, spanning from classical methods that prune weights based on connectivity metrics like the Jacobian/Hessian \cite{LeCun1989OptimalBD, Hassibi1992SecondOD}, to more recent approaches relying on weight magnitude \cite{Han2015LearningBW, Wen2016LearningSS, Molchanov2016PruningCN}. Notably, iterative magnitude pruning, proposed by \cite{Han2015LearningBW} demonstrated remarkable success in deep neural networks, sparking a surge in pruning research~\cite{Zhu2017ToPO, Frankle2019LinearMC, Gale2019TheSO, Liu2018RethinkingTV}. 

Numerous studies demonstrate the beneficial effects of pruning on generalization \cite{LeCun1989OptimalBD, Frankle2018TheLT,Barsbey2021HeavyTI}. Prior research treats pruning as an additional regularization technique, which requires weights to exhibit small norm~\cite{Giles1994PruningRN}, achieve flat minima \cite{Bartoldson2019TheGT}, or enhance robustness to outliers \cite{Jin2022PruningsEO}. However, these studies are predominantly empirical and lack a theoretical foundation. Among the theoretical works, only \cite{yang2023theoretical} examines random pruning within a specific statistical model. Our work extends their framework to encompass general polynomial link functions and data-dependent pruning algorithms, complementing generalization bounds with guarantees of optimality.

\noindent
\textbf{Lottery tickets and sparsity.}  Recent  work has observed that overparameterized neural networks contain subsets, referred to as ``winning tickets'', which can achieve comparable performance to the original network when trained independently \cite{Frankle2018TheLT}. This phenomenon, known as the Lottery Ticket Hypothesis (LTH), has been extensively studied in the literature \cite{Frankle2019LinearMC, Gale2019TheSO, Chen2020TheLT, Zhou2019DeconstructingLT}. Several recent works have focused on investigating the theoretical conditions for the existence of such subnetworks \cite{Malach2020ProvingTL, Laurent2020} and the fundamental limitations of identifying them \cite{Kumar2024NoFP}.   Our study takes a different approach by examining the training dynamics and generalization within the context of pruning. While previous works primarily focus on identifying subnetworks as predicted by the LTH, our research delves into the interplay between generalization and pruning methods.

\noindent
\textbf{Non-linear feature learning with neural networks. } Recent theoretical studies have examined two scaling regimes in neural networks. In the ``lazy'' regime \cite{Chizat2018OnLT}, parameters remain largely unchanged from initialization, resembling kernel methods \cite{Jacot2018NeuralTK,Du2018GradientDP,AllenZhu2018ACT,Oymak2019TowardMO}. 
However, deep learning's superiority over kernel models suggests they can go beyond this regime \cite{Yehudai2019OnTP,Ghorbani2020WhenDN, Geiger2019DisentanglingFA}. In contrast, the ``mean-field'' regime, where gradient descent converges to Wasserstein gradient flow, enables feature learning \cite{Chizat2018OnLT,Mei2019MeanfieldTO,Chizat2022MeanFieldLD}, but primarily applies to infinitely wide networks. Our paper explores a different setting, allowing for arbitrary-width neural networks without excessive overparameterization, while still employing mean-field scaling for weight initialization.

\noindent
\textbf{Feature learning with multiple-index teacher models.}  Learning an unknown low-dimensional function from data is fundamental in statistics \cite{Li1989RegressionAU}. Recent research in learning theory has considered this problem, aiming to demonstrate that neural networks can learn useful feature representations and outperform kernel methods \cite{Ghorbani2020WhenDN, Damian2022NeuralNC, Abbe2023SGDLO}.  In particular, \cite{Abbe2022TheMP} investigates the necessary and sufficient conditions for learning with linear sample complexity in the mean-field limit, focusing on inputs confined to the hypercube.  
Closer to our setting are the recent works \cite{Damian2022NeuralNC, MousaviHosseini2022NeuralNE} which demonstrate a clear separation between NNs and kernel methods, leveraging the effect of representation learning.  More recently,  \cite{Dandi2024TheBO} shows that   mini-batch SGD with finite number steps can learn a certain class of link functions with linear sample complexity.
Our work operates within a similar framework, incorporating an additional sparsity condition on relevant model directions. However, our analysis differs from previous work in two main aspects. First, our pruning results are constructive; we develop an explicit algorithm to establish the sample complexity of the pruned network trained via gradient descent. Second, pruning introduces a new dependency between weights and data, requiring an intricate analysis of gradient descent dynamics.

\section{Preliminaries}
\textbf{Notations.} Let $[n] \coloneqq \{ 1, \cdots, n\}$.   We use $\inner{\cdot}{\cdot}$ and  $\norm{\cdot}_2$ to denote the Euclidean inner product and the norm, respectively. For matrices, $\norm{\cdot}_2$ denotes the usual operator norm.  For a matrix $\bs{A} \in \R^{m \times n}$,  $\bs{A}_{i*}$ and $\bs{A}_{*j}$ 
denote the $i$th row and $j$th column of $\bs{A}$, respectively.  $S^{d-1}$ is the $d$-dimensional unit 
 sphere.  We use $\{ \bs{e}_1, \cdots, \bs{e}_d \}$ to denote the standard basis vectors in $\R^d$.  We use $O(\cdot)$ and $\Omega(\cdot)$ to suppress constants in upper and lower bounds.  We use $\tilde O(\cdot)$  to suppress poly-logarithmic terms in $d$ in upper bounds.  We use $o_d(\cdot)$ to denote vanishing  terms as $d \to \infty$.   We use $f \in \Theta( g )$ to  denote  $\Omega \left( g \right) \leq f  \leq O \left(g \right)$.
For a vector $\bs{x} \in \R^d$, we use $\supp(\bs{x}) \coloneqq \{ i \in [d] ~ : ~ x_i \neq 0 \}.$  For a subset $\cJ \subseteq [d],$  we use $\bs{x} \vert_{\cJ} \in \R^d$ to denote the restriction of the vector $\bs{x}$ on $\cJ$, i.e., the coordinate indices that are not in $\cJ$ are set to be $0$.  For matrices,  $\bs{A} \vert_{\cJ}$ denotes the matrix $\bs{A}$ with everything but the rows indexed by the elements in $\cJ$ set to $0.$  Finally,  $\bs{x} \tm$ denote the vector $\bs{x}$ with everything except $M$ largest entries in magnitude set to $0$.

\vspace{.1in}
\noindent\textbf{Statistical model.}
For 
a link function $\gt: \R^r \to \R$,  we consider the multi-index model 
\eq{
y = \gt(\bs{V}^\top \bs{x}) + \epsilon \quad \text{ with } \quad \bs{x} \sim \cN(0, \ide{d}) \label{eq:learningproblem}
}
where $\bs{x}\in\R^d$ is the input, $\epsilon$ is a zero-mean noise with $O(1)$ sub-Gaussian norm and $\bs{V} \in \R^{d \times r}$ is an orthonormal matrix, i.e,  $\bs{V}^\top \bs{V} = \ide{r}$.  We assume that $\gt$ is a polynomial of degree $p$, and it is normalized to satisfy $\E_{\bs{z} \sim \cN(0,\ide{r})} [\gt(\bs{z})] = 0$ and $\E_{\bs{z} \sim \cN(0, \ide{r})} \left[ \gt(\bs{z})^2 \right] = 1.$ 
We consider the low-dimensional setting $r\ll d$ which, in the extreme case $r=1$, covers single-index models. We are mainly interested in models where $\bs{V}$ exhibits \emph{sparsity}; we use the following matrix norm:
\eq{
\norm{\bs{V}}_{2,q} \coloneqq \norm*{ \big(\norm{\bs{V}_{1*}}_2, \cdots,\norm{\bs{V}_{d*}}_2  \big)  }_q\quad \text{ where }\quad q \in [0, 2),
}
where $\bs{V}_{i*}$ denotes the $i$th row of $\bs{V}$.\footnote{To be precise, $\norm{\cdot}_{2,q}$ is not a norm when $q < 1$.}  
This is simply the usual $\ell_q$ norm of the vector with entries $\ell_2$ norm of rows of $\bs{V}$.  
Since $\bs{V}^\top\bs{V}= \ide{r}$, assuming that $\norm{\bs{V}}_{2,q}$ is small constrains the model complexity significantly.  Indeed,
when $q=0$, $\norm{\cdot}_{2,q}$ counts the number of non-zero rows, serving as a measure of sparsity in high-dimensional settings. In the case $q\in(0,2)$, small $\norm{\cdot}_{2,q}$ norm allows all rows to potentially contain non-zero values, provided their $\ell_2$ norms are all relatively small. When we have
$\norm{\bs{V}}_{2,q}^q \leq R_q $ for some $R_q$,
we adopt a terminology from \cite{raskutti2011minimax}  and refer to $R_q$ as the \emph{soft sparsity} level. Notably, the particular choice 
$\norm{\cdot}_{2,q}$ is motivated by its coordinate-independent property; that is, we have $\norm{\bs{V} \bs{U}}_{2,q} =\norm{\bs{V}}_{2,q}$ for any orthonormal matrix $\bs{U} \in \R^{r \times r}$.

\vspace{.1in}
\noindent \textbf{Two-layer Neural Networks.}
Denoting the ReLU activation with $\phi(t) = \max\{ t, 0\}$, 
we consider learning with two-layer neural networks of the form
\eq{
\hat y(\bs{x}; (\bs{a}, \bs{W},\bs{b})) = \sum_{j = 1}^{2m} a_j \phi(\inner{\bs{W}_{j*}}{\bs{x}} + b_j) = \inner{\bs{a}}{\phi \left( \bs{W} \bs{x} + \bs{b}  \right)},
}
where $\bs{W} = \{ \boldsymbol{W}_{j*} \}_{j = 1}^{2m}$ is the $2m \times d$ matrix whose rows are denoted with $\bs{W}_{j*}$, $\bs{a} = \{ a_j \}_{j = 1}^{2m}$ is the second layer weights, $\bs{b} = \{ b_j \}_{j = 1}^{2m}$ is the biases. Note that $\phi(\cdot)$ is applied element-wise in the second equality.   
We define the population and the empirical risks respectively as
\eq{
\!\!R((\bs{a}, \bs{W}\!,\bs{b})) \!=  \frac{1}{2} \E \left[ (\hat y(\bs{x};  (\bs{a}, \bs{W}\!,\bs{b})) - y)^2 \right], \ \
R_n((\bs{a}, \bs{W}\!,\bs{b})) \!  = \frac{1}{2 n} \sum_{i = 1}^n (\hat y(\bs{x}_i;  (\bs{a}, \bs{W}\!,\bs{b})) - y_i)^2\!\!\!\!
}
where the expectation above is over the data distribution. 

Our training procedure consists of  three-steps: $(i)$ we first prune the network for dimension reduction, then $(ii)$ we take a gradient descent iteration with a large step-size to train $\bs{W}$, and finally $(iii)$ we train the second layer weights $\bs{a}$. We will provide the details of the algorithm, in particular the pruning step in Section \ref{sec:training}.
Similar to the previous works, e.g. \cite{Chizat2018OnLT,Damian2022NeuralNC,Dandi2023HowTN}, we use symmetric initialization so that $\hat y (\bs{x}, (\az, \Wz, \bz)) = 0$;
we assume that the network has a width of $2m$ such that
\eq{
\azj = - \azp_{2 m -j}, \quad   \Wz_{j*} =  \Wz_{(2m-j)*} \in {S}^{d-1}, \quad  \bzj = \bzp_{2m-j}, \quad \text{for }j \in [m]. \label{eq:syminit}
}
Particularly,  we will use the following initialization for the second-layer weights and the biases,
\eq{
\azj \sim \text{Unif}\{-1,1 \}, ~~ \text{and} ~~ \bzj \sim \cN(0,1),\  ~ j \in [m]. \label{eq:syminitdist}
}
Initialization of $\Wz$ will depend on the pruning algorithm and be detailed later.
Note that due to \eqref{eq:syminit}, the gradient of $R_n$ with respect to $\bs{W}_{j*}$ at initialization can be written as follows:
\eq{
\grad_{\bs{W}_{j*}} R_n ( (\bs{a}, \bs{W},\bs{b}) ) =  \frac{- a_j}{n} \sum_{i= 1}^n  y_i \bs{x}_i \phi^\prime \left(\inner{\bs{W}_{j*}}{\bs{x}_i} + b_j \right).
}
We simplify the notation to $\grad_{j} R_n( (\bs{a}, \bs{w}, \bs{b}) )$ whenever $\bs{W}_{i*} = \bs{w}$ for all $i$.

Characteristics of the link function $\gt$ plays an important role in the complexity of learning.  Indeed,   recent works showed that the term in the Hermite expansion of $\gt$ with the smallest degree determines the sample complexity \cite{arous2021online,Abbe2023SGDLO}.  In line of these works,  we also rely on Hermite expansions, for which we define the Hermite polynomials as follows.
\begin{definition}[Hermite Polynomials]
\label{def:hermite}
The $k$th Hermite polynomial $H_{e_k} : \R \to \R$ is the degree $k$  polynomial defined by 

\vspace{-.3in}
\eq{
H_{e_k}(t) = (-1)^k e^{t^2/2} \frac{d^k}{dt^k}  e^{- t^2/2}.
}
\end{definition}

\section{Limitations of Basis Independent Methods: CSQ Lower Bounds}
 \label{sec:csqbound}

In this section, we explore the fundamental barriers under the \emph{soft sparsity} structure we assume on the statistical model.
Specifically, we  establish a lower bound  for Correlational Statistical Query (CSQ) methods within our framework.  We note that the CSQ methods encompasses  a wide class of algorithms under the squared error loss.   We consider the function class
\begin{equation}
\!\!\!\cF_{r,k} \coloneqq \left \lbrace \bs{x} \to \frac{1}{\sqrt{r k!}} \sum_{j = 1}^r H_{e_k}(\inner{\bs{V}_{*j}}{\bs{x}}) ~ \Big \vert ~ \bs{V} \in \R^{d \times r},  ~ \bs{V}^\top \bs{V}= \ide{r},  ~ \norm{\bs{V}}_{2,q}^q \leq  r^{\frac{q}{2}} d^{\alpha \left(1 - \frac{q}{2} \right)}  \right \rbrace\!  \label{eq:frk}
\end{equation}
where  $\alpha \in (0,1)$,  $H_{e_k}$ denotes the $k$th Hermite polynomial (see Definition \ref{def:hermite}), and for $q=0$, we use the convention $\norm{\bs{V}}_{2,0}^0 \coloneqq \norm{\bs{V}}_{2,0}$.
We remark that the constraint $\bs{V}^\top \bs{V}= \ide{r}$ directly implies $r \leq \norm{\bs{V}}_{2,q}^q \leq r^{q/2} d^{1 - q/2}$.  Therefore,  $ \cF_{r,k}$ covers  all possible sparsity levels by varying the parameter $\alpha$.   We have the following result on the query complexity of CSQ methods.

\begin{theorem} 
\label{thm:csqlbtext}
Consider $\cF_{r,k}$ with some $q \in [0,2)$ and $\alpha \in (0,1)$.  For a sufficiently large $d$ depending on $(r,k,q,\alpha)$,  any CSQ algorithm for $\cF_{r,k}$ that guarantees error $\varepsilon = \Omega(1)$ 
requires either queries of accuracy $\tau = \widetilde O \big(  d^{- \left(\alpha \wedge \frac{1}{2} \right)  \frac{k}{2}} \big)$ or super-polynomially many queries in $d$.
\end{theorem}
\noindent
Using the heuristic $\tau \approx \tfrac{1}{\sqrt{n}}$ as in~\cite{Damian2022NeuralNC}, Theorem \ref{thm:csqlbtext} implies that $n \geq \Omega \big( d^{\left(\alpha \wedge \frac{1}{2} \right)  k }  \big)$ samples are necessary to learn a function in $\cF_{r,k}$ unless the algorithm makes super-polynomial queries in $d$.  This recovers the existing lower bound $\Omega \big( d^{k/2} \big)$ given in \cite{Damian2022NeuralNC,Abbe2023SGDLO}, when the constraint is sufficiently large, i.e.,  $\alpha > \tfrac{1}{2}$.  Conversely,  when the soft sparsity level is sufficiently small, i.e.,   $\alpha \leq \tfrac{1}{2}$, we observe that the complexity lower bound reads $\Omega \big( d^{\alpha k }   \big)$. 
Remarkably, in Section \ref{sec:main}, we prove that a pruned neural network trained with gradient descent can indeed attain this lower bound; thus, it achieves optimal sample complexity in this sense. 

We note that $\norm{\bs{V}}_{2,q}^q$ can be as small as $r$; thus, the CSQ lower bound in this regime can be significantly smaller than  the unconstrained version $\Omega \big( d^{k/2} \big)$.  On the other hand, methods that are independent of the underlying basis, such as gradient descent with symmetric initialization,  cannot exploit the additional structure.  As a result, these methods are constrained by the sample complexity lower bound of $\Omega \big( d^{k/2} \big)$ in the worst case. 
Finally, it is worth emphasizing that CSQ lower bounds do not directly apply to algorithms like SGD or one-step gradient descent due to non-adversarial noise. Nevertheless, under the square loss, queries of these algorithms fall under the correlational regime, thus the fundamental barrier CSQ lower bounds provide is frequently referred to when assessing the optimality of these methods; see e.g.~\cite{Damian2022NeuralNC,damian2023smoothing,Abbe2023SGDLO}.

\section{Training Procedure: Pruning as Dimension Reduction}
\label{sec:training}
In this section, we outline the pruning procedure and how it effectively reduces the dimensionality of the learning problem, leading to the optimal sample complexity suggested by Theorem~\ref{thm:csqlbtext}. 

\vspace{.1in}
\noindent \textbf{Intuition.}
To gain intuition, we start with the population dynamics and
 consider a simplified single-index setting to demonstrate the resulting dimension reduction. Let
\eq{
\gt(\inner{\bs{v}}{\bs{x}}) = H_{e_2}(\inner{\bs{v}}{\bs{x}}) ~~\text{with}~~ \bs{v} = \left(d^{-\frac{1}{4}}, \cdots, d^{-\frac{1}{4}}, 0, 0, \cdots, 0 \right), \label{eq:examplesetting}
}
 where the direction $\bs{v}$ is sparse, i.e. $\norm{\bs{v}}_0  = \sqrt{d} \ll d$.
Moreover,  for clarity,  let us fix  the output layer weights to $\azj = 1$ and biases to $\bzj = 0$ and consider the population gradient at initialization.  To see why comparing gradients performs dimension reduction,  we write
\eq{
&\grad_{j}  R((\az, \bs{e}_i, \bz)) = - \E \left[ \gt(\inner{\bs{v}}{\bs{x}}) \phi^\prime(\inner{\bs{e}_i}{\bs{x}}) \bs{x}  \right] = - \sqrt{ \tfrac{2}{\pi} } \inner{\bs{v}}{\bs{e}_i} \bs{v} + \tfrac{1}{\sqrt{2 \pi }} \inner{\bs{v}}{\bs{e}_i}^2 \bs{e}_i ~~~~~  \label{eq:examplegrad}
}
where $\bs{e}_i$ is the $i$th standard basis and constants are due to the Hermite coefficients of the ReLU activation $\phi(\cdot)$. Thus, we have
\eq{
\norm{ \grad_{j}  R((\az, \bs{e}_i, \bz)) }_2^2 = \tfrac{2}{\pi} \bs{v}_i^2 + O(d^{-1}). \label{eq:examplenorm}
}
Since the entries of $\bs{V}$ scale with $d^{-1/4}$ in high dimensions, comparing the norm of gradients is equivalent to comparing the magnitude of each entry $\bs{v}_i$.  Hence, non-zero coordinates of $\bs{V}$ can be picked up by pruning, which is effectively reducing the dimension of the problem from $d$ to the sparsity level $\sqrt{d}$ in this example.

Algorithm \ref{alg:pruningalg} essentially extends the basic intuition above to general link functions $\gt$ and empirical gradients.  However, such an extension requires us to handle two technical difficulties due to the bias in the Hermite expansion of the population gradient. In Section~\ref{sec:technical},  we illustrate how each step in Algorithm  \ref{alg:pruningalg} is designed to avoid those difficulties using the following arguments:
\begin{itemize} 
\item (Data augmentation) We augment the feature vectors with an independent non-informative random variable, i.e.,  $\bs{x}^{\prime} \leftarrow (\bs{x},z)^T$ where $z \sim \cN(0,1)$ and independent of $\bs{x}$.  For notational convenience,  we assume that the augmented features $\bs{x}^\prime$ (henceforth referred to as $\bs{x}$) is $d$-dimensional.  Since the last entry of the feature vector is non-informative, we can assume $\bs{V}_{d*} = 0$,  without loss of generality.
\item  (Shifted standard basis) 
We compare the magnitudes of the gradients initialized  at
\eq{\label{eq:shifted-basis}
    \te_j \coloneqq \begin{cases}
    c \bs{e}_j + \sqrt{1 - c^2} \bs{e}_d,  & j \in [d-1] \\
    \bs{e}_d & j = d.
    \end{cases}
    }
Here, standard basis vectors are \emph{shifted} by a factor of $c \in (0,1)$ to make sure that the extra terms vanish (see Line \ref{algprune:init}  in Algorithm \ref{alg:pruningalg}).  
\item (Even-odd decomposition) We consider the even and odd components of the 
activation separately, i.e., $\phi_{\pm}(t; b) =  (\phi(t + b) \pm \phi(-t + b))/2$,  and evaluate the gradient with these components (Line \ref{algprune:eval}  in Algorithm \ref{alg:pruningalg})  
\eq{\label{eq:grad-even-activation}
\grad_{j}  R_n^{\pm}((\az, \te_i, \bz)) & \coloneqq \frac{1}{2} \left[ \grad_{j} R_n((\az, \te_i, \bz)  \pm   \grad_{j} R_n((\az, -  \te_i, \bz) \right].
}
\end{itemize}

\begin{algorithm}[t]
\caption{PruneNetwork} \label{alg:pruningalg}
\textbf{Inputs:} 
$~(\rnu{1})$    Data: $\mathcal{D} \coloneqq \{ (\bs{x}_i, y_i) \}_{i = 1}^n $ 
$~(\rnu{2})$   Network width:\textsuperscript{\ref{footnote1}}  $m \in \N$
$~(\rnu{3})$  Sparsity level: $M \in [d]$ 
$~(\rnu{4})$    Shrinkage constant: $c \in (0,1)$ 
\begin{algorithmic}[1]
    \State Let $\te_i$ be as in \eqref{eq:shifted-basis},  and initialize $\az$ and $\bz$ as in  \eqref{eq:syminit}-\eqref{eq:syminitdist} \label{algprune:init}
    \State  Let $\widetilde   \grad_j R_n^{\pm}(\te_i) \coloneqq   \grad_j R_n^{\pm} \big( ( \az , \te_i,   \bz) \big) \tm$ and $\norm{ \widetilde   \grad R_n^{\pm}(\te_i)}_F^2 = \sum_{j = 1}^{2m} \norm{\widetilde   \grad_j R_n^{\pm}(\te_i) }_2^2$
    \label{algprune:eval} 
    \State $\cJ =\text{supp}( \widetilde   \grad_{j} R_n^{-}(\te_i) )$ for some $j \in [m]$ with $\bzp_j \geq 0$ if one exists,  otherwise $\cJ = \emptyset$.   \label{algprune:firstHermite}
    \State Sort  $ \norm{\widetilde \grad R^{+}_n(\te_{j_1} )}_2 \geq \norm{\widetilde \grad R^{+}_n(\te_{j_2})}_2  \geq \cdots \geq \norm{\widetilde \grad R^{+}_n(\te_{j_d})}_2$ and  $\cJ \leftarrow \cJ  \cup \{ j_1, \cdots, j_M\}$ \label{algprune:prunepos}
    \State Sort  $\norm{\widetilde \grad R^{-}_n(\te_{k_1} )}_2 \geq \norm{\widetilde \grad R^{-}_n(\te_{k_2})}_2  \geq \cdots \geq \norm{\widetilde \grad R^{-}_n(\te_{k_d})}_2$ and $\cJ  \leftarrow \cJ  \cup \{ k_1, \cdots, k_M\}$ \label{algprune:pruneneg}
     \State \textbf{Return:} $\mathcal{J}$
\end{algorithmic}
\end{algorithm}

\vspace{.1in}
\noindent\textbf{Pruning Algorithm~\ref{alg:pruningalg}.}  
The pruning algorithm is based on comparing gradient magnitudes at initialization to perform dimension reduction.  The challenge lies in utilizing empirical gradients. To estimate the gradient magnitudes, we consider pruned empirical gradients  , i.e.,   $\widetilde   \grad_j R_n^{\pm}(\te_i) \coloneqq   \grad_j R_n^{\pm} \big( (\az , \te_i,   \bz) \big) \tm$ (Line \ref{algprune:eval}).  Improving on the sample mean estimator, which requires $O(d)$ samples,  pruned sample mean requires sample complexity of $\tilde O(d^\alpha)$  by leveraging the sparsity of population gradient,  hence providing the desired sample complexity for the algorithm.

Having computed the empirical gradients, we proceed by evaluating and sorting the gradients (Lines \ref{algprune:prunepos} and \ref{algprune:pruneneg}). We keep the connections with larger gradient magnitude while pruning the remaining small entries.

\begin{algorithm}[t]
\caption{Gradient-based Training} \label{alg:onestepgd}
\textbf{Inputs:}
$~(\rnu{1})$    Data: $\cD \coloneqq  \{ (\bs{x}_i, y_i) \}_{i = 1}^n$ 
$~(\rnu{2})$ Learning rate: $\eta_t > 0$
$~(\rnu{3})$  Weight Decay: $ \lambda_t > 0$ \\
$~(\rnu{4})$   Network width:\textsuperscript{\ref{footnote1}}  $m \in \N$ 
$~(\rnu{5})$ Pruning Level: $M \in [d]$
$~(\rnu{6})$  Shrinkage constant: $c \in (0,1)$
\begin{algorithmic}[1]
    \State \label{alggd:pruningstep} $\cJ \leftarrow \texttt{PruneNetwork}(\cD,m,  M, c)$
    \State \label{alggd:re-init}Re-initialize $\az$ and $\bz$ as in as in  \eqref{eq:syminit}-\eqref{eq:syminitdist}, and
    \eq{
		\Wz_{j*}  \sim S^{d-1}_\mathcal{J}, ~~ \text{and} ~~     \Wz_{j*}  =  \Wz_{(2m- j + 1)*}, ~ j \in [m].
    }
    \State \label{alggd:gradientstep} Train the first layer weights: For $j\in [2m]$
    \eq{
    \Wo_{j*} = \Wz_{j*} - \eta_1 \Big(  \grad_{\bs{W}_{j*}}  R_n \left( (\az,  \Wz_{j*}, \bz)  \right) \big \vert_\mathcal{J} + \lambda_1 \Wz_{j*}  \Big).
    }
    \State \label{alggd:reinit} Re-initialize biases: For $j \in [m]$, let $\bop_j \sim \cN(0,1)$   and $\boj = \bop_{2m-j+1}$.
    \State \label{alggd:secondlayer} Train the second layer weights: 
    \eq{
    \bs{a}^{(t+1)} = \bs{a}^{(t)} - \eta_t \left(  \grad_{\bs{a}} R_n(( \bs{a}^{(t)}, \Wo, \bo) ) + \lambda_t  \bs{a}^{(t)} \right), ~ t \geq 2.
    }
    \State \textbf{Return:} $\hat y(\bs{x}; (\bs{a}^{(T)},\Wo,\bo)) =  \inner{\bs{a}^{(T)} }{\phi ( \Wo \bs{x} +\bo)}$ \label{alggd:prediction}
\end{algorithmic}
\end{algorithm}
\footnotetext{\label{footnote1} Note that the actual width of the network is $2m$ due to symmetric initialization.}

\vspace{.1in}
\noindent\textbf{Training Algorithm \ref{alg:onestepgd}.}
After pruning the neural network,
we perform a gradient-based training procedure.
Let $S^{d - 1}_\cJ  \sim \text{Unif}\left \{ \bs{x} \in S^{d-1} ~\middle \vert ~ \bs{x}_j = 0 ~ \text{for} ~ j \in [d]\setminus\cJ \right \}$ denote the uniform distribution on the set of unit vectors supported on $\cJ$.  The algorithm symmetrically re-initializes the neural network weights randomly restricted to $\cJ$, i.e.,
\eq{
\Wz_{j*} \sim S^{d-1}_\mathcal{J}   ~ \text{and}~  \Wz_{j*}  =   \Wz_{(2m- j + 1)*}.  \label{eq:initializationtext}
}
We consider a slightly modified version of the one-step gradient descent update used in recent works~\cite{Damian2022NeuralNC,Ba2022HighdimensionalAO,ba2023learning}, namely, we perform a gradient step restricted on set $\cJ $ (Line \ref{alggd:gradientstep}). Here, since both $\Wz$ and  $ \grad_{\bs{W}}  R_n \big( ( \az, \Wz,  \bz)  \big)    \vert_{\cJ}$ are supported on $\cJ$, $\Wo$ is also supported on $\cJ$.    Finally, after training the first layer weights $\Wz$,  we again symmetrically re-initialize the biases and train the second-layer weights using gradient descent (Lines \ref{alggd:reinit} and \ref{alggd:secondlayer}). 

We note that Algorithm~\ref{alg:onestepgd}
as stated can be used to learn both single-index and multi-index models,
and falls under the correlational query algorithms discussed in Section~\ref{sec:csqbound}. However, in the multi-index setting, the algorithm
 needs a slight modification, which we detail in Section~\ref{sec:multiind}.

\section{Main Results} 
\label{sec:main}
In this section, we present learning guarantees on Algorithm~\ref{alg:onestepgd} when the data is generated from either a single-index or a multi-index model. We focus on single-index models first.

\vspace{.1in}
\subsection{Learning Sparse Single-index Models with Pruning}
\label{sec:singleind}
In what follows, we define a complexity measure for the link function to be learned.
\begin{definition}[Information exponent]\label{eq:jie}
For the link function $\gt$, we let $\gt \coloneqq \sum_{k = 0}^p \tfrac{\ghc_k}{k!}  H_{e_k}$ be its Hermite expansion. 
The information exponent of $\gt$, which we denote by $k^\star$, is the index of the first non-zero Hermite coefficient of $\gt$, i.e.,  $k^\star \coloneqq \inf\{k \geq 1 ~ \vert ~ \ghc_k \neq 0 \}$.
\end{definition}
\noindent
Intuitively,  information exponent measures the magnitude of information contained in the gradient at initialization,  and larger $k^\star$ implies increased gradient descent complexity \cite{arous2021online}.  The main result in the single-index setting relies on the above definition, and is given below.

\begin{theorem}
\label{thm:single}
Let  $\norm{\bs{V}}_{2,q}^q  = \Theta \big(  d^{\left( 1 - \frac{q}{2} \right) \alpha} \big)$,  for some  $q \in [0,2)$ and $\alpha \in (0,1)$.   
For any $\varepsilon > 0$,  consider Algorithm \ref{alg:onestepgd} with  $m = \Theta \left(d^{\varepsilon}\right)$,  $c = \tfrac{1}{\log d}$,
\eq{
& \eta_1 = \tilde O \left(M^{\frac{k^{\star} - 1}{2}} \right),  ~~   \lambda_1 = \frac{1}{\eta_1}, ~~  \eta_t = \frac{1}{\tilde O(m) + \lambda_t},  ~~   \lambda_t = \tilde O(m),  ~t \geq 2, ~~ \text{and} ~~ T = \tilde O(1).
}
For every $\ell \in \N$,  there exists a constant $d_{\ell,\varepsilon}$,  depending on $\ell$ and $\varepsilon$, such that for $d \geq d_{\ell,\varepsilon}$,  if 
\eq{
n = \tilde O \Big( d^{\alpha k^\star} \Big) ~~ \text{and} ~~ M = \tilde O \Big( d^{\alpha} \Big),
}
then, Algorithm \ref{alg:onestepgd} guarantees that with probability at least  $1 - d^{- \ell}$
\eq{
\E \left[  \big( \hat y(\bs{x}; (\bs{a}^{(T)},\Wo, \bo)) - y \big)^2 \right] - \E[\epsilon^2] \leq \tilde O \left( \frac{1}{m} +  \sqrt{\frac{M}{n}} \right) +  o_d(1).
}
\end{theorem}

We observe that for any constraint level,  the sample complexity in Theorem \ref{thm:single} reduces to $\tilde O(d^{\alpha k^{\star}})$ for  $\alpha \in (0,1)$, which improves upon the existing $O \big(d^{k^{\star}} \big)$ guarantees for gradient-based algorithms \cite{Bietti2022LearningSM,MousaviHosseini2023GradientBasedFL}.  Moreover,  in the case $\alpha \leq 1/2$,
the upper bound matches with the CSQ lower bound in Theorem \ref{thm:csqlbtext}.
Finally, we observe that for the generalization error to be small, the width $m$ and particularly the ambient dimension $d$ need to be both sufficiently large; thus, the right hand side of the bound vanishes only in high-dimensions.

\subsection{Learning Sparse Multi-index Models with Pruning}
\label{sec:multiind}
In this section, we consider multi-index models, i.e., the case $r > 1$.
We consider Algorithm \ref{alg:onestepgd} with two minor modifications, following a similar construction to
\cite{Damian2022NeuralNC} adapted to our pruning framework. Right after the pruning step, between Lines~\ref{alggd:pruningstep} and \ref{alggd:re-init}, we subtract an estimate of the first Hermite component from the response variable. We add this term back at the output, in Line~\ref{alggd:prediction}. These modifications are given as follows.
\eq{
 {\footnotesize\texttt{1.5:}}\ \  & y_i \leftarrow y_i - \inner{\hmu \vert_{\cJ}}{\bs{x}_i},  ~ i \in [n] ~ \text{where}  ~ \hmu \coloneqq \frac{1}{n} \sum_{i = 1}^n y_i \bs{x}_i, \label{eq:preprocess1}  \\
{\footnotesize\texttt{6:}}\ \  &  \text{Return:} ~ \hat y(\bs{x}; (\bs{a}^{(T)},\Wo, \bo)) =    \inner{\hmu \vert_{\cJ}}{\bs{x}} +  \inner{\bs{a}^{(T)} }{\phi ( \Wo \bs{x} + \bo )}. \label{eq:preprocess2} 
}
We will refer to the modified algorithm as Algorithm~\ref{alg:onestepgd}$^+$.

The following condition on the link function, referred to as \emph{non-degeneracy} in \cite{Damian2022NeuralNC}, is helpful in the analysis.

\begin{assumption}\label{asmp:multi}
The link function $\gt : \R^r \to \R$
satisfies
that $\E[\gt(\bs{z}) \bs{z} \bs{z}^\top]  \in \R^{r \times r}$ 
is full rank. 
\end{assumption}
Under this assumption, $\gt$ has information exponent\footnote{In Definition~\ref{eq:jie} , the information exponent is defined for $r=1$. Similar to an argument by \cite{Abbe2023SGDLO}, we can generalize our definition to encompass multi-index settings by considering the degree of the lowest order Hermite components in $\gt$. With this, Assumption~\ref{asmp:multi} leads to an information exponent $k^{\star} = 2$ in the worst-case scenario, encompassing situations where the first Hermite component does not exist.} $k^\star = 2$.
Therefore, this condition is significantly more restrictive than the assumptions in the single-index case. This is, however, expected since recovering the entire principal subspace spanned by the model directions, i.e., the column space of $\bs{V}$, is significantly more challenging  than recovering a single direction.
Under this condition,
we state
the main result of the multi-index setting.

\begin{theorem}
\label{thm:multi}
Suppose that Assumption~\ref{asmp:multi} holds.  Let  $\norm{\bs{V}}_{2,q}^q  = \Theta \big(  d^{\left( 1 - \frac{q}{2} \right) \alpha} \big)$,  for some  $q \in [0,2)$ and $\alpha \in (0,1)$.   For any $\varepsilon >0$,  consider Algorithm \ref{alg:onestepgd}$^+$ with $m = \Theta(d^{\varepsilon})$, $c = \tfrac{1}{\log d}$,
\eq{
& \eta_1 = \tilde O \left(M  \right),  ~~   \lambda_1 = \frac{1}{\eta_1},  ~~ \eta_t = \frac{1}{\tilde O(m) + \lambda_t},  ~~   \lambda_t = \tilde O(m),  ~t \geq 2, ~~ \text{and} ~~ T = \tilde O(1).
}
For every $\ell \in \N$,  there exists a constant $d_{\ell,\varepsilon}$,  depending on $\ell$ and $\varepsilon$, such that for $d \geq d_{\ell,\varepsilon}$,  if
\eq{
n = \tilde O \Big( d^{2\alpha} \Big) ~~ \text{and} ~~ M = \tilde O \Big( d^{\alpha} \Big),
}
then, Algorithm~\ref{alg:onestepgd}$^+$ guarantees that with probability at least  $1 - d^{- \ell}$
\eq{
\E \left[  \big( \hat y(\bs{x}; (\bs{a}^{(T)},\Wo, \bo)) - y \big)^2 \right] - \E[\epsilon^2] \leq \tilde O \left( \frac{1}{m} + \sqrt{\frac{M}{n}} \right) +  o_d(1).
}
\end{theorem}

The above result states that the improvement in sample-complexity due to pruning extends to the multi-index setting as well. As in the single-index case, for all sparsity levels, gradient descent followed by pruning requires $\tilde O(d^{2\alpha})$, for the soft sparsity level $\Theta(d^{(1-q/2)\alpha})$ and $\alpha \in (0,1)$, which  improves over the existing $\tilde O(d^2)$ bound shown in \cite{Damian2022NeuralNC}. 
It is worth noting that the bound in~\cite{Damian2022NeuralNC} does not meet the CSQ lower bound in their setting. This gap, however, was later closed in \cite{damian2023smoothing} via smoothing the loss. With the additional soft sparsity condition in  Theorem~\ref{thm:multi}, even smoothing will achieve suboptimal sample complexity guarantee since the corresponding CSQ lower bound in this regime becomes smaller.
Nevertheless, observing that the function class in \eqref{eq:frk} satisfies Assumption~\ref{asmp:multi} for $r > 1$ and $k = 2$, our lower bound in Theorem \ref{thm:csqlbtext} implies that the above result is tight in this sense, for $\alpha \leq 1/2$.

For the generalization error to be small in Theorem~\ref{thm:multi}, we require the width $m$ to be large. More crucially,
this bound is small only in high-dimensions where the ambient dimension is large.
Therefore, pruned neural networks learn useful representations via gradient descent, and achieves optimal sample complexity in the above sense in high-dimensions, also in the multi-index setting.

\section{Technicalities Around Pruning}\label{sec:technical}

\vspace{.1in}
\noindent\textbf{First Technical Difficulty.
}
A technical difficulty arises due to the bias introduced by the first-order Hermite components.   To illustrate a pathological case for this problem, we consider two models, one with and one without the first-order Hermite component:
\eq{
 y = \underbrace{ \tfrac{1}{\sqrt{2}} H_{e_2}(\inner{\bs{v}_1}{\bs{x}}) +  \tfrac{1}{\sqrt{2}} H_{e_2}(\inner{\bs{v}_2}{\bs{x}})}_{\text{no first-order Hermite component}}    \quad\quad \text{and} \quad\quad  \check{y} = y + \!\!\!\!\!\!\!\! \underbrace{ \inner{\bs{v}}{\bs{x}} }_{ \small \substack{ \text{first-order} \\ \text{Hermite component}}}   \label{eq:pathex1}
}
where  we choose $\bs{v}_1 = \bs{e}_1,$   $\bs{v}_2 = \bs{e}_2$,  $\bs{v} = \tfrac{- 1}{\sqrt{\pi}} (\bs{e}_1 + \bs{e}_2)$.  Here, the second model,  $\check{y}$, includes an additional first-order Hermite term to illustrate its effect.

For the first model, we can derive the population gradient in \eqref{eq:examplegrad} as follows:
\eq{
 \grad_{j}  R((\az,  \bs{e}_i, \bz)) = - \E \left[y \phi^\prime(\inner{\bs{e}_i}{\bs{x}}) \bs{x}  \right]  = \tfrac{- 1}{2 \sqrt{\pi}} \begin{cases}
  \bs{e}_1   & i = 1\\
    \bs{e}_2    & i = 2 \\
0 & i > 2,
\end{cases}  \label{eq:pathex1res}
}
For the second model, denoted by $\grad_j \check{R}$, the population gradient is given by:
\eq{
 \grad_{j}  \check{R}((\az, \bs{e}_i, \bz)) & = - \E \left[\check{y} \phi^\prime(\inner{\te_i}{\bs{x}}) \bs{x}  \right] \\
  & =  - \underbrace{ \E \left[\inner{\bs{v}}{\bs{x}} \phi^\prime(\inner{\bs{e}_i}{\bs{x}}) \bs{x}  \right] }_{\substack{ \text{due to the additional} \\ \text{ first-order Hermite term} }} - \underbrace{ \E \left[y \phi^\prime(\inner{\bs{e}_i}{\bs{x}}) \bs{x}  \right]  }_{= \eqref{eq:pathex1res}}
{\small = \tfrac{1}{2 \sqrt{\pi}} \begin{cases}
     \bs{e}_2 & i = 1\\
    \bs{e}_1   & i = 2 \\
    \bs{e}_1 + \bs{e}_2 & i > 2.
\end{cases} } ~~ ~~
\label{eq:pathex1res2} 
}
We notice that in the first model,  comparing the gradient magnitudes would recover the support,  whereas in the second model the gradients evaluated at the support of $\bs{v}_1$ and $\bs{v}_ 2$ ($i=1,2$) have smaller norms than other cases (see Appendix \ref{sec:furtherdiss} for the details).

The issue described above arises from the presence of the first-order Hermite term in \eqref{eq:pathex1res2}.   To address this, we consider the even and odd components of the activation separately, as detailed in Section~\ref{sec:training}.
This decomposition allows us to separate the first-order Hermite term from the higher-order terms in the Hermite expansion through even-odd decomposition,  and eliminate the problematic bias of the first-order term illustrated in \eqref{eq:pathex1res}-\eqref{eq:pathex1res2}.

\vspace{.1in}
\noindent\textbf{Second Technical Difficulty.
}
The second technical difficulty arises due to the presence of magnitude mismatch within the entries of $\bs{V}$.   To illustrate, let us consider the following case:  For a small $0 < \varepsilon \ll d^{-1/2}$  and constants $\ghc_2$ and $\ghc_4$ specified later, let
\eq{
\gt(\inner{\bs{v}}{\bs{x}})= \tfrac{\ghc_2}{\sqrt{2}} H_{e_2}(\inner{\bs{v}}{\bs{x}}) +   \tfrac{\ghc_4}{\sqrt{4!}} H_{e_4}(\inner{\bs{v}}{\bs{x}}) ~~\text{with}~~  \bs{v} = \Big( \sqrt{1 - (\sqrt{d}-1) \varepsilon^2},  \underbrace{ \varepsilon, \cdots,    \varepsilon}_{\text{$\sqrt{d}-1$ many}}, 0, 0, \cdots, 0 \Big)  ~~~~ \label{eq:examplesetting2}
}
 where $\bs{v}$ is sparse, i.e. $\norm{\bs{v}}_0  = \sqrt{d} \ll d$, and the first entry of $\bs{v}$  is significantly larger than the rest.  The population gradient in this case is given by
 \eq{
 \grad_{j}  R((\az, \bs{e}_i, \bz)) & = - \E \left[\gt(\inner{\bs{v}}{\bs{x}}) \phi^\prime(\inner{\bs{e}_i}{\bs{x}}) \bs{x}  \right]\\
& = - \underbrace{   \bs{v} \left( \sqrt{2} \ghc_2 \rhc_2  \bs{v}_i + \frac{2   \ghc_4 \rhc_4}{\sqrt{6}} \bs{v}_i^3  \right)  }_{\text{informative term} } 
-  \underbrace{   \bs{e}_i \left(  \frac{\rhc_4 \ghc_2}{\sqrt{2}} \bs{v}_i^2 +   \frac{\rhc_6 \ghc_4}{\sqrt{4!}} \bs{v}_i^4 \right) }_{\text{extra term}}, \label{eq:examplegrad2}
}
where  $\rhc_i$ denotes the $i^{th}$ Hermite coefficients of the ReLU activation $\phi(\cdot)$. The informative term contains the information about the direction  $\bs{v}$ while the extra term appears due to the properties of Hermite polynomials.
Here,  a very large $\bs{v}_i$ might cause  extra terms to be comparable to the informative terms,  leading to cancellation.  As detailed in Appendix \ref{sec:furtherdiss},  we can find $(\ghc_2, \ghc_4, \varepsilon)$ such that for $i = 1$ (corresponding to largest entry in $\bs{V}$),  the informative and extra terms cancel each other in \eqref{eq:examplegrad2}, i.e.,   $\text{informative term}  \approx  - \text{extra term}$,  making the algorithm require exponentially many samples to find the largest entry. 

On the other hand,  we observe that if  $\bs{v}_i$'s vanish with $d$  in   \eqref{eq:examplegrad2}, the informative term  would dominate since it scales with $O(\bs{v}_i)$ whereas the extra term scales with $O(\bs{v}_i^2)$. To make sure that is the case in the presence of very large entries in $\bs{V}$, we use data augmentation and compare the magnitude of gradients evaluated at a shifted standard basis,  as detailed in Section~\ref{sec:training}.
Note that in this case, 
\eq{
 \grad_{j}  R((\az, \te_i, \bz)) & = - \E \left[ \gt(\inner{\bs{v}}{\bs{x}}) \phi^\prime(\inner{\te_i}{\bs{x}}) \bs{x}  \right]\\
& = - \underbrace{ c \bs{v} \left( \sqrt{2} \ghc_2 \rhc_2  \bs{v}_i + c^2  \frac{2 \ghc_4 \rhc_4}{\sqrt{6}} \bs{v}_i^3  \right)  }_{\text{informative term} } 
- \underbrace{  c^2 \bs{e}_i \left(  \frac{\rhc_4 \ghc_2}{\sqrt{2!}} \bs{v}_i^2 + c^2 \frac{\rhc_6 \ghc_4}{\sqrt{4!}} \bs{v}_i^4 \right) }_{\text{extra term}},  ~~ ~~~\label{eq:examplegrad2}
}
where a sufficiently small  $c > 0$ ensures that the informative term dominates the right-hand side.

\section{Discussion}
We studied how pruning impacts the sample complexity of learning single and multi-index models. Our results show that pruning the network to a sparsity level proportional to the soft sparsity of relevant model directions significantly improves sample complexity. Moreover, we supported our results with a sparsity-aware CSQ lower bound which revealed that if the sparsity level exceeds a certain threshold, the sample complexity of training a pruned network cannot be improved in general. Conversely, the gap between our lower bound and the CSQ lower bound for the general dense case suggests that basis-independent methods, such as gradient descent initialized with a rotationally independent distribution, cannot achieve the sample complexity of the pruned network.

We outline a few limitations of our current work and discuss directions for future research.
\begin{itemize}[leftmargin = .3in]
    \item In our work, we considered training network weights with a single gradient step. However, recent research suggests that using multiple gradient descent steps in the multi-index setting yields improved sample complexity compared to single-step algorithms \cite{Abbe2023SGDLO, Dandi2024TheBO}. Therefore, considering pruning with a multi-step gradient descent algorithm can provide a more complete picture. Particularly, investigating pruning in the context of incremental (or curriculum) learning presents an interesting direction for future research.
    \item In the gradient-based algorithm, we considered a somewhat unconventional initialization, leveraging the symmetry it introduces.  It would be interesting to examine cases where we train a network with
    multiple neurons starting from a more standard initialization. This analysis is challenging due to the
    interactions between the neurons.
    \item The results presented in this paper are based on the assumption that the input distribution follows an isotropic Gaussian distribution. Recent works~\cite{MousaviHosseini2023GradientBasedFL,ba2023learning} showed that there is an intricate interplay between the model and the important covariance directions, and the overall performance of neural networks is governed by their interplay. Studying the effect of pruning in this regime and also extending our results to 
     other distributions~\cite{roy2021empirical}, for example via zero-biased transformations~\cite{Goldstein1997SteinsMA,goldstein2019non}, is a topic for future research.
\end{itemize}

\section*{Acknowledgements}
Authors thank Berivan Isik and Alireza Mousavi-Hosseini for  helpful discussions and feedback. MAE was partially supported by
NSERC Grant [2019-06167], CIFAR AI Chairs program, and CIFAR Catalyst grant.

\bibliographystyle{alpha}
\bibliography{ref.bib}

\newcommand{\etalchar}[1]{$^{#1}$}
\begin{thebibliography}{MHWSE23}

\bibitem[AAM22]{Abbe2022TheMP}
Emmanuel Abbe, Enric~Boix Adser{\`{a}}, and Theodor Misiakiewicz.
\newblock The merged-staircase property: a necessary and nearly sufficient
  condition for {SGD} learning of sparse functions on two-layer neural
  networks.
\newblock In {\em Conference on Learning Theory, 2-5 July 2022, London, {UK}},
  volume 178 of {\em Proceedings of Machine Learning Research}, pages
  4782--4887. {PMLR}, 2022.

\bibitem[AAM23]{Abbe2023SGDLO}
Emmanuel Abbe, Enric~Boix Adser{\`a}, and Theodor Misiakiewicz.
\newblock Sgd learning on neural networks: leap complexity and saddle-to-saddle
  dynamics.
\newblock In {\em Proceedings of Thirty Sixth Conference on Learning Theory},
  volume 195 of {\em Proceedings of Machine Learning Research}, pages
  2552--2623. PMLR, 12--15 Jul 2023.

\bibitem[AGJ21]{arous2021online}
Gerard~Ben Arous, Reza Gheissari, and Aukosh Jagannath.
\newblock Online stochastic gradient descent on non-convex losses from
  high-dimensional inference.
\newblock {\em The Journal of Machine Learning Research}, 22(1):4788--4838,
  2021.

\bibitem[ALS]{AllenZhu2018ACT}
Zeyuan Allen{-}Zhu, Yuanzhi Li, and Zhao Song.
\newblock A convergence theory for deep learning via over-parameterization.
\newblock In {\em Proceedings of the 36th International Conference on Machine
  Learning, {ICML} 2019, 9-15 June 2019, Long Beach, California, {USA}},
  volume~97 of {\em Proceedings of Machine Learning Research}, pages 242--252.
  {PMLR}.

\bibitem[BBSS22]{Bietti2022LearningSM}
Alberto Bietti, Joan Bruna, Clayton Sanford, and Min~Jae Song.
\newblock Learning single-index models with shallow neural networks.
\newblock In {\em Advances in Neural Information Processing Systems 35: Annual
  Conference on Neural Information Processing Systems 2022, NeurIPS 2022, New
  Orleans, LA, USA, November 28 - December 9, 2022}, 2022.

\bibitem[BES{\etalchar{+}}22]{Ba2022HighdimensionalAO}
Jimmy Ba, Murat~A Erdogdu, Taiji Suzuki, Zhichao Wang, Denny Wu, and Greg Yang.
\newblock High-dimensional asymptotics of feature learning: How one gradient
  step improves the representation.
\newblock In {\em Advances in Neural Information Processing Systems},
  volume~35, pages 37932--37946. Curran Associates, Inc., 2022.

\bibitem[BES{\etalchar{+}}23]{ba2023learning}
Jimmy Ba, Murat~A. Erdogdu, Taiji Suzuki, Zhichao Wang, and Denny Wu.
\newblock Learning in the presence of low-dimensional structure: {A} spiked
  random matrix perspective.
\newblock In {\em Advances in Neural Information Processing Systems 36: Annual
  Conference on Neural Information Processing Systems 2023, NeurIPS 2023, New
  Orleans, LA, USA, December 10 - 16, 2023}, 2023.

\bibitem[BMBE20]{Bartoldson2019TheGT}
Brian~R. Bartoldson, Ari~S. Morcos, Adrian Barbu, and Gordon Erlebacher.
\newblock The generalization-stability tradeoff in neural network pruning.
\newblock In {\em Advances in Neural Information Processing Systems 33: Annual
  Conference on Neural Information Processing Systems 2020, NeurIPS 2020,
  December 6-12, 2020, virtual}, 2020.

\bibitem[BSE{\etalchar{+}}21]{Barsbey2021HeavyTI}
Melih Barsbey, Milad Sefidgaran, Murat~A. Erdogdu, Ga{\"{e}}l Richard, and Umut
  Simsekli.
\newblock Heavy tails in {SGD} and compressibility of overparametrized neural
  networks.
\newblock In {\em Advances in Neural Information Processing Systems 34: Annual
  Conference on Neural Information Processing Systems 2021, NeurIPS 2021,
  December 6-14, 2021, virtual}, pages 29364--29378, 2021.

\bibitem[Bub15]{Bubeck2014ConvexOA}
S{\'e}bastien Bubeck.
\newblock Convex optimization: Algorithms and complexity, 2015.

\bibitem[CFC{\etalchar{+}}20]{Chen2020TheLT}
Tianlong Chen, Jonathan Frankle, Shiyu Chang, Sijia Liu, Yang Zhang, Zhangyang
  Wang, and Michael Carbin.
\newblock The lottery ticket hypothesis for pre-trained {BERT} networks.
\newblock In {\em Advances in Neural Information Processing Systems 33: Annual
  Conference on Neural Information Processing Systems 2020, NeurIPS 2020,
  December 6-12, 2020, virtual}, 2020.

\bibitem[Chi22]{Chizat2022MeanFieldLD}
L{\'e}na{\"i}c Chizat.
\newblock Mean-field langevin dynamics : Exponential convergence and annealing.
\newblock {\em Trans. Mach. Learn. Res.}, 2022.

\bibitem[COB19]{Chizat2018OnLT}
L\'{e}na\"{\i}c Chizat, Edouard Oyallon, and Francis Bach.
\newblock On lazy training in differentiable programming.
\newblock In {\em Advances in Neural Information Processing Systems},
  volume~32. Curran Associates, Inc., 2019.

\bibitem[CW01]{Carbery2001DistributionalAL}
Anthony Carbery and James Wright.
\newblock Distributional and l-q norm inequalities for polynomials over convex
  bodies in r-n.
\newblock {\em Mathematical Research Letters}, 8:233--248, 2001.

\bibitem[DKL{\etalchar{+}}23]{Dandi2023HowTN}
Yatin Dandi, Florent Krzakala, Bruno Loureiro, Luca Pesce, and Ludovic Stephan.
\newblock How two-layer neural networks learn, one (giant) step at a time,
  2023.

\bibitem[DLS22]{Damian2022NeuralNC}
Alexandru Damian, Jason Lee, and Mahdi Soltanolkotabi.
\newblock Neural networks can learn representations with gradient descent.
\newblock In {\em Proceedings of Thirty Fifth Conference on Learning Theory},
  volume 178 of {\em Proceedings of Machine Learning Research}, pages
  5413--5452. PMLR, 02--05 Jul 2022.

\bibitem[DNGL23]{damian2023smoothing}
Alex Damian, Eshaan Nichani, Rong Ge, and Jason~D Lee.
\newblock Smoothing the landscape boosts the signal for sgd: Optimal sample
  complexity for learning single index models.
\newblock In {\em Advances in Neural Information Processing Systems},
  volume~36, pages 752--784. Curran Associates, Inc., 2023.

\bibitem[DTA{\etalchar{+}}24]{Dandi2024TheBO}
Yatin Dandi, Emanuele Troiani, Luca Arnaboldi, Luca Pesce, Lenka
  Zdeborov{\'{a}}, and Florent Krzakala.
\newblock The benefits of reusing batches for gradient descent in two-layer
  networks: Breaking the curse of information and leap exponents.
\newblock {\em CoRR}, abs/2402.03220, 2024.

\bibitem[DZPS19]{Du2018GradientDP}
Simon~S. Du, Xiyu Zhai, Barnab{\'{a}}s P{\'{o}}czos, and Aarti Singh.
\newblock Gradient descent provably optimizes over-parameterized neural
  networks.
\newblock In {\em 7th International Conference on Learning Representations,
  {ICLR} 2019, New Orleans, LA, USA, May 6-9, 2019}. OpenReview.net, 2019.

\bibitem[EBD19]{erdogdu2019scalable}
Murat Erdogdu, Mohsen Bayati, and Lee~H Dicker.
\newblock Scalable approximations for generalized linear problems.
\newblock {\em Journal of Machine Learning Research}, 20(7):1--45, 2019.

\bibitem[Erd15]{erdogdu2015newton}
Murat~A Erdogdu.
\newblock Newton-stein method: a second order method for glms via stein's
  lemma.
\newblock In {\em Proceedings of Advances in Neural Information Processing
  Systems}, pages 1216--1224, 2015.

\bibitem[FC19]{Frankle2018TheLT}
Jonathan Frankle and Michael Carbin.
\newblock The lottery ticket hypothesis: Finding sparse, trainable neural
  networks.
\newblock In {\em 7th International Conference on Learning Representations,
  {ICLR} 2019, New Orleans, LA, USA, May 6-9, 2019}. OpenReview.net, 2019.

\bibitem[FDRC20]{Frankle2019LinearMC}
Jonathan Frankle, Gintare~Karolina Dziugaite, Daniel Roy, and Michael Carbin.
\newblock Linear mode connectivity and the lottery ticket hypothesis.
\newblock In {\em Proceedings of the 37th International Conference on Machine
  Learning}, volume 119 of {\em Proceedings of Machine Learning Research},
  pages 3259--3269. PMLR, 13--18 Jul 2020.

\bibitem[GEH19]{Gale2019TheSO}
Trevor Gale, Erich Elsen, and Sara Hooker.
\newblock The state of sparsity in deep neural networks.
\newblock {\em ArXiv}, abs/1902.09574, 2019.

\bibitem[GMMM20]{Ghorbani2020WhenDN}
Behrooz Ghorbani, Song Mei, Theodor Misiakiewicz, and Andrea Montanari.
\newblock When do neural networks outperform kernel methods?
\newblock In {\em Advances in Neural Information Processing Systems 33: Annual
  Conference on Neural Information Processing Systems 2020, NeurIPS 2020,
  December 6-12, 2020, virtual}, 2020.

\bibitem[GO94]{Giles1994PruningRN}
C.~Lee Giles and Christian~W. Omlin.
\newblock Pruning recurrent neural networks for improved generalization
  performance.
\newblock {\em {IEEE} Trans. Neural Networks}, 5(5):848--851, 1994.

\bibitem[GR97]{Goldstein1997SteinsMA}
Larry Goldstein and Gesine Reinert.
\newblock Stein's method and the zero bias transformation with application to
  simple random sampling.
\newblock {\em The Annals of Applied Probability}, 7(4), November 1997.

\bibitem[GSJW19]{Geiger2019DisentanglingFA}
Mario Geiger, Stefano Spigler, Arthur Jacot, and Matthieu Wyart.
\newblock Disentangling feature and lazy learning in deep neural networks: an
  empirical study.
\newblock {\em CoRR}, abs/1906.08034, 2019.

\bibitem[GW19]{goldstein2019non}
Larry Goldstein and Xiaohan Wei.
\newblock Non-gaussian observations in nonlinear compressed sensing via stein
  discrepancies.
\newblock {\em Information and Inference: A Journal of the IMA}, 8(1):125--159,
  2019.

\bibitem[HPTD15]{Han2015LearningBW}
Song Han, Jeff Pool, John Tran, and William Dally.
\newblock Learning both weights and connections for efficient neural network.
\newblock In {\em Advances in Neural Information Processing Systems},
  volume~28. Curran Associates, Inc., 2015.

\bibitem[HS92]{Hassibi1992SecondOD}
Babak Hassibi and David Stork.
\newblock Second order derivatives for network pruning: Optimal brain surgeon.
\newblock In {\em Advances in Neural Information Processing Systems}, volume~5.
  Morgan-Kaufmann, 1992.

\bibitem[JCR{\etalchar{+}}22]{Jin2022PruningsEO}
Tian Jin, Michael Carbin, Daniel~M. Roy, Jonathan Frankle, and Gintare~Karolina
  Dziugaite.
\newblock Pruning's effect on generalization through the lens of training and
  regularization.
\newblock In {\em Advances in Neural Information Processing Systems 35: Annual
  Conference on Neural Information Processing Systems 2022, NeurIPS 2022, New
  Orleans, LA, USA, November 28 - December 9, 2022}, 2022.

\bibitem[JHG18]{Jacot2018NeuralTK}
Arthur Jacot, Cl{\'{e}}ment Hongler, and Franck Gabriel.
\newblock Neural tangent kernel: Convergence and generalization in neural
  networks.
\newblock In {\em Advances in Neural Information Processing Systems 31: Annual
  Conference on Neural Information Processing Systems 2018, NeurIPS 2018,
  December 3-8, 2018, Montr{\'{e}}al, Canada}, pages 8580--8589, 2018.

\bibitem[KLS24]{Kumar2024NoFP}
Tanishq Kumar, Kevin Luo, and Mark Sellke.
\newblock No free prune: Information-theoretic barriers to pruning at
  initialization.
\newblock {\em CoRR}, abs/2402.01089, 2024.

\bibitem[LD89]{Li1989RegressionAU}
Ker-Chau Li and Naihua Duan.
\newblock Regression analysis under link violation.
\newblock {\em Annals of Statistics}, 17:1009--1052, 1989.

\bibitem[LDS89]{LeCun1989OptimalBD}
Yann LeCun, John Denker, and Sara Solla.
\newblock Optimal brain damage.
\newblock In D.~Touretzky, editor, {\em Advances in Neural Information
  Processing Systems}, volume~2. Morgan-Kaufmann, 1989.

\bibitem[LSZ{\etalchar{+}}19]{Liu2018RethinkingTV}
Zhuang Liu, Mingjie Sun, Tinghui Zhou, Gao Huang, and Trevor Darrell.
\newblock Rethinking the value of network pruning.
\newblock In {\em 7th International Conference on Learning Representations,
  {ICLR} 2019, New Orleans, LA, USA, May 6-9, 2019}. OpenReview.net, 2019.

\bibitem[MHPG{\etalchar{+}}23]{MousaviHosseini2022NeuralNE}
Alireza Mousavi-Hosseini, Sejun Park, Manuela Girotti, Ioannis Mitliagkas, and
  Murat~A. Erdogdu.
\newblock Neural networks efficiently learn low-dimensional representations
  with {SGD}.
\newblock In {\em The Eleventh International Conference on Learning
  Representations, {ICLR} 2023, Kigali, Rwanda, May 1-5, 2023}. OpenReview.net,
  2023.

\bibitem[MHWSE23]{MousaviHosseini2023GradientBasedFL}
Alireza Mousavi-Hosseini, Denny Wu, Taiji Suzuki, and Murat~A Erdogdu.
\newblock Gradient-based feature learning under structured data.
\newblock In {\em Advances in Neural Information Processing Systems},
  volume~36, pages 71449--71485. Curran Associates, Inc., 2023.

\bibitem[MMM19]{Mei2019MeanfieldTO}
Song Mei, Theodor Misiakiewicz, and Andrea Montanari.
\newblock Mean-field theory of two-layers neural networks: dimension-free
  bounds and kernel limit.
\newblock In {\em Conference on Learning Theory, {COLT} 2019, 25-28 June 2019,
  Phoenix, AZ, {USA}}, volume~99 of {\em Proceedings of Machine Learning
  Research}, pages 2388--2464. {PMLR}, 2019.

\bibitem[MTK{\etalchar{+}}17]{Molchanov2016PruningCN}
Pavlo Molchanov, Stephen Tyree, Tero Karras, Timo Aila, and Jan Kautz.
\newblock Pruning convolutional neural networks for resource efficient
  inference.
\newblock In {\em 5th International Conference on Learning Representations,
  {ICLR} 2017, Toulon, France, April 24-26, 2017, Conference Track
  Proceedings}. OpenReview.net, 2017.

\bibitem[MYSS20]{Malach2020ProvingTL}
Eran Malach, Gilad Yehudai, Shai Shalev{-}Shwartz, and Ohad Shamir.
\newblock Proving the lottery ticket hypothesis: Pruning is all you need.
\newblock In {\em Proceedings of the 37th International Conference on Machine
  Learning, {ICML} 2020, 13-18 July 2020, Virtual Event}, volume 119 of {\em
  Proceedings of Machine Learning Research}, pages 6682--6691. {PMLR}, 2020.

\bibitem[OHR20]{Laurent2020}
Laurent Orseau, Marcus Hutter, and Omar Rivasplata.
\newblock Logarithmic pruning is all you need.
\newblock In {\em Advances in Neural Information Processing Systems 33: Annual
  Conference on Neural Information Processing Systems 2020, NeurIPS 2020,
  December 6-12, 2020, virtual}, 2020.

\bibitem[OS20]{Oymak2019TowardMO}
Samet Oymak and Mahdi Soltanolkotabi.
\newblock Toward moderate overparameterization: Global convergence guarantees
  for training shallow neural networks.
\newblock {\em {IEEE} J. Sel. Areas Inf. Theory}, 1(1):84--105, 2020.

\bibitem[Pin94]{Pinelis1994}
Iosif Pinelis.
\newblock {Optimum Bounds for the Distributions of Martingales in Banach
  Spaces}.
\newblock {\em The Annals of Probability}, 22(4):1679 -- 1706, 1994.

\bibitem[RBE21]{roy2021empirical}
Abhishek Roy, Krishnakumar Balasubramanian, and Murat~A Erdogdu.
\newblock On empirical risk minimization with dependent and heavy-tailed data.
\newblock {\em Advances in Neural Information Processing Systems},
  34:8913--8926, 2021.

\bibitem[RWY11]{raskutti2011minimax}
Garvesh Raskutti, Martin~J Wainwright, and Bin Yu.
\newblock Minimax rates of estimation for high-dimensional linear regression
  over q-balls.
\newblock {\em IEEE transactions on information theory}, 57(10):6976--6994,
  2011.

\bibitem[Tao12]{Tao2012TopicsIR}
Terence Tao.
\newblock Topics in random matrix theory.
\newblock 2012.

\bibitem[Ver10]{Vershynin2010IntroductionTT}
Roman Vershynin.
\newblock Introduction to the non-asymptotic analysis of random matrices.
\newblock In {\em Compressed Sensing}, 2010.

\bibitem[Ver18]{Vershynin2018HighDimensionalP}
Roman Vershynin.
\newblock {\em High-Dimensional Probability: An Introduction with Applications
  in Data Science}.
\newblock Cambridge Series in Statistical and Probabilistic Mathematics.
  Cambridge University Press, 2018.

\bibitem[WWW{\etalchar{+}}16]{Wen2016LearningSS}
Wei Wen, Chunpeng Wu, Yandan Wang, Yiran Chen, and Hai Li.
\newblock Learning structured sparsity in deep neural networks.
\newblock In {\em Advances in Neural Information Processing Systems},
  volume~29. Curran Associates, Inc., 2016.

\bibitem[YLG{\etalchar{+}}23]{yang2023theoretical}
Hongru Yang, Yingbin Liang, Xiaojie Guo, Lingfei Wu, and Zhangyang Wang.
\newblock Theoretical characterization of how neural network pruning affects
  its generalization, 2023.

\bibitem[YS19]{Yehudai2019OnTP}
Gilad Yehudai and Ohad Shamir.
\newblock On the power and limitations of random features for understanding
  neural networks.
\newblock In {\em Advances in Neural Information Processing Systems 32: Annual
  Conference on Neural Information Processing Systems 2019, NeurIPS 2019,
  December 8-14, 2019, Vancouver, BC, Canada}, pages 6594--6604, 2019.

\bibitem[ZG18]{Zhu2017ToPO}
Michael Zhu and Suyog Gupta.
\newblock To prune, or not to prune: Exploring the efficacy of pruning for
  model compression.
\newblock In {\em 6th International Conference on Learning Representations,
  {ICLR} 2018, Vancouver, BC, Canada, April 30 - May 3, 2018, Workshop Track
  Proceedings}. OpenReview.net, 2018.

\bibitem[ZLLY19]{Zhou2019DeconstructingLT}
Hattie Zhou, Janice Lan, Rosanne Liu, and Jason Yosinski.
\newblock Deconstructing lottery tickets: Zeros, signs, and the supermask.
\newblock In {\em Advances in Neural Information Processing Systems 32: Annual
  Conference on Neural Information Processing Systems 2019, NeurIPS 2019,
  December 8-14, 2019, Vancouver, BC, Canada}, pages 3592--3602, 2019.

\end{thebibliography}

\newpage
\appendix

\tableofcontents
\section{Further Discussion for Section \ref{sec:training}}
\label{sec:furtherdiss}
In this section, we detail the examples discussed in Section  \ref{sec:training}.  Recall that $\phi$ is the ReLU activation with the Hermite expansion $\phi = \sum_{k \geq 0} \frac{\rhc_k}{k!} H_{e_k}$.  Notably,  the coefficients are $\rhc_1 = \tfrac{1}{2}$, $\rhc_2 = \tfrac{1}{\sqrt{2\pi}}$, , $\rhc_3 = 0$, $\rhc_4 = \tfrac{ - 1}{\sqrt{2\pi}}$, and $\rhc_6 =  \tfrac{3}{\sqrt{2\pi}}$ (see \eqref{eq:constantterms} with $b = 0$). 

First, we consider the setting in  \eqref{eq:pathex1}.  In this case,  for $\bs{w} \in S^{d-1}$, we have
\eq{
\E \left[ y \phi^{\prime} (\inner{\bs{w}}{\bs{x}}) \bs{x} \right] & =  \sqrt{2} \rhc_2 \inner{\bs{v}_1}{\bs{w}} \bs{v}_1 +  \sqrt{2} \rhc_2 \inner{\bs{v}_2}{\bs{w}} \bs{v}_2 + \tfrac{\rhc_4}{\sqrt{2}} ( \inner{\bs{w}	}{\bs{v}_1}^2 +  \inner{\bs{w}	}{\bs{v}_2}^2 ) \bs{w} \\
& =  \tfrac{1}{\sqrt{\pi}}   \inner{\bs{e}_1}{\bs{w}} \bs{e}_1 +   \tfrac{1}{\sqrt{\pi}}  \inner{\bs{e}_2}{\bs{w}} \bs{e}_2  - \tfrac{1}{2 \sqrt{\pi}}  \big( \inner{\bs{w}	}{\bs{e}_1}^2 +  \inner{\bs{w}}{\bs{e}_2}^2 \big) \bs{w}, \label{eq:pathex1corr0}
}
using an argument by~\cite{erdogdu2019scalable} and
\eq{
\E \left[ \check{y} \phi^{\prime} (\inner{\bs{w}}{\bs{x}}) \bs{x} \right] & = \rhc_1 \bs{v} + \E \left[ y \phi^{\prime} (\inner{\bs{w}}{\bs{x}}) \bs{x} \right]    =   \tfrac{- 1}{2 \sqrt{\pi}}  (\bs{e}_1 + \bs{e}_2)  +  \E \left[ y \phi^{\prime} (\inner{\bs{w}}{\bs{x}}) \bs{x} \right],    \label{eq:pathex1corr1}
}
where we used the defined values in  \eqref{eq:pathex1}.   From \eqref{eq:pathex1corr0}-\eqref{eq:pathex1corr1},  we deduce
\eq{
\E \left[ y \phi^{\prime} (\inner{\bs{e}_i}{\bs{x}}) \bs{x} \right] =   \tfrac{1}{2\sqrt{\pi}} \begin{cases}
\bs{e}_1 & i = 1 \\
\bs{e}_2 & i = 2 \\
0 & i > 2,
\end{cases} ~~ \text{and} ~~ 
\E \left[ \check{y} \phi^{\prime} (\inner{\bs{e}_i}{\bs{x}}) \bs{x} \right] =   \tfrac{- 1}{2\sqrt{\pi}} \begin{cases}
\bs{e}_2 & i = 1 \\
\bs{e}_1 & i = 2 \\
\bs{e}_1 + \bs{e}_2 & i > 2,
\end{cases}
}
confirming \eqref{eq:pathex1res} and \eqref{eq:pathex1res2}.

For \eqref{eq:examplesetting2}, let us consider $\ghc_2 = 1$, $\ghc_4 =  2\sqrt{3}$, and $\varepsilon = e^{-d}$.  Using \eqref{eq:examplegrad2},  we can show that the population gradient in this case satisfies:
\eq{
\norm{ \E \left[ y \phi^{\prime} (\inner{\bs{e}_i}{\bs{x}}) \bs{x} \right]  }_2 = \begin{cases}
O \big(d^{\frac{1}{4}} e^{-d} \big), & i = 1 \\
O(e^{-d}), & i = 2, \cdots, \sqrt{d} \\
0, & i > \sqrt{d}.
\end{cases}
}
We note that in this case,  an exponentially large sample size in $d$ is required to differentiate between $i = 1$ then $i = d$  using empirical gradients.

\section{Preliminaries for Proofs}

\smallskip
\textbf{Additional Notation:} Unless otherwise stated,   $Z$ follows the standard Gaussian distribution with a dimension depending on the context.   We let $\cgp \coloneqq  \E[\norm{ \grad \gt(Z) }_2^2]^{1/2}$.  We use $S_{M}^{d-1}$ to denote the $M$-sparse $d$-dimensional unit vectors, i.e.,   $S_{M}^{d-1} \coloneqq \{ \bs{x} \in S^{d-1} ~ \vert ~ \norm{\bs{x}}_0 \leq M  \}$.  For a matrix $\bs{A} \in \R^{d_1 \times d_2}$,   $\sigma_1(\bs{A}) \geq \sigma_2(\bs{A}) \geq \cdots \geq \sigma_{d_1 \wedge d_2}(\bs{A})$  denotes the singular values of $\bs{A}$.  For  $\cJ_1 \subseteq [d_1]$ and $\cJ_2 \subseteq [d_2]$, we let  $\bs{A} \vert_{\cJ_1}, ~ \bs{A} \vert_{\text{${ \scriptscriptstyle \cJ_1 \times \cJ_2 }$}}  \in \R^{d_1 \times d_2}$ such that
\eq{
( \bs{A} \vert_{\text{${ \scriptscriptstyle \cJ_1}$}} )_{ij} = \begin{cases}
\bs{A}_{ij} & i \in \cJ_1   \\
0 & \text{otherwise}.
\end{cases}
~~
\text{and}
~~
( \bs{A} \vert_{\text{${ \scriptscriptstyle \cJ_1 \times \cJ_2 }$}} )_{ij} = \begin{cases}
\bs{A}_{ij} & i \in \cJ_1 ~ \text{and} ~ j \in \cJ_2 \\
0 & \text{otherwise}.
\end{cases}
}
In the following,  $C,  K > 0$ are constants that might take different values in different statements.  For reader's convenience, we track on which variable they depend.     For a set $E$,
\eq{
\indic{E}(\bs{x}) \coloneqq \begin{cases}
1 & \bs{x} \in E \\
0 & \text{otherwise}
\end{cases}
}
We use $\cD \coloneqq \{ (\bs{x}_i, y_i) \}_{i = 1}^n$ to denote the dataset.

\smallskip \noindent 
\textbf{Additional Definitions:}  For notational simplicity,  we assume that
\eq{
\abs{\gt(\bs{z})} \leq C_1 (1 + \norm{\bs{z}}_2^2)^{C_2}  ~~ \text{for some} ~~C_1 > 0,  ~C_2 \geq \frac{1}{2}.  \label{eq:asspolytails}
}
We note that since $\gt$ is a polynomial this assumption will always hold.  Furthermore,  in the proof, we particularly consider the model 
\eq{
y \coloneqq \gt(\bs{V}^\top \bs{x}) + \sqrt{\Delta} \epsilon,  \label{eq:model}
}
where $\Delta > 0$ and $\epsilon$ has sub-Gaussian tails, i.e., $\mpr \left[ \abs{\epsilon} > t \right] \leq 2 e^{-t^2}$.

We recall that $\phi(t)  = \max \{ 0, t\}$ denotes the ReLU activation. To be precise, we define the initialization considered in Algorithm \ref{alg:onestepgd} mathematically as follows:
\eq{
\tag{INIT}
\bs{W}_{j*}^{(0)} =  \left( \sum_{i \in \cJ} \bs{W}^2_{ji} \right)^{-1} (\bs{W}_{j1}\indic{1 \in \cJ},  \cdots, \bs{W}_{jd}\indic{d \in \cJ}) \label{eq:initdist}
}
where $\cJ$ is the output of \texttt{PruneNetwork} (see Algorithm \ref{alg:pruningalg}),  $\bs{W} \in \R^{m \times d}$,  $\bs{W}_{ij} \sim_{iid} \cN(0,1)$, and $\bs{W}$ is independent of $\cD$. As for definition \eqref{eq:model},  in the multi-index setting,  we use
\eq{
\tag{DEF-H}
\E[\gt(\bs{z}) \bs{z} \bs{z}^\top] \coloneqq \bs{D} \in \R^{r \times r} ~~ and  ~~  \E[\gt(\bs{V}^\top \bs{x}) \bs{x} \bs{x}^\top] = \bs{V} \bs{D} \bs{V}^\top \coloneqq \bs{H},  \label{eq:multiindexas}
}
which follows from Stein's lemma~\cite{erdogdu2015newton,erdogdu2019scalable}.
Without loss of generality,  we assume $\bs{D}$ is diagonal.

\section{Hermite Expansion in the Multi-Index Setting}
\subsection{Background on Tensors}
In the following, we will use the tensor representation of multivariate Hermite polynomials. Therefore, we introduce some new notation to work with tensors:  We denote tensors with boldface uppercase letters, (e.g. $\bm{T}$).   Unless specified,  we assume that tensors take a value from an abstract inner product space, denoted with   $\mathcal{H}$,  with an inner product, of $\inner{\cdot}{\cdot}_{\cH}$.  For a $k$-tensor $\bm{T_k} :(\R^d)^{\otimes k} \to \cH$ and an index tuple $(i_1, \cdots, i_k) \in [d]^{k}$, we use $\bm{T_k} \vert_{i_1\cdots i_k} \coloneqq \bm{T_k}[\bs{e}_{i_1}, \bs{e}_{i_2}, \cdots, \bs{e}_{i_k}]$, where $\{ \bs{e}_i \}_{i \in [d]}$ is the standard basis for $\R^d$.  We define the inner product and Frobenius norm for $k$-tensors $\bm{T_k}, \bm{\tilde T_k} : (\R^d)^{\otimes k} \to \cH$ as
\eq{
\inner{\bm{T_k}}{\bm{ \tilde T_k}} \coloneqq \sum_{(i_1,\cdots, i_k) \in [d]^k} \inner{ \bm{T_k} \vert_{i_1\cdots i_k}}{ \bm{\tilde T_k} \vert_{i_1\cdots i_k} }_{\cH} ~~ \text{and} ~~  \norm{\bm{T_k}}_F \coloneqq \sqrt{\inner{\bm{T_k}}{\bm{T_k}}}.  \label{def:frobtensor}
}
We use $sym(\cdot)$ to denote symmetrization operator, i.e., 
\eq{
sym(\bm{T_k})[\bs{e}_{i_1}, \bs{e}_{i_2}, \cdots, \bs{e}_{i_k}] = \frac{1}{k!} \sum_{\tau \in S_k} \bm{T_k}[\bs{e}_{\tau(i_1)}, \bs{e}_{\tau(i_2)}, \cdots, \bs{e}_{\tau(i_k)}]  \label{def:symtensor}
}
where $S_k$ is the set of permutations for $[k]$. We say a tensor is symmetric if $\bm{T_k} = sym(\bm{T_k} )$.  For a vector $\bs{u} \in \R^d$,  $\bs{u}^{\otimes k}: (\R^d)^{\otimes k} \to \R$ is a symmetric $k$-tensor defined as  $\bs{u}^{\otimes k}[\bs{v}_1, \cdots, \bs{v}_k] = \prod_{i = 1}^k \inner{\bs{u}}{\bs{v}_i}$.
\subsubsection{Auxiliary Tensor Results}
In this part, we present some useful tensor related result that we will use in the following.
\begin{proposition}
\label{prop:symmtensor}
Let $\bm{T_k} : (\R^d)^{\otimes k} \to \cH$ be a symmetric $k$-tensor.  For any $k$-tensor $\bm{\tilde T_k},$  we have  $\inner{\bm{\tilde T_k}}{\bm{T_k}} =\inner{sym(\bm{\tilde T_k})}{\bm{T_k}}$.
\end{proposition}

\begin{proof}
We have 
\eq{
\inner{\bm{\tilde T_k}}{\bm{T_k}}  
& \labelrel={sym:eqq0}          \sum_{(i_1,\cdots,i_k) \in [d]^k}    \frac{1}{k!} \sum_{\tau \in S_k}  \inner{ \bm{\tilde T_k} \vert_{i_1\cdots i_k} }{ \bm{T_k}[\bs{e}_{\tau(i_1)}, \cdots, \bs{e}_{\tau(i_k)}] } \\
& \labelrel={sym:eqq1}            \sum_{(i_1,\cdots,i_k) \in [d]^k}    \frac{1}{k!} \sum_{\tau \in S_k}  \inner{ \bm{\tilde T_k} \vert_{\tau(i_1) \cdots \tau(i_k) } }{ \bm{T_k}[\bs{e}_{i_1}, \cdots, \bs{e}_{i_k}] } = \inner{sym( \bm{\tilde T_k} )}{\bm{T_k}},
}
where  \eqref{sym:eqq0}  follows since $\bm{T_k}$ is symmetric,  and  \eqref{sym:eqq1}  follows by changing the indexing.
\end{proof}

\begin{lemma}
\label{lem:gradtensor}
Let $\bm{T_{j+k}} : (\R^d)^{\otimes (j + k)} \to \R$ be a symmetric tensor. We define $\bm{\grad^j T_{j+k} } : (\R^d)^{\otimes k}\to (\R^d)^{\otimes j}$ as
\eq{
\bm{\grad^j T_{j+k}}[\bs{e}_{i_1}, \cdots, \bs{e}_{i_k}] \vert_{i_{k+1}\cdots i_{k+j} } \coloneqq \bm{T_{j+k}}[\bs{e}_{i_1}, \cdots, \bs{e}_{i_k},  \bs{e}_{i_{k+1}}, \cdots, \bs{e}_{i_{k+j}}]. \label{eq:graddef}
}
We have  $\bm{\grad^j T_{j+k}}$ is symmetric and  $\norm{\bm{\grad^j T_{j+k}}}_F = \norm{ \bm{T_{j+k}} }_F$.
\end{lemma}

\begin{proof}
Both statements follow from definitions in \eqref{def:frobtensor} and \eqref{def:symtensor}.
\end{proof}

\begin{lemma}
\label{lem:tensorfroblb}
For $\bs{A} \in \R^{d \times r}$ and $\bm{T_k} : ( \R^{r} ) ^{\otimes k} \to \R$, let  $\bm{\hat T_k} : ( \R^{d} ) ^{\otimes k} \to \R$ such that  $\bm{\hat T_k}[\bs{u}_1, \cdots, \bs{u}_k] = \bm{T_k}[\bs{A}^\top \bs{u}_1, \cdots, \bs{A}^\top \bs{u}_k]$.  Then,  $\norm{\bm{\hat T_k} }_F \geq \sigma^k_r(\bs{A}) \norm{\bm{T_k}}_F.$
\end{lemma}

\begin{proof}
Let singular value decomposition of $\bs{A}$ be $\bs{A} \coloneqq \bs{U} \bm{\Sigma} \bs{L}^\top$, where $\bs{U} \in \R^{d \times r}$ and $\bs{L} \in \R^{r \times r}$ are orthonormal vectors and $\bm{\Sigma}_{ii} = \sigma_i(\bs{A})$ for $i \in [r].$  First, we observe that for any $\bs{v} \in \R^d$ such that $\bs{v} \perp col(\bs{U})$,   $\bs{A}^\top \bs{v} = 0$.  Since Frobenius norm of a tensor is independent of the choice of basis,  we can write that
\eq{
\norm{\bm{\hat T_k} }_F^2 =  \sum_{i_1,\cdots, i_k \in [r]^k} \bm{\tilde T_k} \left[ \bs{U}_{*i_1},  \cdots,  \bs{U}_{*i_k} \right]^2.
}
Hence, by definition
\eq{
\norm{\bm{\hat T_k} }_F^2  
=              \sum_{i_1,\cdots, i_k \in [r]^k}                \bm{T_k} \left[ \sigma_{i_1}(\bs{A})   \bs{L}_{*i_1},  \cdots, \sigma_{i_k}(\bs{A}) \bs{L}_{*i_k} \right]^2 
&\labelrel\geq{frob:ineqq0} \sigma_r^{2k}(\bs{A})             \sum_{i_1,\cdots, i_k \in [r]^k}  \bm{T_k} \left[   \bs{L}_{*i_1},  \cdots,   \bs{L}_{*i_k} \right]^2  \\
&=  \sigma_r^{2k}(\bs{A}) \norm{\bm{T_k}}_F^2,
}
where we use the multi-linear property of tensors in \eqref{frob:ineqq0}.
\end{proof}

\paragraph{Lemmas for Hermite Tensors}
\begin{definition}[Hermite Tensors]
\label{def:hermitetensor}
We define the Hermite tensor with a degree of $k$ as $\bm{\Hek}: \R^d \to  (\R^{d})^{\otimes k}$ as
\eq{ 
\bm{ \Hek }(\bs{x}) \vert_{i_1, \cdots, i_k}  \coloneqq e^{\frac{\norm{\bs{x}}_2^2}{2}} (-1)^k \frac{\partial^k }{\partial \bs{x}_{i_1}  \cdots \partial \bs{x}_{i_k}} \left( e^{\frac{ - \norm{\bs{x}}_2^2}{2}}   \right).
}
\end{definition}

\noindent
We use the following facts about Hermite tensors in our proofs.
\begin{lemma}
\label{lem:hermitemonomial}
For any orthonormal basis $\{ \bs{b}_1, \cdots,  \bs{b}_d \}$ and  $\bs{x} \in \R^d$,  we have  
\eq{
\inner{ \bm{ \Hek }(\bs{x})}{\bs{b}_{i_1} \otimes \cdots \otimes \bs{b}_{i_d}} = H_{e_{j_1}}(\inner{\bs{b}_1}{\bs{x}}) \cdots H_{e_{j_d}}(\inner{\bs{b}_d}{\bs{x}}), 
}
where  $j_l$ is the number of occurrences of $l \in [d]$ in $(i_1, \cdots, i_k)$, i.e., $j_l = \indic{i_1 = l} + \cdots + \indic{i_k = l}$. 
\end{lemma}

\begin{proof}
If $\{ \bs{b}_1, \cdots,  \bs{b}_d \}$ is the standard basis,  the statement follows from Definition \ref{def:hermitetensor}.  To extend it for any orthonormal basis,  let $\bs{B}$ denote the matrix with columns $\{ \bs{b}_1, \cdots, \bs{b}_d \}$,  let $h(\bs{x}) \coloneqq \exp \left(  - \norm{\bs{x}}_2^2/2 \right)$  and let $\grad^k h(\bs{x}) : (\R^d)^{\otimes k} \to \R$ represent the $k^{\text{th}}$ derivative of $h$.   We want to prove that for any $(i_1, \cdots, i_k) \in [d]^k$,    $\grad^k h(\bs{x})[\bs{B} \bs{e}_{i_1}, \cdots, \bs{B} \bs{e}_{i_k}] \stackrel{(*)}{=} \grad^k h(\bs{B}^\top \bs{x})[\bs{e}_{i_1}, \cdots, \bs{e}_{i_k}]$,  which will prove the statement.  We will use proof by induction.  We observe that $(*)$ holds for $k = 1.$  For $k > 1$, by assuming $(*)$ holds for $k-1$, we have
\eq{
\grad^k h(\bs{x})[\bs{B} \bs{e}_{i_1}, \cdots, \bs{B} \bs{e}_{i_k}] & = \lim_{t \to 0} \frac{ \left( \grad^{k-1} h(\bs{x} + t \bs{B} \bs{e}_{i_k}) - \grad^{k-1} h(\bs{x}) \right)[\bs{B} \bs{e}_{i_1}, \cdots, \bs{B} \bs{e}_{i_{k-1}}]}{t} \\
& = \lim_{t \to 0} \frac{ \left( \grad^{k-1} h(\bs{B}^\top \bs{x} + t \bs{e}_{i_k}) - \grad^{k-1} h(\bs{B}^\top \bs{x}) \right)[\bs{e}_{i_1}, \cdots, \bs{e}_{i_{k-1}}]}{t} \\
& = \grad^k h(\bs{B}^\top \bs{x})[\bs{e}_{i_1}, \cdots, \bs{e}_{i_k}].
}
\end{proof}

\begin{corollary}
\label{cor:hermitematrix}
Let  $\bs{V} \in \R^{d \times r}$ be an orthonormal matrix and $\bm{T_k} :    (\R^{r})^{\otimes k} \to \R$ be a symmetric $k$-tensor,  and  $\bm{H^{(r)}_{e_k}}$ and $\bm{H^{(d)}_{e_k}}$ denote $k$-degree Hermite tensor defined on $\R^r$ and $\R^d$ respectively.   For  $\bm{ \tilde T_k }[\bs{e}_{i_1}, \cdots, \bs{e}_{i_k}] \coloneqq  \bm{ T_k }[ \bs{V}^\top \bs{e}_{i_1}, \cdots, \bs{V}^\top \bs{e}_{i_k}]$,  we have  $ \inner{ \bm{T_k} }{\bm{H^{(r)}_{e_k}} (\bs{V}^\top \bs{x})} =  \inner{\bm{ \tilde T_k } }{\bm{H^{(d)}_{e_k}} (\bs{x})}$.
\end{corollary}

\begin{proof}
It immediately follows from Lemma \ref{lem:hermitemonomial}.
\end{proof}

\begin{lemma}
\label{lem:hermiteatzero}
We have $\bm{\Hek}(0) = (- i)^k \E_{\bs{w} \sim \cN(0,\ide{d})} \left[  \bs{w}^{\otimes k} \right]$, where $i = \sqrt{-1}$.  Consequently, we have $$\E_{\bs{w} \sim \cN(0,\ide{d})} \left[  \bs{w}^{\otimes 2k} \right] = (2k - 1)!! sym(\bs{I}_d^{\otimes k}).$$
\end{lemma}

\begin{proof}
See \cite[Eqs.  2.159 and 2.160]{Tao2012TopicsIR} and \cite[Lemma 22]{Damian2022NeuralNC}.
\end{proof}

\subsection{Hermite Expansion of  the Population Gradient}

For a symmetric $(k+1)$-tensor $\bm{T_{k+1}} :  ( \R^r )^{\otimes k+1} \to \R$,  we define a $k$-tensor $\bm{\grad T_{k+1}} :  ( \R^r )^{\otimes k} \to \R^r$ as in \eqref{eq:graddef} with $j = 1.$  For the following, we use the following notation: For $b \in \R$,
\eq{
\phi(\cdot + b) \coloneqq \sum_{k \geq 0} \frac{\rhc_k(b)}{k!} \Hek ~ \text{and} ~  \gt \coloneqq \sum_{k \geq 0} \frac{1}{k!} \inner{\bm{T_k}}{\bm{ \Hek }},
}
where $\rhc_k(b) \in \R$ and $\bm{T_k}$ is a symmetric $k$-tensor for $k \in \N$.
The main statement of this part is given below.

\begin{proposition}
\label{prop:hermiteexpansion}
For an orthonormal matrix $\bs{V} \in \R^{d \times r}$ and $\bs{w} \in S^{d-1}$, we have
\eq{
\E_{\bs{x}} [ \gt(\bs{V}^\top \bs{x}) \phi^\prime(\inner{\bs{w}}{\bs{x}}   +  b) \bs{x} ] & = \bs{V} \sum_{k \geq 0} \tfrac{\rhc_{k+1}(b)}{k!} \bm{ \grad T_{k+1} }     \left[ (\bs{V}^\top \bs{w})^{ \otimes k}  \right]  + \bs{w} \sum_{k \geq 0} \tfrac{\rhc_{k+2}(b)}{k!} \bm{ T_k }     \left[   (\bs{V}^\top \bs{w})^{ \otimes k}   \right]  \label{eq:expansion}
}
and
\eq{
\rhc_{k}(b) =  \begin{cases}
1 - \Phi(- b),  & k = 1 \\
\frac{e^{\frac{- b^2}{2}}}{\sqrt{2 \pi}} H_{e_{k-2}}(- b),  & k \geq 2
\end{cases}  \label{eq:constantterms}
}
where  $\Phi(b)$ is the CDF of the standard Gaussian distribution.  
\end{proposition}

To prove Proposition \ref{prop:hermiteexpansion}, we will need two lemmas.
\begin{lemma}
\label{lem:hermitetensorlem1}
For $\bs{w} \in \R^d$ and $k \in \N$, let 
$\bm{T_k} \coloneqq k \: sym( \bs{e}_l  \otimes  \bs{w}^{\otimes k- 1} )$.
For $i_1, \cdots, i_k \in [d],$ we have  $\bm{T_k} \vert_{i_1 \cdots i_k} = j_l  \bs{w}_1^{j_1} \times \cdots \times \bs{w}_{l}^{j_l - 1}  \times \cdots \times \bs{w}_d^{j_d}$,   where  $j_l = \indic{i_1 = l} + \cdots + \indic{i_k = l}$.
\end{lemma}

\begin{proof}
We have $\bm{T_k} \stackrel{(*)}{=}  \bs{e}_l  \otimes  \bs{w}^{\otimes k- 1}  + \bs{w} \otimes \bs{e}_l  \otimes  \bs{w}^{\otimes k- 2}  + \bs{w}^{\otimes 2} \otimes \bs{e}_l  \otimes  \bs{w}^{\otimes k- 3}  + \cdots + \bs{w}^{\otimes k - 1} \otimes \bs{e}_l$.   Without loss of generality, we can assume $j_l > 0$ and $i_1, \cdots, i_{j_l} = l$ (since for $j_l = 0$, the statement is true).  The statement follows from  $(*)$ since in the right-hand side only $j_l$ terms will be nonzero and the other terms will be equal to  $\bs{w}^{\otimes k -1} \vert_{i_2, \cdots, i_k} = \bs{w}_1^{j_1} \times \cdots  \times \bs{w}_{l}^{j_l - 1}  \times \cdots \times \bs{w}_d^{j_d}$.
\end{proof}

\begin{lemma}
For $\bs{w} \in S^{d-1}$, $l \in [d]$ and $k \in \N$, we have  $\E \left[ \phi^\prime(\inner{\bs{w}}{\bs{x}} + b) \bs{x}_l \bm{\Hek}(\bs{x}) \right] =  \rhc_{k + 2}(b) \bs{w}_l \bs{w}^{ \otimes k} + \rhc_{k}(b) k \: sym(\bs{e}_l \otimes \bs{w}^{\otimes k-1})$.
\end{lemma}

\begin{proof}
We recall that  $\bm{ \Hek }(\bs{x}) \vert_{i_1 \cdots i_k} = H_{e_{j_1}}(\bs{x}_1) \cdots H_{e_{j_d}}(\bs{x}_d),$  where $j_l = \indic{i_1 = l} + \cdots + \indic{i_k = l}$.  The for any fixed $(i_1, \cdots, i_k) \in [d]^k$,
\eq{
\E \left[ \phi^\prime(\inner{\bs{w}}{\bs{x}} + b) \bs{x}_l \bm{ \Hek }(\bs{x}) \vert_{i_1 \cdots i_k}   \right]  &=  \E \left[ \phi^\prime(\inner{\bs{w}}{\bs{x}} + b) H_{e_{j_1}}(\bs{x}_1) \cdots H_{e_{j_l + 1}} (\bs{x}_l) \cdots H_{e_{j_d}}(\bs{x}_d)  \right]  \\
& + j_l  \E \left[ \phi^\prime(\inner{\bs{w}}{\bs{x}} + b) H_{e_{j_1}}(\bs{x}_1) \cdots H_{e_{j_l - 1}} (\bs{x}_l) \cdots H_{e_{j_d}}(\bs{x}_d)  \right]  \\
& =  \rhc_{k+2}(b) \bs{w}_1^{j_1} \cdots  \bs{w}_i^{j_l + 1} \cdots \bs{w}_d^{j_d}      +      \rhc_{k}(b) j_l \bs{w}_1^{j_1} \cdots  \bs{w}_i^{j_l - 1} \cdots \bs{w}_d^{j_d}  \\
& =  \rhc_{k + 2}(b) \bs{w}_l \bs{w}^{\otimes k} \vert_{i_1 \cdots i_k} +  \rhc_{k}(b) k \:  sym(\bs{e}_l \otimes \bs{w}^{\otimes k-1}) \vert_{i_1 \cdots i_k},
}
where  we use Lemma \ref{lem:hermitetensorlem1} in the last line.
\end{proof}

\begin{proof}[Proof of Proposition \ref{prop:hermiteexpansion}]
We fix $l \in [d]$. Since $\E[\phi(Z)^4] < \infty$, we have 
\eq{
\E [ \gt(\bs{V}^\top \bs{x}) \phi^\prime(\inner{\bs{w}}{\bs{x}} + b) \bs{x}_l ] &   =  \sum_{k = 0}^{\infty} \frac{1}{k!} \E\left[  \inner{\bm{T_k}}{\bm{\Hek} (\bs{V}^\top \bs{x}) }  \phi^\prime(\inner{\bs{w}}{\bs{x}} + b) \bs{x}_l  \right] \\
& =  \sum_{k = 0}^{\infty}   \frac{1}{k!} \inner{\bm{\tilde T_k}}{  \E\left[ \bm{\Hek} (\bs{x})   \phi^\prime(\inner{\bs{w}}{\bs{x}} + b) \bs{x}_l  \right] }, \label{eq:eq2}
}
where $\bm{\tilde T_k}$ is defined in  Corollary \ref{cor:hermitematrix}.
For a fixed $k \in \N$, we have
\eq{
\inner{\bm{\tilde T_k}}{  \E     \left[ \bm{\Hek} (\bs{x})   \phi^\prime(\inner{\bs{w}}{\bs{x}} + b) \bs{x}_l  \right] }
     & \labelrel={exp:eqq0}   \rhc_{k + 2}(b) \bs{w}_l \inner{\bm{\tilde T_k}}{ \bs{w}^{\otimes k}} + \rhc_{k}(b)  k \inner{\bm{\tilde T_k}}{ \bs{e}_l \otimes \bs{w}^{\otimes k - 1}} \\
& =      \rhc_{k + 2}(b)  \bs{w}_l \bm{T_{k}} \left[  (\bs{V}^\top \bs{w})^{\otimes k} \right]     +    \rhc_{k}(b)  k \ \bs{V}_{l*}^\top  \bm{\grad  T_k} \left[ (\bs{V}^\top \bs{w})^{\otimes k  - 1} \right],   \label{eq:eq3} 
}
where \eqref{exp:eqq0}  follows  by Proposition \ref{prop:symmtensor} since $\bm{ \tilde T_k}$ symmetric.
\eqref{eq:expansion} follows from \eqref{eq:eq2} and \eqref{eq:eq3}.  For  \eqref{eq:constantterms}, see \cite[Lemma 15]{ba2023learning}.
\end{proof}

\begin{corollary}
\label{cor:hermiteevenodd}
Let $\phi_{\pm} (t,;b) \coloneqq \frac{\phi(t + b) \pm \phi(- t + b)}{2}$. We have
\begin{enumerate}[label=(\roman*), leftmargin=0em]
\item[]  $  \E [ \gt(\bs{V}^\top \bs{x}) \phi_{+}^\prime(\inner{\bs{w}}{\bs{x}}   ;   b) \bs{x} ]   =   \bs{V}       \sum_{\small \substack{ k \geq 1 \\ k ~ odd}}       \tfrac{\rhc_{k+1} (b)}{k!} \bm{ \grad T_{k+1} } \left[ (\bs{V}^\top \bs{w})^{ \otimes k}  \right] + \bs{w}      \sum_{\small \substack{ k \geq 0 \\ k \: even}}      \tfrac{\rhc_{k+2} (b)}{k!} \bm{ T_k } \left[   (\bs{V}^\top \bs{w})^{ \otimes k}   \right]$
\item[] $ \E [ \gt(\bs{V}^\top \bs{x}) \phi_{-}^\prime(\inner{\bs{w}}{\bs{x}}    ;    b) \bs{x} ]    =    \bs{V}        \sum_{\small \substack{ k \geq 0 \\ k ~ even}}       \tfrac{\rhc_{k+1}(b)}{k!} \bm{ \grad T_{k+1} } \left[ (\bs{V}^\top \bs{w})^{ \otimes k}  \right] + \bs{w}        \sum_{\small \substack{ k \geq 1 \\ k ~ odd}}        \tfrac{\rhc_{k+2} (b)}{k!} \bm{ T_k } \left[   (\bs{V}^\top \bs{w})^{ \otimes k}   \right]   $
\end{enumerate} 
\end{corollary}

\begin{proof}
We observe that  $\phi_{+}(\cdot + b) = \sum_{\substack{ k \geq 0 \\ k \: even}} \frac{\rhc_k(b)}{k!} \Hek$ and $\phi_{-}(\cdot + b) = \sum_{\substack{ k \geq 0 \\ k \: odd}} \frac{\rhc_k(b)}{k!} \Hek$. By the argument in \eqref{eq:eq2} and \eqref{eq:eq3}, the statement follows.
\end{proof}

\subsection{Bounding the Higher Order Terms in the Hermite Expansion}

\begin{proposition}
\label{prop:residualbound}
For $N \in \N \cup \{ -1, 0 \}$,  $\bs{w} \in S^{d-1}$ and $b \in \R$,  let
\eq{
\zeta_N    \coloneqq    \E \Big[  \gt(\bs{V}^\top \bs{x}) &\phi^\prime(\inner{\bs{w}}{\bs{x}} + b) x  \Big]   -  \bs{V}  \sum_{k = 0}^N \tfrac{\rhc_{k+1}(b)}{k!} \bm{ \grad T_{k+1} } \left[ (\bs{V}^\top \bs{w})^{ \otimes k}  \right]   -   \bs{w} \sum_{k = 0}^N \tfrac{\rhc_{k+2}(b)}{k!} \bm{ T_k } \left[   (\bs{V}^\top \bs{w})^{ \otimes k}   \right]. 
}
We have
\eq{
\norm*{ \zeta_N }_2 \leq  (1+ \sqrt{N+2}) \cgp   \begin{cases}
\frac{\norm{\bs{V}^\top \bs{w}}_2^{N+1}}{ 1 - \norm{\bs{V}^\top \bs{w}}^2} &  \norm{\bs{V}^\top \bs{w}}_2 > 0 ~ \text{or}~  N \geq 0 \\
1 & \text{otherwise}.
\end{cases}
}
\end{proposition}

\begin{proof}[Proof of Proposition \ref{prop:residualbound}]
By Proposition \ref{prop:hermiteexpansion}, we know that
\eq{
\zeta_N =  \bs{V} \sum_{k \geq N+1} \frac{\rhc_{k+1}(b)}{k!} \bm{ \grad T_{k+1} } \left[ (\bs{V}^\top \bs{w})^{ \otimes k}  \right] + \bs{w} \sum_{k \geq N+1} \frac{\rhc_{k+2}(b)}{k!} \bm{ T_k } \left[   (\bs{V}^\top \bs{w})^{ \otimes k}   \right]. 
}
Therefore,  
\eq{
 \norm*{ \zeta_N }_2
&  \labelrel={resbound:eqq1}     \norm*{ \sum_{k \geq N+1}       \frac{\rhc_{k+1}(b)}{k!} \bm{ \grad T_{k+1} } \left[ (\bs{V}^\top \bs{w})^{ \otimes k}  \right]}_2 + \abs*{\sum_{k \geq N+1}       \frac{\rhc_{k+2}(b)}{k!} \bm{ T_k } \left[   (\bs{V}^\top \bs{w})^{ \otimes k}   \right]}   \label{eq:eq325} \\
&    \labelrel\leq{resbound:ineqq3}         \left( \sum_{k \geq N+1}             \frac{\rhc^2_{k+1}(b) \norm*{\bs{V}^\top \bs{w}}_2^{2k}  }{k!} \right)^{\frac{1}{2}}       \left( \sum_{k \geq N+1}       \frac{1 }{k!}  \norm*{ \bm{ \grad T_{k+1} }   \left[  \left( \frac{\bs{V}^\top \bs{w}}{\norm{\bs{V}^\top \bs{w}}_2} \right)^{ \otimes k}  \right] }^2_2 \right)^{\frac{1}{2}}   \\
&\qquad  \qquad  \qquad +   \left( \sum_{k \geq N+1}             \frac{\rhc^2_{k+2}(b) \norm*{\bs{V}^\top \bs{w}}_2^{2k}  }{k!} \right)^{\frac{1}{2}} \left( \sum_{k \geq N+1}         \frac{1 }{k!} \bm{ T_{k} }   \left[  \left( \frac{\bs{V}^\top \bs{w}}{\norm{\bs{V}^\top \bs{w}}_2} \right)^{ \otimes k}  \right]^2 \right)^{\frac{1}{2}}   \\
&   \labelrel\leq{resbound:ineqq4}     \left( \sum_{k \geq N+1}            \frac{\rhc^2_{k+1}(b) \norm*{\bs{V}^\top \bs{w}}_2^{2k}  }{k!} \right)^{\frac{1}{2}}         \E[\norm{\grad \gt(\bs{z})}_2^2]^{\frac{1}{2}}        +         \left( \sum_{k \geq N+1}             \frac{\rhc^2_{k+2}(b) \norm*{\bs{V}^\top \bs{w}}_2^{2k}  }{k!} \right)^{\frac{1}{2}}       \E[\gt(\bs{z})_2^2]^{\frac{1}{2}} ~~~~ \label{eq:eq4}
}
where we use that $\bs{V}$ is orthonormal  and $\bs{w}$ is a unit vector  in \eqref{resbound:eqq1}, the multi-linear property of tensors and Cauchy-Schwartz inequality for \eqref{resbound:ineqq3}, and  Parseval's identity for  \eqref{resbound:ineqq4}. We observe that for  $\norm{\bs{V}^\top \bs{w}}_2 > 0$ or $N \geq 0$
\eq{ 
\sum_{k \geq N+1}         \frac{\rhc^2_{k+1}(b) \norm*{\bs{V}^\top \bs{w}}_2^{2k}  }{k!}  \leq \left( \sup_{k \geq N+1} \frac{\rhc^2_{k+1}(b)}{k!}  \right) \sum_{k \geq N+1}  \norm*{\bs{V}^\top \bs{w}}_2^{2k}   \leq   \frac{\norm{\bs{V}^\top \bs{w}}_2^{2(N+1)}}{1 - \norm{\bs{V}^\top \bs{w}}^2_2}  \label{eq:boundterm1}
}
and
\eq{  
\sum_{k \geq N+1}         \frac{\rhc^2_{k+2}(b) \norm*{\bs{V}^\top \bs{w}}_2^{2k}  }{k!}   \leq  \left( \sup_{k \geq N+ 1} \frac{\rhc^2_{k+2}(b)}{(k+1)!} \right) \sum_{k \geq N+1}    (k+1)  \norm*{\bs{V}^\top \bs{w}}_2^{2k}  
\leq    \frac{(N+2) \norm{\bs{V}^\top \bs{w}}^{2(N+1)}}{(1 - \norm{\bs{V}^\top \bs{w}}^2)^2} \label{eq:boundterm2}
}
where we used $\sum_{k \geq 0} \tfrac{\rhc^2_{k+1}(b)}{k!} =  \E[\phi^\prime(Z + b)] \leq 1$ and the sum formula for  $\sum_{k \geq k^{\star}} k z^{k+1}$.  Since  $ \E[\gt(\bs{z})_2^2]  \leq   \E[\norm{\grad \gt(\bs{z})}_2^2]^{1/2} = \cgp$ and $\norm{\bs{V}^\top \bs{w}}_2 \leq 1$, we have
\eq{
\eqref{eq:eq4} 
\leq (1 + \sqrt{N+2}) \cgp   \frac{\norm{\bs{V}^\top \bs{w}}^{N+1}}{1 - \norm{\bs{V}^\top \bs{w}}^2}.  \label{eq:boundterm33}
}
For $\norm{\bs{V}^\top \bs{w}}_2 > 0$ or $N \geq 0$ do not hold, we observe that the right-hand-side of both   \eqref{eq:boundterm1}- \eqref{eq:boundterm2} is $1$. Therefore, by the argument in   \eqref{eq:boundterm33}, the statement follows in this case too.
\end{proof}

\begin{corollary}
\label{cor:residualevenodd}
Let $\phi_{\pm}$ be the functions introduced in Corollary \ref{cor:hermiteevenodd}.   For For $N \in \N \cup \{ -1, 0 \}$,  $\bs{w} \in S^{d-1}$ and $b \in \R$,  let
\eq{
&  \zeta^{+}_N      \coloneqq     \E   \left[  \gt(\bs{V}^\top \bs{x}) \phi_{+}^\prime(\inner{\bs{w}}{\bs{x}}    ; b) \bs{x}  \right]  - \bs{V}    \sum_{\small \substack{k = 0 \\ k ~ odd}}^N    \tfrac{\rhc_{k+1}(b)}{k!} \bm{ \grad T_{k+1} } \left[ (\bs{V}^\top \bs{w})^{ \otimes k}  \right] - \bs{w}  \sum_{\small \substack{k = 0 \\ k ~ even}}^N         \tfrac{\rhc_{k+2}(b)}{k!} \bm{ T_k } \left[   (\bs{V}^\top \bs{w})^{ \otimes k}  \right],  \\ 
&   \zeta^{-}_N    \coloneqq    \E   \left[  \gt(\bs{V}^\top \bs{x}) \phi_{-}^\prime(\inner{\bs{w}}{\bs{x}}     ;  b) \bs{x}  \right] - \bs{V}  \sum_{\small \substack{k = 0 \\ k ~ even}}^N      \tfrac{\rhc_{k+1}(b)}{k!} \bm{ \grad T_{k+1} } \left[ (\bs{V}^\top \bs{w})^{ \otimes k}  \right] - \bs{w}    \sum_{\small \substack{k = 0 \\ k ~ odd}}^N      \tfrac{\rhc_{k+2}(b)}{k!} \bm{ T_k } \left[   (\bs{V}^\top \bs{w})^{ \otimes k}  \right].
} 
We have
\eq{
\norm{\zeta^{\pm}_N}_2 \leq (1 + \sqrt{N+2})   \cgp \begin{cases}
\frac{\norm{\bs{V}^\top \bs{w}}_2^{N+1}}{ 1 - \norm{\bs{V}^\top \bs{w}}^2} &  \norm{\bs{V}^\top \bs{w}}_2 > 0 ~ \text{or}~  N \geq 0 \\
1 & \text{otherwise}
\end{cases}
}
\end{corollary}

\begin{proof}
The statement follows from $\E[\phi^{\prime}_{\pm} (Z + b)^2 ] \leq 1$ and Proposition \ref{prop:residualbound} (see \eqref{eq:boundterm1} and \eqref{eq:boundterm2}).
\end{proof}

\subsection{Bounding $\ell_q$ Norm of the Higher-Order Terms}
\begin{proposition}
\label{prop:residualqnorm}
By using the notation of Proposition \ref{prop:residualbound} and Corollary \ref{cor:residualevenodd},  for $\bs{w} \in S^{d-1}$, $N \in \N \cup \{-1, 0 \}$ and $q \in [0,2)$, we have
\eq{
 \norm{\zeta_N}_q^q \vee \norm{\zeta^{\pm}_N}_q^q \leq 2^{{  (q-1) \vee 0}} \cgp^q  \left[ \norm{\bs{V}}_{2,q}^q + (N+2)^{\frac{q}{2}} \norm{\bs{w}}_q^q   \right]  
\begin{cases}
\left( \frac{\norm{\bs{V}^\top \bs{w}}_2^{N+1}}{1 - \norm{\bs{V}^\top \bs{w}}_2^2}  \right)^q   &  \norm{\bs{V}^\top \bs{w}}_2   >   0 ~ \text{or}~  N   \geq    0 \\
1 & \text{otherwise}.
\end{cases}  
}
\end{proposition}

\begin{proof}
By Propositions \ref{prop:hermiteexpansion} and \ref{prop:lqtri},   if $\norm{\bs{V}^\top \bs{w}}_2 > 0$ or $N \geq 0$ hold,  we have
\eq{
\norm*{ \zeta_N }_q^q     
& \labelrel\leq{resq:ineqq0}       2^{(q-1) \vee 0}      \left(       \norm{\bs{V}}_{2,q}^q \norm*{ \sum_{\small k \geq N+1}           \tfrac{\rhc_{k+1}(b)}{k!} \bm{ \grad T_{k+1} }     \left[ (\bs{V}^\top \bs{w})^{ \otimes k}     \right]    }^q_2         +    \norm{\bs{w}}_q^q  \abs*{\sum_{\small k \geq N+1}         \tfrac{\rhc_{k+2}(b)}{k!} \bm{ T_k } \left[   (\bs{V}^\top \bs{w})^{ \otimes k}   \right]   }^q \right)   \\
&  \labelrel\leq{resq:ineqq1} 2^{(q-1) \vee 0}  \norm{\bs{V}}_{2,q}^q \cgp^q   \left( \frac{\norm{\bs{V}^\top \bs{w}}_2^{N+1}}{1 - \norm{\bs{V}^\top \bs{w}}_2^2}  \right)^q  + 2^{(q-1) \vee 0}  \norm{\bs{w}}_q^q  \cgp^q  \left( \frac{\sqrt{N+2} \norm{\bs{V}^\top \bs{w}}_2^{N+1}}{1 - \norm{\bs{V}^\top \bs{w}}_2^2}  \right)^q   \\
& =  2^{(q-1) \vee 0} \cgp^q \left( \frac{\norm{\bs{V}^\top \bs{w}}_2^{N+1}}{1 - \norm{\bs{V}^\top \bs{w}}_2^2}  \right)^q \left[ \norm{\bs{V}}_{2,q}^q + (N+2)^{\frac{q}{2}} \norm{\bs{w}}_q^q   \right],
}
where  \eqref{resq:ineqq0} follows $\norm{\bs{V} \bs{u}}_q^q \leq \norm{\bs{V}}_{2,q}^q \norm{\bs{u}}_2^q$ and  \eqref{resq:ineqq1} follows the steps in \eqref{eq:eq325}- \eqref{eq:boundterm2}.  For   $\norm{\zeta^{\pm}_N}_q^q$,  the same argument applies.   if neither $\norm{\bs{V}^\top \bs{w}}_2 > 0$ nor $N \geq 0$ hold, since we can replace  $\frac{\sqrt{N+2} \norm{\bs{V}^\top \bs{w}}_2^{N+1}}{1 - \norm{\bs{V}^\top \bs{w}}_2^2}$ in \eqref{resq:ineqq1} with $1$, the statement follows in this case as well.
\end{proof}

\section{Concentration Bound for Empirical Gradients}

In this part, we derive a concentration bound for the empirical gradient 
\eq{
g(\bs{w},b) \coloneqq \frac{1}{n} \sum_{i = 1}^n \left( y_i - \inner{\hmu \vert_{\cJ}}{\bs{x}_i} \right)  \bs{x}_i \phi^\prime \left(\inner{\bs{w}}{\bs{x}_i} + b  \right),  \label{empgrad:epmgraddef}
}
where $\hmu = 0$ in the single index setting and $\hmu = \frac{1}{n} \sum_{j =1}^n y_j \bs{x}_j$  in the multi index setting.  In the following,  to avoid repetitions, we will consider \eqref{empgrad:epmgraddef} with $\phi(t) \in \phiset$ and  particularly with $\hmu = \frac{1}{n} \sum_{j =1}^n y_j \bs{x}_j$.   Our proof will give us a bound for the $\hmu = 0$ case as well.

To handle dependencies between $\{(\bs{x}_i, y_i) \}_{i = 1}^n$ and $\cJ$, we will consider the following process: For  $\theta \coloneqq (\bs{w},b) \in S_M^{d-1} \times \R$,   
\eq{
\bs{T}_\theta & \coloneqq  g(\theta)   - \E_{(\bs{x}, y)} \left[  \overline y \bs{x} \phi^\prime(\inner{\bs{w}}{\bs{x}} + b)  \right] \\
&= \frac{1}{n} \sum_{i = 1}^n \tilde y_i \bs{x}_i \phi^\prime(\inner{\bs{w}}{\bs{x}_i} + b) - \E_{(\bs{x},y)} \left[  \overline y \bs{x} \phi^\prime(\inner{\bs{w}}{\bs{x}} + b)  \right],  
}
where $(\bs{x}, y)$ is a generic data point that is independent of $\{ (\bs{x}_i, y_i) \}_{i = 1}^n$ and  
\eq{ 
\tilde y_i = y_i - \inner{\hmu \vert_{\cJ}}{\bs{x}_i}  ~~ \text{and} ~~ \overline y = y - \inner{\E[y\bs{x}] \vert_{\cJ}}{\bs{x}}. \label{eq:defys}
}
We particularly derive a concentration bound for
\eq{
 \sup_{\substack{\cJ \subseteq [d] \\ \abs{\cJ} = M^\prime}} \sup_{ \pspace} \norm{\bs{T}_\theta \vert_{ \cJ} }_2, \label{eq:targetstatement}
}
where $M, M^\prime \in [d]$, and the restriction sets in \eqref{eq:defys} and \eqref{eq:targetstatement}, i.e.,  $\cJ$,  are the same.   We observe that  for a fixed $(\bs{w},b) \in S_M^{d-1} \times \R$,
\eq{
\bs{T}_\theta  &= \left(  \frac{1}{n} \sum_{i = 1}^n y_i \bs{x}_i \phi^\prime(\inner{\bs{w}}{\bs{x}_i} + b) - \E \left[ y \bs{x} \phi^\prime(\inner{\bs{w}}{\bs{x}} + b) \right]   \right)  \\
& - \left(  \frac{1}{n} \sum_{i = 1}^n \inner{ \E[y \bs{x}] \vert_{\cJ} }{\bs{x}_i} \bs{x}_i \phi^\prime(\inner{\bs{w}}{\bs{x}_i} + b) - \E_{(\bs{x}, y)} \left[  \inner{\E[y \bs{x}] \vert_{\cJ}}{\bs{x}} \bs{x} \phi^\prime(\inner{\bs{w}}{\bs{x}} + b) \right]   \right)    \\
& -   \left(  \frac{1}{n} \sum_{i = 1}^n \inner{  ( \hmu - \E[y \bs{x}] ) \vert_{\cJ} }{\bs{x}_i} \bs{x}_i \phi^\prime(\inner{\bs{w}}{\bs{x}_i} + b) - \E \left[  \inner{   ( \hmu - \E[y \bs{x}] ) \vert_{\cJ} }{\bs{x}} \bs{x} \phi^\prime(\inner{\bs{w}}{\bs{x}} + b) \right]   \right)   \\
& - \E \left[  \phi^\prime(\inner{\bs{w}}{\bs{x}} + b) \bs{x} \bs{x}^\top  \right] (\hmu - \E[y \bs{x}])\vert_{\cJ}. \label{eq:twdecomposition}
} 
Let
\eq{
& \bs{Y}_\theta \coloneqq \frac{1}{n} \sum_{i = 1}^n y_i \bs{x}_i \phi^\prime(\inner{\bs{w}}{\bs{x}_i} + b)  - \E_{(\bs{x}, y)} \left[ y  \bs{x} \phi^\prime(\inner{\bs{w}}{\bs{x}} + b)  \right] ,  \\
& \bm{\Sigma}_\theta \coloneqq \frac{1}{n} \sum_{i = 1}^n \phi^\prime(\inner{\bs{w}}{\bs{x}_i} + b) \bs{x}_i \bs{x}_i^\top - \E_{\bs{x}} \left[ \phi^\prime(\inner{\bs{w}}{\bs{x}} + b) \bs{x} \bs{x}^\top   \right].
}
Then, we can write
\eq{
\bs{T}_\theta   \vert_{\cJ}      =    \bs{Y}_{\theta}   \vert_{\cJ} -  \bm{\Sigma}_{\theta} \cJJ \E [y \bs{x}]   \vert_{\cJ}  - \left(  \bm{\Sigma}_\theta \cJJ+ \E  \Big[  \phi^\prime(\inner{\bs{w}}{\bs{x}} + b) \bs{x} \bs{x}^\top  \right]   \cJJ \Big)  (\hmu - \E[y \bs{x}])\vert_{\cJ}.  ~~~~ \label{eq:targetstatement2}
}
In the following, we derive concentration bounds for $\bs{Y}_\theta$ and $\bm{\Sigma}_\theta$, which will lead us a bound for \eqref{eq:targetstatement2}. Our proof technique relies on the use of Radamacher averages with an extension of the symmetrization lemma for the moment-generating function,  which is presented as follows:

\begin{lemma}
\label{lem:generalsymmetrization}
Let   $\bs{X}_1, \cdots, \bs{X}_n \in \R^d$ be independent random vectors and  let  $\{ \varepsilon_i \}_{i \in [n]}$ be iid Radamacher random variables,  independent of $\{ \bs{X}_i \}_{i \in [n]}.$  For $\ell: \R^d \times S_M^{d-1} \times \R,$ $\lambda  > 0$  and $h(t) \in \{ t, \exp(t)\}$, we have
\eq{
\E \left[ h \left(  \sup_{ \pspace } \frac{\lambda}{n} \sum_{i = 1}^n \ell(\bs{X}_i, (\bs{w}, b))  - \E[ \ell(\bs{X}, (\bs{w}, b))  ]  \right) \right] \leq \E    \left[  \sup_{  \pspace }        h \left( \frac{2\lambda}{n} \sum_{i=1}^n \varepsilon_i \ell(\bs{X}_i, (\bs{w}, b))  \right)  \right].
}
\end{lemma}

\begin{proof}
Let $Z \coloneqq   \sup_{ \bs{w}, b  }  \frac{1}{n} \sum_{i = 1}^n \ell(\bs{X}_i, (\bs{w}, b))  - \E[ \ell(\bs{X}, (\bs{w}, b)) ]$.  By using Jensen's inequality,  one can show that for any convex and nondecreasing function $h$, $$ \E[h(Z)] \leq \E \left[\sup_{\bs{w},b} h \left(\frac{2}{n} \sum_{i =1}^n \varepsilon_i \ell \left(\bs{X}_i, (\bs{w},b) \right) \right) \right].$$ Since $t \to h(\lambda t)$,  where $h(t) \in \{ t, \exp(t)\}$ and $\lambda > 0$, is convex and nondecreasing, the statement follows.
\end{proof}

\subsection{VC Dimension of $\{ \cdot \to \phi^\prime(\inner{\vect{w}}{\cdot} + b)$; $(\vect{w},b) \in S_{M}^{d-1} \times \R \} $}
Let  $\cF_M \coloneqq \{  \cdot \to \phi^\prime(\inner{\bs{w}}{\cdot} + b) ~ \vert ~ (\bs{w},b) \in S_{M}^{d-1} \times  \R  \}$.
We want to bound the VC dimension of $\cF_M$.

\begin{proposition}
\label{prop:vcdim}
Let $VC(\cF_M) = d^*$. We have $M \leq d^* \leq 6 M \log \left( \tfrac{ed}{M}  \right)$.
\end{proposition}

\begin{proof}
Let $\cF^{(d)} \coloneqq \{  \cdot \to \phi^\prime(\inner{\bs{w}}{\cdot} + b) ~ \vert ~ (\bs{w},b) \in S^{d-1} \times  \R  \}$ and $s(\cF^{(d)}, n)$ be the shattering coefficient of $\cF^{(d)}$.   Since $VC(\cF^{(d)}) = d + 1$,  we have $M+1 \leq d^* \leq d + 1$. 

To improve the upper bound,  we observe that  $S_{M}^{d-1}$ has $\binom{d}{M}$ different possible support,  hence,  we have $s(\cF_M,n) \leq   \binom{d}{M}  S(\cF^{(M)}, n)$. Then, by definition of VC dimension,
\eq{
s(F_M,d^*) = 2^{d^*} \leq  \binom{d}{M} s(\cF^{(M)}, d^*)   \labelrel\leq{vc:ineqq0}   \binom{d}{M} s(\cF^{(M)}, d + 1)  
& \labelrel\leq{vc:ineqq1}      \binom{d}{M} \left( \frac{e (d+1)}{(M+ 1)}  \right)^{(M+1)}  \\
& \labelrel\leq{vc:ineqq2}    \left( \frac{ed}{M}  \right)^{2M + 1}. 
}
where we use  $d^* \leq d + 1$ in \eqref{vc:ineqq0},   Sauer's lemma in  \eqref{vc:ineqq1}, and $\binom{d}{M} \leq   \left( \frac{e d}{M}  \right)^M$ and $(d+1)/(M+1) \leq d/M$ in \eqref{vc:ineqq2}.  By observing that $e^{d^*/2} \leq 2^{d^*}$ and $4M+2\leq 6M$ , we obtain the upper bound as well.
\end{proof}

\begin{corollary}
\label{cor:vcdimimp}
Let $n \geq  d^*$. For any $\bs{x}_1, \cdots, \bs{x}_n \in \R^d$, there exists $Q^x \subset S_{M}^{d-1} \times \R$ and $\pi : S_M^{d-1} \times \R \to Q^x$ with $\abs{Q^x} \leq \left( \frac{en}{d^*} \right)^{d^*}$ such that for any $(\bs{w},b) \in S_{M}^{d-1} \times \R$,  $\phi^\prime(\inner{\bs{w}}{\bs{x}_i} + b) =   \phi^\prime(\inner{\pi( (\bs{w}, b))}{( \bs{x}_i, 1)})$ for $i = 1, \cdots, n$.
\end{corollary}

\begin{proof}
By Sauer's lemma, the image of $\Phi( (\bs{w},b)) \coloneqq \big( \phi^\prime(\inner{\bs{w}}{\bs{x}_1} + b), \cdots, \phi^\prime(\inner{\bs{w}}{\bs{x}_n} + b) \big)$,  $(\bs{w},b) \in S_{M}^{d-1} \times \R$, has at most $\left( en/d^* \right)^{d^*}$ elements. We can define $Q_x$ by mapping each $(\bs{w}, b) \in S_{M}^{d-1} \times \R$ to a fixed $(\bs{w}^\prime, b^\prime)$ such that $\Phi( (\bs{w}, b) ) = \Phi( (\bs{w}^\prime, b^\prime))$.
\end{proof}

\subsection{Concentration for  $\vect{Y}_{\theta}$}

In this section, we derive a concentration bound for
\eq{
\sup_{\substack{\cJ \subseteq [d] \\ \abs{\cJ} = M^\prime  }} \sup_{\substack{ \bs{w} \in S^{d-1}_M \\ b \in \R }} \norm*{ \bs{Y}_\theta \vert_{\cJ} }_2,
}
We will prove our bound in two steps. First, we will prove a bound for the truncated version $\bs{Y}_\theta$. In its following, we will extend that result by bounding the bias introduced by truncation.
\paragraph{Concentration of the truncated process:}
For  some $R > 0$ and $\bs{v} \in S^{d-1}$ and $\theta = (\bs{w},b) \in S_M^{d-1} \times \R$, we let
\eq{
\tilde{\bs{Y}}_{\theta,v} \coloneqq \frac{1}{n} \sum_{i = 1}^n y_i \indic{\abs{y_i} \leq R } \inner{\bs{v}}{\bs{x}_i} \phi^\prime(\inner{\bs{w}}{\bs{x}_i} + b)  -\E \left[ y \indic{\abs{y} \leq R } \inner{\bs{v}}{\bs{x}} \phi^\prime(\inner{\bs{w}}{\bs{x}} + b)  \right].
}

\begin{lemma}
\label{lem:supprocessaux}
For $\phi(t) \in \phiset$, $n \geq d^*$ and $t \geq 0$,  we have 
\eq{
\mpr \left[ \sup_{\theta \in S_M^{d-1} \times \R} \tilde{\bs{Y}}_{\theta,v}   \geq  8 R \max \{ t, t^2 \} \right] \leq  \left( \frac{e n}{d^*} \right)^{d^*} \exp \left(  - n t^2  \right).
}
\end{lemma}

\begin{proof}
In the following, we will use that $\abs{\phi^\prime} \leq 1$ and $VC(\cF_M) \leq d^*$, where $d^*$ is defined in Proposition \ref{prop:vcdim}. We note that both hold for $\phi(t) \in  \phiset$.  Let 
\eq{
\ell \left( (\bs{x}, \epsilon) , (\bs{w}, b) \right) \coloneqq  y \indic{\abs{y} \leq R } \inner{\bs{v}}{\bs{x}} \phi^\prime(\inner{\bs{w}}{\bs{x}} + b) ~~ \text{and} ~~ \tilde Z \coloneqq \sup_{\theta \in S_{M}^{d-1} \times \R} \tilde{\bs{Y}}_{\theta,v}.  
}
By Lemma \ref{lem:generalsymmetrization},  for $\lambda > 0$,  we have that
\eq{
\E \left[  \exp \left( \lambda \tilde Z  \right) \right] \leq \E \left[ \sup_{  \pspace } \exp \left( \frac{2\lambda}{n} \sum_{i=1}^n \varepsilon_i \ell \left((\bs{x}_i, \epsilon_i), (\bs{w}, b) \right)   \right)  \right].
}
Let's focus on the empirical complexity.  We have
\eq{
\E_\varepsilon     \Bigg[ \sup_{ \pspace}     \exp \left( \frac{2\lambda}{n} \sum_{i=1}^n \varepsilon_i \ell \left((\bs{x}_i, \epsilon_i), (\bs{w},b) \right)  \right)  \Bigg]  
& \labelrel={lemsup:eqq0}   \E_\varepsilon     \left[ \sup_{(\bs{w},b) \in Q_x}       \exp \left( \frac{2\lambda}{n} \sum_{i=1}^n \varepsilon_i  \ell \left((\bs{x}_i, \epsilon_i), (\bs{w},b) \right)   \right)  \right]   \\ 
& \leq         \sum_{(\bs{w},b) \in Q_x}        \E_\varepsilon  \left[ \exp \left( \frac{2\lambda}{n} \sum_{i=1}^n \varepsilon_i \ell \left((\bs{x}_i, \epsilon_i), (\bs{w},b) \right)  \right)  \right]   \\
& \labelrel={lemsup:ineqq1}        \sum_{(\bs{w},b) \in Q_x} \prod_{i = 1}^n \E_\varepsilon  \left[ \exp \left( \frac{2\lambda}{n} \varepsilon_i \ell \left((\bs{x}_i, \epsilon_i), (\bs{w},b) \right)  \right)  \right]   ~~~ \label{truncproc:eq0}
}
where \eqref{lemsup:eqq0} follows from Corollary \ref{cor:vcdimimp} and \eqref{lemsup:ineqq1} follows from the independence of $\varepsilon_i$.  By using the moment generating function for Radamacher random variables,   Lemma \ref{lem:gaussmgf} and Corollary \ref{cor:vcdimimp}, we have  for $\lambda \in \left[0, \frac{n}{ 4   R} \right]$,
\eq{
\eqref{truncproc:eq0} \leq  \sum_{(\bs{w},b) \in Q_x} \prod_{i = 1}^n \exp \left( \frac{4\lambda^2 }{n^2}  \ell \left((\bs{x}_i, \epsilon_i), (\bs{w},b) \right)^2 \right)  &\leq \sum_{(\bs{w},b) \in Q_x} \prod_{i = 1}^n \exp \left( \frac{8 \lambda^2 R^2  }{n^2} \right) \\
& \leq    \left( \frac{e n}{d^*} \right)^{d^*}  \exp \left( \frac{8 \lambda^2 R^2  }{n} \right).
}
By Chernoff bound, the statement follows.
\end{proof}
\paragraph{Concentration of $\bs{Y}_\theta$}
\begin{lemma}
\label{lem:supbound}
Let $\phi(t) \in  \phiset$,   $d \geq 4M$ and $M^\prime \leq 2 M$,  and
\eq{
n \geq 24 M \log^2 \left( \frac{24dn}{M} \right) ~~ \text{and} ~~ M \geq \log (2/\delta).
}
We have for $\delta \in (0,1]$,
\eq{
\mpr \left[   \sup_{\substack{\cJ \subseteq [d] \\ \abs{\cJ} = M^\prime  }} \sup_{ \pspace} \norm*{ \bs{Y}_\theta \vert_{\cJ} }_2 \geq K \log^{C_2} \left( 6n / \delta \right) \sqrt{\frac{M \log^2\left( \frac{24 dn}{M} \right)}{n}}  \right] \leq \delta,
}
where $K$ is a constant depending on $(C_1, C_2, r,\Delta)$.
\end{lemma}

\begin{proof}
Let  $\tilde{\bs{Y}}_\theta \coloneqq  \frac{1}{n} \sum_{i = 1}^n y_i \indic{\abs{y_i} \leq R } \bs{x}_i \phi^\prime(\inner{\bs{w}}{\bs{x}_i} + b)  -\E \left[ y \indic{\abs{y} \leq R } \bs{x}_i \phi^\prime(\inner{\bs{w}}{\bs{x}} + b)  \right]$,
where $R = C_1 (r + 2)^{C_2} ( e \log(6n/\delta) )^{C_2} + \sqrt{\frac{\Delta}{e} } ( e \log(6n/\delta) )^{\frac{1}{2}}$.
We observe that 
\eq{
 \sup_{\substack{\cJ \subseteq [d] \\ \abs{\cJ} = M^\prime  }}\sup_{  \pspace }   \norm*{ \bs{Y}_\theta \vert_{\cJ} }_2 & \leq  \underbrace{ \sup_{\substack{\cJ \subseteq [d] \\ \abs{\cJ} = M^\prime  }} \sup_{  \pspace }  \norm*{ \tilde{\bs{Y}}_\theta \vert_{\cJ} }_2 }_{\coloneqq S_1}  +  \underbrace{  \sup_{ \pspace }  \norm*{  \frac{1}{n} \sum_{i = 1}^n y_i \indic{\abs{y_i} > R } \bs{x}_i \phi^\prime(\inner{\bs{w}}{\bs{x}_i} + b) }_2 }_{\coloneqq S_2} \\
 &+  \underbrace{ \sup_{ \pspace}  \norm*{ \E \left[ y \indic{\abs{y} > R } \bs{x} \phi^\prime(\inner{\bs{w}}{\bs{x}} + b)   \right]}_2 }_{\coloneqq S_3}.
}
For $K = \Bigg( C_1^4  (4 C_2)^{4C_2}  (r+2)^{4 C_2} +  2 \Delta^2  \Bigg)^{\frac{1}{4}}$,  we have 
\eq{
 \mpr \Bigg[  \sup_{\substack{\cJ \subseteq [d] \\ \abs{\cJ} = M^\prime  }} \sup_{  \pspace } \norm*{ \bs{Y}_\theta \vert_{\cJ} }_2  \geq 16 R \max \{ t, t^2 \}  + 4 K \sqrt{   \tfrac{\delta}{6n } } \Bigg] 
& \labelrel\leq{thmsup:ineqq0}  \mpr \left[S_1 \geq 16 R \max \{ t, t^2 \}  \right] \\
&  + \mpr \left[S_2  \geq \left(4 - 6^{\frac{3}{4}}  \right) K \sqrt{   \tfrac{\delta}{6n } }  \right] \\
&  \labelrel\leq{thmsup:ineqq1}  \mpr \left[ S_1 \geq 16  R \max \{ t, t^2 \}  \right] + \tfrac{\delta}{2}  \label{thmsup:proofarg}
}
where \eqref{thmsup:ineqq0} follows from Proposition \ref{prop:biasbound} (since $4 > 6^{\frac{3}{4}}$),  and  \eqref{thmsup:ineqq1}  from Proposition \ref{prop:concfory}.

Next, we need to establish a high probability bound via covering argument. Let $\cN_{M^\prime}^{1/2}$ be the minimal $1/2$-cover of $S_{M^\prime}^{d-1}$. We have
\eq{
S_1  = \sup_{ \pspace} \sup_{\bs{v} \in S_{M^ \prime}^{d-1}} \inner{\bs{v}}{\tilde{\bs{Y}}_\theta}  & \leq 2  \sup_{ \pspace } \sup_{\bs{v} \in \cN_{M^\prime}^{1/2}  } \inner{\bs{v}}{\tilde{\bs{Y}}_\theta},  \label{thmsup:eq0}
}
where  $\tilde{\bs{Y}}_{\theta, v}$ is introduced in Lemma \ref{lem:supprocessaux}. 
Therefore, by \eqref{thmsup:eq0}, we have
\eq{
 \mpr \left[ S_1 \geq 16 R \max \{ t, t^2 \}  \right]    \leq    \sum_{\bs{v} \in \cN_{M^\prime}^{1/2}  }      \mpr  \Bigg[  \sup_{ \pspace }    \tilde{\bs{Y}}_{\theta, v}  \geq 8 R \max \{ t, t^2 \}  \Bigg] 
    \labelrel\leq{thmsup:ineqq2}       \binom{d}{M^\prime} 5^{M^\prime} \left( \frac{en}{d^*} \right)^{d^*}    e^{ - n t^2} 
}
where \eqref{thmsup:ineqq2} follows from Corollary \ref{cor:size}. Therefore, we have 
\eq{
 \mpr \Bigg[  \sup_{\substack{\cJ \subseteq [d] \\ \abs{\cJ} = M^\prime  }}   \sup_{ \pspace } \norm*{ \bs{Y}_\theta \vert_{\cJ} }_2 \geq  16 R \max \{ t, t^2 \} + 4 K \sqrt{  \frac{\delta}{6n } } \Bigg] \leq \frac{\delta}{2}   +   \binom{d}{M^\prime} 5^{M^\prime} \left( \frac{en}{d^*} \right)^{d^*}   e^{ - n t^2}.  ~~~~~~~ \label{corsup:eq99}
}
We note that
\eq{
R  \leq ( C_1 (r + 2)^{C_2} e^{C_2}  +  \sqrt{ \Delta e })  \log^{C_2}(6n/\delta). \label{corsup:eq0}
}
Moreover, for $d \geq 4M$ and $M^\prime \leq 2 M,$  we have
\eq{
\binom{d}{M^\prime} 5^{M^\prime}  \left( \frac{en}{d^*} \right)^{d^*} \leq \binom{d}{2M} 5^{2M} \left(\frac{en}{M} \right)^{6M \log \left(\frac{ed}{M}  \right)} 
& \labelrel\leq{corsup:ineqq0} \left( \frac{5ed}{2M} \right)^{2M} \left(\frac{en}{M} \right)^{6M \log \left(\frac{ed}{M}  \right)}     \\
& \leq  \left( \frac{5e^2 n d}{2M} \right)^{6M \log \left(\frac{ed}{M}  \right)} \label{corsup:eq1}
}
where \eqref{corsup:ineqq0} follows from $\binom{d}{M} \leq \left( \frac{ed}{M} \right)^M$.
Therefore,
\eq{
\log \left[ \binom{d}{M^\prime} 5^{M^\prime}  \left( \frac{en}{d^*} \right)^{d^*}   \right] \leq  6M \log \left(\frac{ed}{M}  \right) \log  \left( \frac{5e^2 n d}{2M} \right) \leq   6M \log^2 \left(\frac{24 n d}{M}  \right). \label{corsup:eq2}
}

By using \eqref{corsup:eq2} and  \eqref{corsup:eq99} with $t = \sqrt{\frac{  6M \log^2 \left(\frac{24 n d}{M}  \right) }{n}} + \sqrt{\frac{\log (2/\delta)}{n}} \in [0,1]$ and $u = e \log(6n/\delta)$,  we obtain the statement.
\end{proof}

\subsection{Concentration for $\vect{\Sigma}_\theta$}
In this part, we are interested in deriving a concentration bound for
\eq{
\sup_{\substack{\cJ \subseteq [d] \\ \abs{\cJ} = M^\prime}} \sup_{ \pspace} \norm{\bm{\Sigma}_\theta \cJJ}_2.
}
For a fixed $(\bs{w}, b) \in S_{M}^{d-1} \times \R,$  by using the Rayleigh quotient formula,  we can write that
\eq{
\sup_{\substack{\cJ \subseteq [d] \\ \abs{\cJ} = M^\prime}} \norm{\bm{\Sigma}_\theta \cJJ}_2 = \sup_{\substack{\cJ \subseteq [d] \\ \abs{\cJ} = M^\prime}}  \sup_{\bs{v} \in S^{d-1}} \abs{\inner{\bs{v}}{\bm{\Sigma}_\theta \cJJ \bs{v}}} =   \sup_{\bs{v} \in S_{M^\prime}^{d-1}} \abs{\inner{\bs{v}}{\bm{\Sigma}_\theta \bs{v}}}.
}
Let $\cN_{M^\prime}^{1/4}$ be the minimal $1/4$-cover of $S_{M^\prime}^{d-1}$.   It is easy to check that for $\bs{v} \in \cN_{M^\prime}^{1/4}$, we have 
$$\sup_{\bs{v} \in S_{M^\prime}^{d-1}} \abs{\inner{\bs{v}}{\bm{\Sigma}_\theta \bs{v}}} \leq 2 \sup_{\bs{v} \in \cN_{M^\prime}^{1/4}} \abs{\inner{\bs{v}}{\bm{\Sigma}_\theta \bs{v}}}.$$ 
Therefore, we have
\eq{
\sup_{\substack{\cJ \subseteq [d] \\ \abs{\cJ} = M^\prime}} \sup_{ \pspace } \norm{\bm{\Sigma}_\theta \cJJ}_2 \leq   \sup_{\bs{v} \in \cN_{M^\prime}^{1/4}}  2 \sup_{\pspace }  \abs{\inner{\bs{v}}{\bm{\Sigma}_\theta \bs{v}}}.  \label{eq:boundforsupop}
} 
Since we already have a  bound for the size of $\cN_{M^\prime}^{1/4}$,  we first derive a concentration bound for $\sup_{ \bs{w}, b}  \abs{\inner{\bs{v}}{\bm{\Sigma}_\theta \bs{v}}}$ for a fixed $\bs{v} \in S_{M^\prime}^{d - 1}$.

\paragraph{Concentration for $\sup_{\bs{w},b}  \abs{\inner{\bs{v}}{\bm{\Sigma}_\theta \bs{v}}}$ }

\begin{lemma}
\label{lem:supopnormaux}
For $\phi(t) \in \phiset$, $M, \in [d]$,   and for a fixed $\bs{v} \in S^{d-1}$ and $n \geq d^*$,  we have that for $t \geq 0$,
\eq{
\mpr \Bigg[  \sup_{ \pspace }  \abs{\inner{\bs{v}}{\bm{\Sigma}_\theta \bs{v}}} \geq  8 \sqrt{2} \max\{  t,t^2\}  \Bigg] \leq 2 \left( \frac{en}{d^*} \right)^{d^*} \exp \left(  - n t^2 \right).
}
\end{lemma}
\begin{proof}
We observe that 
\eq{
\inner{\bs{v}}{\bm{\Sigma}_\theta \bs{v}} =  \frac{1}{n} \sum_{i = 1}^n \phi^\prime(\inner{\bs{w}}{\bs{x}_i} + b) \inner{\bs{v}}{\bs{x}_i}^2 - \E \left[ \phi^\prime(\inner{\bs{w}}{\bs{x}} + b) \inner{\bs{v}}{\bs{x}}^2   \right].
}
For 
\eq{
 Z \coloneqq  \sup_{ \pspace }  \frac{1}{n} \sum_{i = 1}^n \phi^\prime(\inner{\bs{w}}{\bs{x}_i} +b) \inner{\bs{v}}{\bs{x}_i}^2 - \E \left[ \phi^\prime(\inner{\bs{w}}{\bs{x}} + b) \inner{\bs{v}}{\bs{x}}^2   \right] 
}
by using  Lemma \ref{lem:generalsymmetrization}, we can write that for $\lambda \geq 0$,
\eq{
\E \left[ \exp (\lambda Z)  \right] \leq \E \Bigg[ \sup_{ \pspace } \exp \left( \frac{2\lambda}{n} \sum_{i=1}^n \varepsilon_i \phi^\prime(\inner{\bs{w}}{\bs{x}_i} + b) \inner{\bs{v}}{\bs{x}_i}^2 \right) \Bigg].   \label{supop:mgf}
} 
Let's look at the empirical complexity. We have
\eq{
\E_\varepsilon  \Bigg[ \sup_{ \pspace }   \exp \Bigg( \frac{2\lambda}{n} \sum_{i=1}^n \varepsilon_i  &\phi^\prime(\inner{\bs{w}}{\bs{x}_i} + b)  \inner{\bs{v}}{\bs{x}_i}^2 \Bigg)  \Bigg] \\
 &  \labelrel={supopaux:eq0}  \E_\varepsilon  \Bigg[ \sup_{(\bs{w},b) \in Q_x} \exp \left( \frac{2\lambda}{n} \sum_{i=1}^n \varepsilon_i \phi^\prime(\inner{\bs{w}}{\bs{x}_i} + b) \inner{\bs{v}}{\bs{x}_i}^2 \right)  \Bigg]  \\
& \leq \sum_{(\bs{w},b) \in Q_x} \E_\varepsilon  \left[ \exp \left( \frac{2\lambda}{n} \sum_{i=1}^n \varepsilon_i \phi^\prime(\inner{\bs{w}}{\bs{x}_i} + b) \inner{\bs{v}}{\bs{x}_i}^2 \right)   \right]  \\
&  \labelrel={supopaux:eq1}   \sum_{(\bs{w},b) \in Q_x} \prod_{i = 1}^n \E_\varepsilon  \left[ \exp \left( \frac{2\lambda}{n} \varepsilon_i   \phi^\prime(\inner{\bs{w}}{\bs{x}_i} + b) \inner{\bs{v}}{\bs{x}_i}^2 \right)  \right],   \label{supop:eq1}
 }where \eqref{supopaux:eq0} follows from  Corollary \ref{cor:vcdimimp} and \eqref{supopaux:eq1} follows  by independence.   Let $\cosh(t) \coloneqq \tfrac{e^t + e^{-t}}{2}.$  We observe that for a fixed $i \in [n]$, 
\eq{
\E_\varepsilon       \left[ \exp \left( \tfrac{2\lambda}{n} \varepsilon_i   \phi^\prime(\inner{\bs{w}}{\bs{x}_i} +b) \inner{\bs{v}}{\bs{x}_i}^2 \right)  \right] 
  & =     \cosh  \left( \tfrac{2\lambda}{n} \abs{\phi^\prime(\inner{\bs{w}}{\bs{x}_i} + b) }   \inner{\bs{v}}{\bs{x}_i}^2 \right)   \\ 
& \leq      \cosh \left(   \tfrac{2\lambda}{n}    \inner{\bs{v}}{\bs{x}_i}^2 \right)   \label{supop:eq2}
}where we use  $\abs{\phi^\prime} \leq 1$ and that $\cosh$ is increasing on $t \geq 0$.
Therefore by \eqref{supop:eq1} and \eqref{supop:eq2},   for $\lambda \in \left[0, n/4\sqrt{2} \right]$
\eq{
\E \left[ \exp (\lambda Z)  \right] \leq  \sum_{(\bs{w},b) \in Q_x} \prod_{i = 1}^n \E \left[ \cosh \left( \frac{2\lambda}{n}  \inner{\bs{v}}{\bs{x}_i}^2 \right)  \right] \leq \left( \frac{e n}{d^*}  \right)^{d^*}  \exp \left( \frac{16 \lambda^2}{n}  \right),
}
where we used  Lemma \ref{lem:gaussmgf} and Corollary \ref{cor:vcdimimp}.
By Chernoff's bound, the statement follows.
\end{proof}

\paragraph{Concentration for $\vect{\Sigma}_{\theta}$} 
\hfill \\
\noindent
The next statement provides a concentration bound for \eqref{eq:boundforsupop}. 
\begin{lemma}
\label{lem:supopnorm}
For $\phi(t) \in \phiset$, $M, M^\prime \in [d]$,   and for  $d \geq 4M$ and $M^\prime \leq 2M,$ 
\eq{
n \geq 24 M \log^2 \left( \frac{35dn}{M} \right) ~~ \text{and} ~~ M \geq \log (2/\delta),
}
we have for $\delta \in (0,1]$,
\eq{
\mpr \left[   \sup_{\substack{\cJ \subseteq [d] \\ \abs{\cJ} = M^\prime}} \sup_{ \pspace } \norm{\bm{\Sigma}_\theta \cJJ}_2  \geq K \sqrt{\frac{M \log^2\left( \frac{35 dn}{M} \right)}{n}}  \right] \leq \delta,
}
where $K$ is a universal positive constant.
\end{lemma}

\begin{proof}
By using \eqref{eq:boundforsupop} and Lemma \ref{lem:supopnormaux}, we can write that for $n \geq d^*$
\eq{
\mpr \left[  \sup_{\substack{\cJ \subseteq [d] \\ \abs{\cJ} = M^\prime}} \sup_{ \pspace} \norm{\bm{\Sigma}_\theta \cJJ}_2  \geq 16 \sqrt{2}   \max \{ t, t^2 \}   \right]       & \leq           \sum_{\bs{v} \in \cN_{M^\prime}^{1/4}}       \mpr \Bigg[  \sup_{\pspace}  \abs{\inner{\bs{v}}{\bm{\Sigma}_\theta \bs{v}}} \geq  8\sqrt{2} \max\{  t,t^2\}  \Bigg] \\
& \labelrel\leq{supopthm:ineqq0} 2 \binom{d}{	M^\prime}  9^{M^\prime} \left( \frac{en}{d^*} \right)^{d^*} \exp \left(  - n t^2 \right). \label{corsupop:eq99}
}where \eqref{supopthm:ineqq0}  follows from Corollary \ref{cor:size}.  We  note that  $n \geq 24 M \log^2 \left( \frac{35dn}{M} \right)  \geq   6 M \log \left(  \frac{ed}{M}  \right) \geq d^*$ by Proposition \ref{prop:vcdim}.  Moreover,  for $d \geq 4M$ and $M^\prime \leq 2M,$  we have
\eq{  
\binom{d}{M^\prime} 9^{M^\prime}  \left( \frac{en}{d^*} \right)^{d^*}            \leq  \binom{d}{2M} 9^{2M} (en)^{6M \log \left(\frac{ed}{M}  \right)} \leq \left( \frac{9ed}{2M} \right)^{2M}          (en)^{6M \log \left(\frac{ed}{M}  \right)}       \leq       \left( \frac{9 e^2 n d}{2M} \right)^{6M \log \left(\frac{ed}{M}  \right)}  \label{corsupop:eq1}
}
where the second inequality follows from $\binom{d}{M} \leq \left( \frac{ed}{M} \right)^M$.
Therefore,
\eq{
\log \left[ \binom{d}{M^\prime} 9^{M^\prime}  \left( \frac{6n}{d^*} \right)^{d^*}   \right] \leq  6 M \log \left(\frac{ed}{M}  \right) \log  \left( \frac{9 e^2  n d}{2M} \right) \leq   6 M \log^2 \left(\frac{35 n d}{M}  \right). \label{corsupop:eq2}
}
By using \eqref{corsupop:eq2} and \eqref{corsupop:eq99} with $t = \sqrt{\frac{  6M \log^2 \left(\frac{35 n d}{M}  \right) }{n}} + \sqrt{\frac{\log (2/\delta)}{n}} \in [0,1]$,  we obtain the statement.
\end{proof}

\subsection{Concentration for  $\vect{T}_\theta$}
By \eqref{eq:targetstatement2}  and $\norm{\E[\phi^\prime(\inner{\bs{w}}{\bs{x}} + b) \bs{x} \bs{x}^\top]}_2 \leq 1$,  we have
\begin{align}
\norm*{\bs{T}_\theta  \vert_{\cJ} }_2 \leq \norm*{\bs{Y}_\theta \vert_{\cJ} }_2 + \norm*{\bm{\Sigma}_\theta \cJJ}_2 \norm*{\E[y \bs{x}]}_2 + \left( \norm*{\bm{\Sigma}_\theta \cJJ}_2  + 1 \right)   \norm*{(\hmu - \E[y \bs{x}])\vert_{\cJ}}_2.   
\end{align} 
We have the following statement.
\begin{lemma}
\label{lem:supboundfinal}
For $\phi(t) \in \phiset$, $M, M^\prime \in [d]$,   for $d \geq 4M$,  $M^\prime \leq 2M$
\eq{
n \geq 24 M \log^2 \left(  \frac{35 d n}{M} \right)  ~~ \text{and} ~~ M \geq \log(6/ \delta),
}
we have that for $\delta \in (0,1]$
\eq{
\mpr \Bigg[   \sup_{\substack{\cJ \subseteq [d] \\ \abs{\cJ} = M^\prime}} \sup_{ \pspace } \norm{\bs{T}_\theta \vert_{ \cJ} }_2 \geq K \log^{C_2} \left(18 n/\delta \right)  \sqrt{\frac{M  \log^2 \left(  \frac{35 d n}{M} \right) }{n}}   \Bigg] \leq \delta,
}
where $K$ is a positive constant depending on $(C_1, C_2, r, \Delta)$.
\end{lemma}

\begin{proof}
We note that Lemma \ref{lem:supbound} applies to $\phi(t) \in \phiset$.  Therefore,  by Lemma \ref{lem:supbound} for  $\phi(t) = \abs{t}$, and $\phi(t) =  \text{ReLU}(t)$,  and Lemma \ref{lem:supopnorm},  we have the statement.
\end{proof}

\subsection{Concentration Bound for the Empirical Gradient in the Single-Index Setting}
 In this part, since $r = 1,$  for clarity,  we use the following notation: $\gt = \sum_{k \geq k^{\star}}   \frac{\ghc_k}{k!} \Hek$  and $y  =  \gt(\inner{\bs{v}}{\bs{x}}) + \sqrt{\Delta} \epsilon$.  
\begin{proposition}
\label{prop:emprgradientconcsingle}
We consider \eqref{empgrad:epmgraddef} with $\hmu =  0$ and $\phi(t) \in \phiset$.   Let $j \in [2 m]$ be a fixed index and  $\cJ$ be any function of $\{ (\bs{x}_i,y_i)\}_{i = 1}^n$ such that $\abs{\cJ} \leq M$ almost surely.   For  $d \geq 4M$, 
\eq{
n \geq 24 M \log^2 \left(  \frac{24 d n}{M} \right)  ~~ \text{and} ~~ M \geq 24 (1 + \log(4/\delta)),
}
the intersection of the following events holds with at least probability $1 - \delta$,
\begin{enumerate} 
\item $ \norm*{  \left.  g\left(\Wz_{j*}  , b \right) \right \pJ       -     \frac{ \ghc_{k^{\star}}  \rhc_{k^{\star}}(b) }{(k^\star -1)!}  \inner{\bs{v}}{\Wz_{j*} }^{k^{\star} - 1}     \left.  \bs{v} \right \pJ  }_2       \leq     K   \left(       \sqrt{\tfrac{M \log^2\left( \frac{24 dn}{M} \right) \log^{2C_2}   \left( \frac{12n}{\delta}\right)}{n}}   +    \left( \tfrac{ 1 + \log(4/\delta) }{M} \right)^{\frac{k^{\star}}{2}}      \right)$   
\item  $ \norm*{ \left.   g\left(\Wz_{j*}  , b \right)  \right \pJ  }_2  \leq   K \left(   \frac{  \abs{  \ghc_{k^{\star}}  \rhc_{k^{\star}}(b)  } }{(k^\star - 1)!}   \left( \tfrac{1+ \log(4/\delta)}{M}  \right)^{\frac{k^{\star}-1}{2}}          +  \sqrt{ \frac{M \log^2\left( \tfrac{24 dn}{M} \right) \log^{2C_2} \left( \tfrac{12n}{\delta}\right)}{n}}  \right).$  
\end{enumerate}
where $K > 0$ is a constant depending on $(C_1, C_2, k^{\star},  \Delta, \cgp)$.
\end{proposition}

\begin{proof}
We first observe that by Proposition \ref{prop:hermiteexpansion},  
\eq{
\E_{(\bs{x}, y)} \Big[  y \bs{x}&  \phi^\prime \Big(\inner{\Wz_{j*} }{x}  + b \Big)  \Big] \Big \vert_{\cJ}  = \bs{v} \vert_{\cJ}          \sum_{k \geq k^{\star} - 1}         \tfrac{\ghc_{k+1} \rhc_{k+1}(b)}{k!}  \inner{\bs{v}}{\Wz_{j*} }^{k} +   \Wz_{j*}   \sum_{k \geq k^{\star}}      \tfrac{\ghc_{k} \rhc_{k+2}(b)}{k!}  \inner{\bs{v}}{\Wz_{j*} }^{k}.
}
Therefore, we have
\eq{
\Big \lVert \left. g\left(\Wz_{j*}  , b \right) \right \pJ  - &  \frac{ \ghc_{k^{\star}}  \rhc_{k^{\star}}(b) }{(k^\star - 1)!}   \inner{\bs{v}}{\Wz_{j*} }^{k^{\star} - 1} \left. \bs{v} \right \pJ  \Big \rVert_2  \\
& \leq \norm*{ \left. g\left(\Wz_{j*}  , b \right) \right \pJ -  \left. \E_{(\bs{x}, y)} \left[  y \bs{x} \phi^\prime \left(\inner{\Wz_{j*} }{x} + b\right) \right] \right \pJ }_2 + \norm{\zeta_{k^{\star} - 1}} \\
& \leq   \sup_{\substack{\cJ \subseteq [d] \\ \abs{\cJ} = M}} \sup_{ \pspace} \norm{\bs{Y}_\theta \vert_{ \cJ} }_2  + (1 + \sqrt{k^{\star} + 1}) \cgp \frac{\abs*{ \inner{\bs{v}}{\Wz_{j*} } }^{k^{\star}}}{1 -   \inner{\bs{v}}{\Wz_{j*} }^2 } 
}
where $\zeta_{k^{\star} - 1}$ is the higher order terms in the Hermite expansion defined in Proposition \ref{prop:residualbound} and we use  Proposition \ref{prop:residualbound} in the third line line.  

To bound the second term,  we recall that $\Wz_{j*}  = \frac{\bs{W}_{j*} \vert_{\cJ}}{\norm{\bs{W}_{j*} \vert_{\cJ}}_2}$ where $\bs{W}_{j*} \sim \cN(0, \ide{d})$ and it is independent of $\{ (\bs{x}_i, y_i) \}_{i = 1}^n$.  Since $\cJ$ is independent of $\bs{W}_{j*}$,  without loss of generality,  we can fix a $\cJ$ with $\abs{\cJ} = M$.   By using  Corollaries \ref{cor:laurentmassart1} and \ref{cor:laurentmassart2},  the intersection of    $(\rnu{1})  ~ \sum_{i \in \cJ} \bs{W}_{ij}^2 \geq \frac{M}{2}$,  $  (\rnu{2}) ~ \inner{\bs{v}}{\bs{W}_{j*} \vert_{\cJ}}^2 \leq  3 (1 + \log(4/\delta) )$ holds with  probability at least $1 - \delta/2$.  Within that event,  for $M \geq 24 (1 + \log(4/\delta))$, we have
\eq{
(1 + \sqrt{k^{\star} + 1}) \cgp \frac{\abs*{ \inner{\bs{v}}{\Wz_{j*} } }^{k^{\star}}}{1 -   \inner{\bs{v}}{\Wz_{j*} }^2 }  \leq 6^{\frac{k^{\star} + 1}{2}} \cgp (1 + \sqrt{k^{\star} + 1}) \left( \frac{(1 + \log(4/\delta) )}{M} \right)^{\frac{k^{\star}}{2}}. \label{empgradsingle:bias}
}Then, by Lemma \ref{lem:supbound}, the first item in the statement follows.  For the second item, by using the event used for \eqref{empgradsingle:bias}, we have 
\eq{
\norm*{ \left. g\left(\Wz_{j*}  , b \right)  \right \pJ  } _2      
& \leq      \norm*{\left. g\left(\Wz_{j*}  , b \right) \right \pJ        -  \tfrac{   \ghc_{k^{\star}}  \rhc_{k^{\star}}(b)    }{(k^\star - 1)!}   \inner{\bs{v}}{\Wz_{j*} }^{k^{\star} - 1}            \bs{v} \vert_{\cJ}}_2     +      \tfrac{  \abs{  \ghc_{k^{\star}}  \rhc_{k^{\star}}(b)  } }{(k^\star - 1)!}   \left( \tfrac{6(1+ \log(4/\delta))}{M}  \right)^{      \frac{k^{\star}-1}{2}}          .
}
By using the first item in the statement,  the second item also follows.
\end{proof}

\subsection{Concentration Bound for the Empirical Gradient in the Multi-Index Setting}
We first derive the Hermite expansion of  $\E_{(\bs{x}, y)} \left[  \overline y \bs{x} \phi^\prime(\inner{\bs{w}}{\bs{x}} + b) \right]$ (see \eqref{eq:defys} for its definition).
\begin{lemma}
\label{lem:higherordertermsJ}
We recall that $\phi(\cdot + b) \coloneqq \sum_{k \geq 0} \frac{\rhc_k(b)}{k!} \Hek$ and $\gt \coloneqq \sum_{k \geq 0} \frac{1}{k!} \inner{\bm{ T_k }}{\bm{\Hek}}$.  For any $\cJ \subseteq [d]$ and any $\bs{w} \in S^{d-1}$ supported on $\cJ$, we have
\eq{ 
\left. \E_{(\bs{x}, y)} \left[  \overline y  \bs{x} \phi^\prime \left( \inner{\bs{w}}{\bs{x}} + b \right)   \right]\right \pJ & = \rhc_2(b)  \bs{H} \cJJ \bs{w} \\
&+  \bs{V} \vert_{\cJ} \sum_{k \geq 2} \tfrac{\rhc_{k+1}(b)}{k!} \bm{  \grad T_{k+1} } \left[  (\bs{V}^\top \bs{w})^{\otimes k}  \right]  +  \bs{w} \sum_{k \geq 2} \tfrac{\rhc_{k+2}(b)}{k!} \bm{ T_k} \left[   (\bs{V}^\top \bs{w})^{\otimes k}  \right],  
}
where $\bs{H}$ is defined in \eqref{eq:multiindexas}.
\end{lemma}

\begin{proof}
We first observe that  $\E[y \bs{x} ] = \E[ \gt(\bs{V}^\top \bs{x}) \bs{x}] = \bs{V} \E[\gt(\bs{z}) \bs{z}]$  and $\E[y \bs{x} ]  \vert_{\cJ} = \bs{V} \vert_{\cJ} \E[\gt(\bs{z}) \bs{z}]$.
By Proposition \ref{prop:hermiteexpansion}, we have
\eq{
 \E_{(\bs{x}, y)} \left[  \overline y  \bs{x} \phi^\prime \left(  \inner{\bs{w}}{\bs{x}} + b \right)   \right] 
& \labelrel={highJ:eqq0}  \rhc_1(b)  \bs{V} \vert_{ \cJ^c}   \E[\gt(\bs{z}) \bs{z}] +   \rhc_3(b) \bs{w} \bs{w}^\top  \bs{V} \vert_{ \cJ^c}   \E[\gt(\bs{z}) \bs{z}] +  \rhc_2(b) \bs{H} \bs{w}   \\
&  +  \bs{V} \sum_{k \geq 2} \frac{\rhc_{k+1}(b)}{k!} \bm{ \grad T_{k+1} } \left[ (\bs{V}^\top \bs{w})^{ \otimes k}  \right]   +  \bs{w} \sum_{k \geq 2} \frac{\rhc_{k+2}(b)}{k!} \bm{ T_k } \left[   (\bs{V}^\top \bs{w})^{ \otimes k}   \right] \label{horder:eq0} 
}where \eqref{highJ:eqq0}  holds since  $\bm{ \grad T_1 } = \E[\gt(\bs{z}) \bs{z}]$,  $ \bs{V} \bm{ \grad T_2 } \left[ (\bs{V}^\top \bs{w})^{ \otimes 1}  \right]  =  \bs{H} \bs{w}$,   $\bm{T_0} = 0$,  $\bm{ T_1 } \left[   (\bs{V}^\top \bs{w})^{ \otimes 1}   \right] = \inner{\bs{w}}{\bs{V} \E[\gt(\bs{z}) \bs{z}]}$.  
Since   $\bs{w}^\top  \bs{V} \vert_{ \cJ^c}   \E[\gt(\bs{z}) \bs{z}]  = 0$, we have
\eq{
\eqref{horder:eq0} & =   \rhc_2(b) \bs{H} \bs{w}   +  \rhc_1(b)   \bs{V} \vert_{ \cJ^c}   \E[\gt(\bs{z}) \bs{z}]  \\
& +   \bs{V} \sum_{k \geq 2} \tfrac{\rhc_{k+1}(b)}{k!} \bm{ \grad T_{k+1} } \left[ (\bs{V}^\top \bs{w})^{ \otimes k}  \right] + \bs{w} \sum_{k \geq 2} \tfrac{\rhc_{k+2}(b)}{k!} \bm{ T_k } \left[   (\bs{V}^\top \bs{w})^{ \otimes k}   \right].
}
Since $\bs{w}$ is supported on $\cJ$,   the statement follows.
\end{proof}

\begin{proposition}
\label{prop:emprgradientconc}
We consider \eqref{empgrad:epmgraddef} with $\hmu = \sum_{i = 1}^n y_i \bs{x}_i$ and $\phi(t) \in \phiset$.   Let $j \in [m]$ be a fixed index and  $\cJ$ be any function of $\{ (\bs{x}_i,y_i)\}_{i = 1}^n$ such that $\abs{\cJ} \leq M$ almost surely.   For  $d \geq 4M$, 
\eq{
n \geq 24 M \log^2 \left(  \frac{35 d n}{M} \right)  ~~ \text{and} ~~ M \geq 24 (r + \log(12/\delta)),
}
the intersection the following events hold with at least probability $1 - \delta$,
\begin{enumerate}
\item  $\norm*{ \left.  g\left(\Wz_{j*}  , b \right)  \right \pJ - \rhc_2(b)  \bs{H} \cJJ \Wz_{j*}  }_2  \leq  K \left(   \sqrt{\frac{M  \log^2 \left(  \tfrac{35 d n}{M} \right)\log^{2C_2} \left( \tfrac{18 n}{\delta}  \right) }{n}} +   \tfrac{(r + \log(4/\delta))}{M} \right)$
\item $\norm*{ \left. g\left(\Wz_{j*}  , b \right)  \right \pJ  } \leq   K \left( \abs{\rhc_2(b)} \sqrt{  \tfrac{ r + \log(4/\delta)}{M} } +   \sqrt{\frac{M  \log^2 \left(  \tfrac{35 d n}{M} \right)  \log^{2C_2} \left( \frac{18 n}{\delta}  \right) }{n}}   \right).$
\end{enumerate}
where $K > 0$ is a constant depending on $(C_1, C_2, r, \Delta, \cgp)$.
\end{proposition}

\begin{proof}
We have that
\eq{
\norm*{ \left. g\left(\Wz_{j*}  , b \right) \right \pJ  -  \rhc_2(b)  \bs{H} \cJJ \Wz_{j*}  }_2    
& \leq   \norm*{ \left. g\left(\Wz_{j*}  , b \right) \right \pJ -  \left. \E  \left[   \overline y  \bs{x} \phi^\prime \left(  \inner{ \Wz_{j*}  }{\bs{x}} + b \right)   \right] \right \pJ }_2  \\
& +  \norm*{  \left. \E \left[  \overline y  \bs{x} \phi^\prime \left(  \inner{\Wz_{j*} }{\bs{x}} + b \right)   \right] \right \vert_{\cJ}         - \rhc_2(b)  \bs{H} \cJJ  \Wz_{j*} }_2   \\ 
& \labelrel\leq{empgradmulti:ineqq0}    \sup_{\substack{\cJ \subseteq [d] \\ \abs{\cJ} = M}} \sup_{ \pspace }  \norm{\bs{T}_\theta \vert_{ \cJ} }_2  +   \norm*{ \zeta_1 \vert_{\cJ}  }_2  \\
&  \labelrel\leq{empgradmulti:ineqq1}  \sup_{\substack{\cJ \subseteq [d] \\ \abs{\cJ} = M}} \sup_{ \pspace}  \norm{\bs{T}_\theta \vert_{ \cJ} }_2  +  2 \sqrt{3}  \cgp  \frac{\norm{\bs{V}^\top \Wz_{j*} }_2^2}{1 - \norm{\bs{V}^\top \Wz_{j*} }_2^2}. 
}
where we used Lemma \ref{lem:higherordertermsJ} in \eqref{empgradmulti:ineqq0} and Proposition \ref{prop:residualbound} in \eqref{empgradmulti:ineqq1}.

We will first bound the second term.  We recall that $\Wz_{j*}  = \frac{\bs{W}_{j*} \vert_{\cJ}}{\norm{\bs{W}_{j*} \vert_{\cJ}}_2}$ where $\bs{W}_{j*} \sim \cN(0, \ide{d})$ and it is independent of $\{ (\bs{x}_i, y_i) \}_{i = 1}^n$.  Since $\cJ$ is independent of $\bs{W}_{j*}$,  without loss of generality,  we can fix a $\cJ$ with $\abs{\cJ} = M$.    By using  Corollaries \ref{cor:laurentmassart1} and \ref{cor:laurentmassart2},  the intersection of    $(\rnu{1}) ~ \sum_{i \in \cJ} \bs{W}_{ij}^2 \geq \frac{M}{2},$ $(\rnu{2}) ~ \norm{\bs{V} ^\top \bs{W}_{j*} \vert_{\cJ}}_2^2 \leq  3 (r + \log(4/\delta) )$  holds with probability at least $1 - \delta/2$.
Within that event,  for $M \geq 24 (r + \log(12/\delta))$, we have
\eq{
2 \sqrt{3}  \cgp  \frac{\norm{\bs{V}^\top \Wz_{j*} }_2^2}{1 - \norm{\bs{V}^\top \Wz_{j*} }_2^2} \leq 16 \sqrt{3} \cgp \frac{ (r + \log(4/\delta))}{M}. \label{empgrad:bias}
}
Therefore, by Lemma \ref{lem:supboundfinal}, the first item follows. For the second item, we observe that
\eq{
\norm*{  \left. g\left(\Wz_{j*}  , b \right)  \right \pJ }_2 \leq \rhc_2(b) \norm*{\bs{H} \cJJ \Wz_{j*}}_2 + \norm*{  \left. g\left(\Wz_{j*}  , b \right)   \right \pJ -  \rhc_2(b) \bs{H} \cJJ \Wz_{j*}   }_2 
}
We have that
\eq{
 \norm{\bs{H} \cJJ \Wz_{j*}}_2 \leq  \rhc_2(b)  \frac{ \norm*{\bs{V}^\top \Wz_{j*}}_2 }{\norm*{\Wz_{j*} \vert_{\cJ}}_2} \leq   \rhc_2(b) \sqrt{ \frac{6 (r + \log(4/\delta) )}{M} }. \label{empgrad:normbound}
}
where we used $\sigma_1(\bs{H}) \leq 1$ in the first step, and the event used for \eqref{empgrad:bias}.  Hence by the first part of the statement, the second item also follows.
\end{proof}

\section{Guarantee for \texttt{PruneNetwork}}
We recall the following notation: For $\bs{a},  \bs{b} \in \R^{2m}$ and $\bs{W} \in \R^{2m \times d}$,
\eq{
& R_n^{\pm}(\bs{a},  \te_l , \bs{b}) \coloneqq \frac{1}{2n} \sum_{i = 1}^n \big(y_i - \hat y^{\pm}\left(\bs{x}_i; (\bs{a},   \te_l , \bs{b} )\right) \big)  \\
& \hat y^{\pm} \left(\bs{x}; ( \bs{a},   \te_l , \bs{b} ) \right)  \coloneqq    \sum_{j = 1}^{2m} a_j \underbrace{ \left(  \tfrac{\phi(\inner{\te_l}{\bs{x}} + \bs{b}_j) \pm \phi(- \inner{\te_l}{\bs{x}} + b_j) }{2} \right)}_{\phi_{\pm}(\inner{\te_l}{\bs{x}} ; b_j)} 
}
and  the gradients of the empirical/population risks are
\eq{
& \grad_{j} R_n^{\pm}(\bs{a},  \te_l , \bs{b}) =   \frac{ -a_j}{n} \sum_{i = 1}^n \big( y_i - \hat y^{\pm} \left(\bs{x}_i; (\bs{a},  \te_l , \bs{b}) \right) \big)  \phi_{\pm}^\prime( \inner{\te_l}{\bs{x}_i} ; b_j ) \bs{x}_i \\
& \grad_{j} R^{\pm}(\bs{a},  \te_l , \bs{b}) = - a_j  \E_{(\bs{x}, y)} \left[ \big( y - \hat y^{\pm} \left(\bs{x}; (\bs{a},  \te_l , \bs{b})  \right) \big)  \phi_{\pm}^\prime \left( \inner{\te_l}{\bs{x}} ; b_j \right) \bs{x}  \right].
}
Finally,  we recall that
\eq{
& \norm{ \grad R_n(\bs{a},  \te_l , \bs{b}) }_F^2 = \sum_{j = 1}^m \norm{  \grad_{j} R_n(\bs{a},  \te_l , \bs{b}) }_2^2  ~~ \text{and} ~~ \norm{ \grad R_n^{\pm}(\bs{a},  \te_l , \bs{b}) }_F^2 = \sum_{j = 1}^m \norm{  \grad_{j} R_n^{\pm}(\bs{a},  \te_l, \bs{b}) }_2^2.
}
\subsection{Auxiliary Results}
We have the following statement:
\begin{proposition}
\label{prop:pruningalgaux}
Let  $\mrhc^2_k \coloneqq \frac{1}{m} \sum_{j = 1}^m \rhc^2_{k}(\bzj)$.
For any $\mathcal{J} \subseteq [d]$, we have
\begin{enumerate}[leftmargin=*]
\item For the single-index setting and $k^{\star} > 1$,
\eq{
 \left[  \left( \frac{  \mrhc_{k^{\star}} \abs{ \ghc_{k^{\star}}}}{(k^{\star} - 1)!} \right)^\frac{2}{k^{\star} - 1}    -  8  \left(  \frac{c \sqrt{2} \cgp}{1 - c^2} \right)^{\frac{2}{k^{\star} -1}}   \right] \norm{\bs{v} \vert_{\cJ^c}}_2^2  \leq   \frac{m^{\frac{- 1}{k^{\star} - 1}}}{c^2}      \sum_{i \in \mathcal{J}^c} \norm{ \grad R^{\pm}(\az, \te_i,  \bz)}_F^{\frac{2}{k^{\star} - 1}}.
}where the statement with  $\grad R^{+}$ holds for even $k^\star$, and  $\grad R^{-}$ holds for odd $k^\star$.
\item For the multi-index setting,  we have
\eq{
 \left[     \mrhc_{2}^2  \sigma^2_r(\bs{H})   -  16 \left(  \frac{c \cgp}{1 - c^2} \right)^{2}   \right] \norm{\bs{V} \vert_{\cJ^c}}_F^2 \leq   \frac{ m^{-1}}{c^2} \sum_{i \in \mathcal{J}^c} \norm{ \grad R^{+}( \az, \te_i, \bz)}_F^2.
}
\end{enumerate}
\end{proposition}

\begin{proof}
We first observe that by \eqref{eq:syminit}, we have $\hat y^{\pm} \big(\bs{x}; ( \az, \te_i , \bz )  \big)= 0$. 
Therefore,
\eq{
\grad_{j} R^{\pm}( \az, \te_i,  \bz  ) =  - \azj \E_{(\bs{x}, y)} \left[ \sigma^*(\bs{V}^\top \bs{x}) \phi_{\pm}^\prime( \inner{\te_i}{\bs{x}} ; \bzj ) \bs{x}  \right]. \label{prunaux:popgrad}
}
Moreover, we observe that by \eqref{eq:syminit},   $\mrhc^2_k =  \frac{1}{m} \sum_{j = 1}^{m} \rhc^2_{k}(\bzj)$. 
\begin{enumerate}[leftmargin = *]
\item We will prove this item only for even $k^\star > 1$.  The proof for the odd case is identical when  $(+)$ signs are replaced with $(-).$ We have
\eq{
\Big(  \frac{\rhc_{k^{\star}}(\bzj) \ghc_{k^{\star}} }{(k^{\star} - 1)!} ( c \bs{v}_i) ^{k^{\star} - 1} \Big)^2 
& \labelrel\leq{prunaux:ineqq0}  2  \norm*{   \E_{(\bs{x}, y)} \left[ \sigma^*(\inner{\bs{v}}{\bs{x}}) \phi_{+}^\prime( \inner{\te_l}{\bs{x}} ; \bzj ) \bs{x}  \right] -   \frac{\rhc_{k^{\star}}(\bzj) \ghc_{k^{\star}} \inner{\bs{v}}{\te_i}^{k^{\star} - 1} }{(k^{\star} - 1)!}  \bs{v} }_2^2  \\
& + 2    \norm*{    \grad_{j} R^{+}(\az, \te_i , \bz) }^2_2 \\
& \labelrel\leq{prunaux:ineqq1}  2 (1 + \sqrt{k^{\star} + 1})^2 \cgp^2 \frac{c^{2k^{\star}} \abs{\bs{v}_i}^{2k^{\star}}}{(1 - c^2)^2}  + 2   \norm*{    \grad_{j} R^{-}(\az, \te_i , \bz ) }^2_2
}
where  \eqref{prunaux:ineqq0}  follows from \eqref{prunaux:popgrad},      \eqref{prunaux:ineqq1}  follows from Corollary \ref{cor:residualevenodd}. By summing each side over $j \in [2 m]$ and dividing by $1/2m$, we get
\eq{
\left(  \frac{\mrhc_{k^{\star}}  \ghc_{k^{\star}} }{(k^{\star} - 1)!} c^{k^{\star} - 1} \bs{v}_i^{k^{\star} - 1} \right)^2 \leq  2 (1 + \sqrt{k^{\star} + 1})^2 \cgp^2 \frac{c^{2k^{\star}} \abs{\bs{v}_i}^{2k^{\star}}}{(1 - c^2)^2}  +  \frac{2}{2m}  \norm*{    \grad R^{+}(\az, \te_i ,  \bz ) }^2_F.
}
By taking $\tfrac{1}{(k^{\star} - 1)}$th power of each sides,  we get
\eq{
\left(  \frac{\mrhc_{k^{\star}}  \abs{ \ghc_{k^{\star}} } }{(k^{\star} - 1)!} \right)^{\frac{2}{k^{\star} - 1}} c^2 \bs{v}_i^2 
&  \labelrel\leq{prunaux:ineqq2}   2^{\frac{1}{k^{\star} - 1}}   (1 + \sqrt{k^{\star} + 1})^{ \frac{2}{k^{\star} - 1} }   \left(  \frac{c \cgp \abs{\bs{v}_i} }{1 - c^2} \right)^{\frac{2}{k^{\star} - 1}} \!\!\! c^2 \bs{v}_i^2  +    m^{\frac{- 1}{k^{\star} - 1}}      \norm*{    \grad R^{+}(\az, \te_i , \bz ) }^\frac{2}{k^{\star} - 1}_F \\
& \labelrel\leq{prunaux:ineqq25}  2^{\frac{1}{k^{\star} - 1}}  8   \left( \cgp \frac{c  }{1 - c^2} \right)^{\frac{2}{k^{\star} - 1}} c^2 \bs{v}_i^2 +   m^{\frac{- 1}{k^{\star} - 1}}       \norm*{    \grad R^{+}(\az, \te_i , \bz ) }^\frac{2}{k^{\star} - 1}_F.
} 
where \eqref{prunaux:ineqq2} follows from Proposition \ref{prop:lqtri} and \eqref{prunaux:ineqq25} holds since $\abs{\bs{v}_i} \leq 1$ and  $(1 + \sqrt{k^{\star} + 1})^{ \frac{2}{k^{\star} - 1} }$  is decreasing for $k^{\star} \geq 2$.
Then, we get
\eq{
\left[ \left(  \frac{\mrhc_{k^{\star}}  \abs{ \ghc_{k^{\star}} } }{(k^{\star} - 1)!} \right)^{\frac{2}{k^{\star} - 1}}  -    2^{\frac{1}{k^{\star} - 1}}  8   \left( \cgp \frac{c  }{1 - c^2} \right)^{\frac{2}{k^{\star} - 1}}  \right] \bs{v}_i^2 \leq  \frac{  m^{\frac{- 1}{k^{\star} - 1}} }{c^2}     \norm*{    \grad R^{+}(\az, \te_i ,\bz ) }^\frac{2}{k^{\star} - 1}_F.
}
By summing each sides over $i \in \cJ^c$, we have the statement.
\item  By observing that $c \bs{H}_{i*} = \bs{H} \te_i$, we have 
\eq{
 \norm{ \rhc_2  (\bzj)   c \bs{H}_{i*}  }_2^2  
&   \labelrel\leq{prunaux:ineqq3}  2 \norm*{ \E_{(\bs{x}, y)} \left[ \sigma^*(\bs{V}^\top \bs{x}) \phi_{+}^\prime( \inner{\te_i}{\bs{x}} ; \bzj ) \bs{x}  \right]- \rhc_2(\bzj) \bs{H} \te_i }_2^2 + 2  \norm*{\grad_j R^{+}( \az, \te_i , \bz ) }_2^2  \\
&   \labelrel\leq{prunaux:ineqq4}   16 \cgp^2 \left(  \frac{c}{1 - c^2} \right)^2 c^2 \norm{\bs{V}_{i*}}_2^2 +  2 \norm*{\grad_j R^{+}(\az, \te_i ,\bz ) }_2^2. 
}
where   \eqref{prunaux:ineqq3} follows from  \eqref{prunaux:popgrad},   and \eqref{prunaux:ineqq4} holds since Corollary \ref{cor:residualevenodd} and $\norm{\bs{V}_{i*}}_2 \leq 1$.  By summing each side over $j \in [2m]$ and dividing by $1/2m$, we get
\eq{
\mrhc_2^2 c^2  \norm{\bs{H}_{i*}}_2^2  \leq   16 \cgp^2 \left(  \frac{c}{1 - c^2} \right)^2 c^2 \norm{\bs{V}_{i*}}_2^2 +   2 (2m)^{-1} \norm*{\grad R^{+}(\az, \te_i ,\bz ) }_F^2.
}
Therefore, we have
\eq{
\left[ \mrhc_2^2 \sigma^2_r(\bs{H})  -   16  \cgp^2 \left(  \frac{c}{1 - c^2} \right)^2 \right] \norm{\bs{V}_{i*}}_2^2  \leq \frac{  m^{-1}}{c^2} \norm*{\grad R^{+}(\az, \te_i ,\bz ) }_F^2.
}
By summing each sides over $i \in \cJ^c$, we have the statement.
\end{enumerate}
\end{proof}

\begin{proposition}
\label{prop:gradl2diffconc}
For this statement, by abusing the notation,  we use $0^0 = 1$.   Let 
\eq{
&\tilde R_i^{\pm}  \coloneqq  \frac{1}{2m} \sum_{j = 1}^{2m} \norm{\widetilde \grad_j R^{\pm}_n(\az, \te_i,  \bz)  - \grad_j R^{\pm} (\az, \te_i,  \bz) }_2^2, 
} 
where  $\widetilde \grad_j R^{\pm}_n(\az, \te_i, \bz) \coloneqq  \grad_j R^{\pm}_n(\az, \te_i,  \bz) \tm$,
\eq{
\tilde M \coloneqq M \log^2 \left(  \frac{35 nd}{M} \right) ~~ \text{and} ~~ C_q \coloneqq 8q(2 - q)^{\frac{2- q}{q}}.
}
For $d \geq 4M$, $n \geq 24 \tilde M$ and $M \geq \log(2/\delta)$,  each of the following items holds with probability at least $1 - \delta$:
\begin{enumerate}[leftmargin=*]
\item For the single-index setting with $k^{\star} \geq 1$, we have  
\eq{
\max_{i \in [d]} \tilde R_i^{\pm}  \leq   \left\{\begin{aligned}\quad
&               \frac{ K \tilde M \log^{2C_2} \left(  \frac{12nd}{\delta} \right) }{n}     &   q = 0,  M \geq \norm{\bs{v}}_{0} + 2   \\
&                \frac{ K \tilde M\log^{2 C_2} \left(  \frac{12nd}{\delta} \right) }{n}  + \frac{  C_q   \left(   \tfrac{c^{(k^{\star} - 1)} \cgp }{1 - c^2} \right)^2 \abs{\bs{v}_i}^{2(k^{\star} - 1)}  \left[ \norm{\bs{v}}^2_{q} \vee  k^{\star} 2^{\frac{2}{q}} \right]  }{M^{ \frac{2}{q} - 1}}  \hspace{-4.5em}   & q  \in (0,2).
     \end{aligned}\right.
}
\item For the  multi-index setting, we have
\eq{
\max_{i \in [d]}  \tilde R_i^{\pm} \leq   \left\{\begin{aligned}\quad
&               \frac{ K \tilde M \log^{2C_2} \left(  \frac{12nd}{\delta} \right) }{n}     &   q = 0,  M \geq \norm{\bs{V}}_{2,0}+ 2   \\
&                \frac{ K \tilde M\log^{2 C_2} \left(  \frac{12nd}{\delta} \right) }{n}  + \frac{  C_q  \left( \tfrac{\cgp}{1 - c^2} \right)^2  \left( c \norm{\bs{V}_{i*}}_2 \right)^{1 \pm 1}  \left[  \norm{\bs{V}}^2_{2,q} \vee    2^{\frac{2}{q} + 1} \right]  }{M^{ \frac{2}{q} - 1}}  \hspace{-5.5em}   & q  \in (0,2).
     \end{aligned}\right.
}
\end{enumerate}
Here, $K$ is a positive constant depending on  $(C_1, C_2, r, \Delta)$.
\end{proposition}

\begin{proof}
By Lemma \ref{lem:prunelemma},  we have
\eq{
 \norm{\widetilde \grad_j R^{\pm}_n(\az, \te_i, \bz)  - & \grad_j R^{\pm} (\az, \te_i, \bz) }_2^2 \\
&  \leq  5 \sup_{\substack{J \subseteq [d] \\ \abs{J} = 2M}}  \norm*{   \big(  \grad_j R^{\pm}_n(\az, \te_i, \bz)  - \grad_j R^{\pm} (\az, \te_i, \bz) \big)  \big \pJ  }_2^2 \\
& + 4  \norm{  \grad_j R^{\pm} (\az, \te_i, \bz) - \grad_j R^{\pm} (\az, \te_i, \bz) \tm  }_2^2. \label{pruneconc:eq0}
}
For any $\cJ \subseteq [d]$ with $\abs{\cJ} = 2M$,  by using Jensen's inequality,  we can show that
\eq{
& \norm*{  \big(  \grad_j R^{\pm}_n(\az, \te_i, \bz)  - \grad_j R^{\pm} (\az, \te_i, \bz) \big) \big \pJ }_2^2 
\leq     \sup_{\substack{J \subseteq [d] \\ \abs{J} = 2M}} \sup_{\pspace} \norm{\bs{Y}_{\theta} \vert_{\cJ}}_2^2. \label{pruneconc:eq1}
}
By \eqref{pruneconc:eq0} and \eqref{pruneconc:eq1}, we have for any $i \in [d]$,
\eq{
\tilde R_i^{\pm}  & \leq 5 \sup_{\substack{J \subseteq [d] \\ \abs{J} = 2M}} \sup_{\pspace} \norm{\bs{Y}_{\theta} \vert_{\cJ}}_2^2 \\
& +  \frac{4}{2m}  \sum_{j = 1}^{2m} \norm*{    \E  \left[ \sigma^*(\bs{V}^\top \bs{x}) \phi^\prime_{\pm} ( \inner{\te_i}{\bs{x}} ; \bzj )\bs{x} \right]  -   \E  \left[  \sigma^*(\bs{V}^\top \bs{x})  \phi^\prime_{\pm} ( \inner{\te_i}{\bs{x}} ; \bzj )\bs{x} \right]        \Big \vert_{\text{\scriptsize top($M$)}}  }_2^2.  
\label{pruneconc:eq2}
}If $q = 0$ and $M \geq \norm{\bs{V}}_{2,0} + 2$,  the statement follows for each item by Proposition \ref{prop:residualqnorm}.
For $q > 0$,  we have the following:
\begin{enumerate}[leftmargin = *]
\item We consider $k^{\star} \geq 1$ and even.  We have
\eq{
\tilde R_i^{\pm}  &  \labelrel\leq{pruneconc:ineqq0}  5 \sup_{\substack{J \subseteq [d] \\ \abs{J} = 2M}} \sup_{\pspace} \norm{\bs{Y}_{\theta} \vert_{\cJ}}_2^2  +   \frac{ 2q   \left( 1 - \frac{q}{2} \right)^{\frac{2-q}{q}}  M^{\frac{- 2}{q} + 1 } }{2m}   \sum_{j = 1}^{2m} \norm*{  \E_{(\bs{x}, y)} \left[ \sigma^*(\bs{V}^\top \bs{x}) \phi^\prime_{\pm} ( \inner{\te_i}{\bs{x}} ; \bzj )\bs{x} \right]  }^2_q  \\ 
& \labelrel\leq{pruneconc:ineqq1}   5 \sup_{\substack{J \subseteq [d] \\ \abs{J} = 2M}} \sup_{\pspace} \norm{\bs{Y}_{\theta} \vert_{\cJ}}_2^2 \\
& +     2q   \left( 1 - \frac{q}{2} \right)^{\frac{2-q}{q}}  M^{\frac{- 2}{q} + 1 }  4^{\frac{(q-1)}{q} \vee 0} 2^{\frac{2}{q} - 1 \vee 0}  \left(  \frac{c^{k^{\star} - 1} \cgp \abs{\bs{v}_i}^{k^{\star} - 1} }{1 - c^2} \right)^2 \left[   \norm{\bs{V}}_{2,q}^2 	+  k^{\star} \norm{\te_l}_q^2  \right] \\
&  \leq   5 \sup_{\substack{J \subseteq [d] \\ \abs{J} = 2M}} \sup_{\pspace} \norm{\bs{Y}_{\theta} \vert_{\cJ}}_2^2 +    C_q M^{\frac{- 2}{q} + 1 }  \left(  \frac{c^{k^{\star} - 1} \cgp \abs{\bs{v}_i}^{k^{\star} - 1} }{1 - c^2} \right)^2 \left[ \norm{\bs{V}}_{2,q}^2 	\vee  k^{\star} 2^{\frac{2}{q}} \right] 
}where we used Lemma \ref{lem:donoho} for \eqref{pruneconc:ineqq0},  and  Proposition \ref{prop:residualqnorm} with $N = k^{\star} - 2$ and Proposition \ref{prop:lqtri} for  \eqref{pruneconc:ineqq1}. 
By using Lemma \ref{lem:supbound} with $\tfrac{\delta}{2d}$ (for $i \in [d]$ and $(\pm)$ cases),  we have the result.
\item By using $k^{\star} = 1$ for $(-)$ and $k^{\star} = 2$ for $(+)$ in the proof of first item, one can prove this item as well.
\end{enumerate}
\end{proof}

\subsubsection{Concentration for $\mrhc_k$}
\begin{proposition}
\label{prop:alphaconc}
Let $m = \Theta(  d^{\varepsilon} )$ where $\varepsilon > 0$ is a small constant,  $Z_i \sim_{iid} \cN(0,1)$ for $i \in [m]$,  and let $\rhc_k(\cdot)$ be as in  \eqref{eq:constantterms}.  For any $u \in \N$, we have with probability at least $1 - d^{-u}$
\eq{
\frac{1}{m} \sum_{i = 1}^m \rhc_k(Z_i)^2 \geq  c_k (k- 1)!
}
for $d$ larger  than a constant depending on $(k, u, \varepsilon)$.
\end{proposition}

\begin{proof}
For $p \geq 1,$ by Jensen's inequality,  we have $\E[ \abs{\rhc^2_k(Z) - \E[\rhc^2_k(Z)]}^p ]^{1/p}  \leq 2 \E[ \rhc_k^{2p}(Z)  ]^{1/p}$. 
For $k \geq 2$,
\eq{
2 \E[ \rhc_k^{2p}(Z)  ]^{1/p}  = \frac{2}{2\pi} \E[ e^{- p Z^2} H^{2p}_{e_{k - 2}}(Z)  ]^{1/p} \leq   \frac{1}{\pi} \E[H^{2p}_{e_{k - 2}}(Z)  ]^{1/p} \labelrel\leq{alp:ineqq05}  \frac{(2p - 1)^{k-2}}{\pi} (k - 2)!,
}
where we use  Lemma \ref{lem:hypercontactivity} for \eqref{alp:ineqq05}.
Therefore,  if $Y_m \coloneqq \sum_{i = 1}^m \rhc_k(Z_i)^2  - \E[\rhc^2_k(Z)]$ and $K_p \coloneqq  \tfrac{(2p - 1)^{k-2}}{\pi} (k - 2)!$, by Lemma \ref{lem:rosenthal}, we have  
\eq{
\E[Y_m^{2p}]^{1/2p} \leq C   \left[  \sqrt{p K_2} \sqrt{m} + p m^{1/2p} K_p \right]    \Rightarrow    \mpr \left[  \abs*{\frac{1}{m} Y_m} \geq e  C \left(  \sqrt{ \frac{ p K_2}{m}}  + \frac{p m^{1/2p} K_p}{m} \right)  \right] \leq e^{-p}.
}By using $p = u \log d$ and hiding all of the constants with $k$ in  $C_k,$ we have for $k \geq 1$
\eq{
\mpr \left[  \abs*{\frac{1}{m} Y_m} \geq C_k  \sqrt{ \frac{( u \log d )^{(k - 1) \vee 1}}{m}}     \right] \leq  d^u.
}
Therefore, with probability $1 - d^u,$ we have
\eq{
\frac{1}{m} \sum_{i = 1}^m \rhc_k(Z_i)^2 & \geq \E[ \rhc_k(Z)^2 ] -  C_k  \sqrt{ \frac{( u \log d )^{2(k - 1) \vee 1}}{m}}
\labelrel\geq{alp:ineqq0}  \frac{1}{2} \E[ \rhc_k(Z)^2 ].  
}
where for \eqref{alp:ineqq0}, we assume that $d$ is larger than a constant depending on $(k, u, \varepsilon)$.    Since $\E[ \rhc_k(Z)^2 ]\geq c_k (k-1)!$, where $c_k$ is some $k$-dependent constant, the statement follows.
\end{proof}

\subsection{Main Results}
\begin{lemma}[Single-Index Setting]
\label{lem:pruningalgmainsingleindex}
Consider the single index setting.  For $u \in \N$ and a small constant $\varepsilon > 0$,  let
\eq{
m = \Theta(d^\varepsilon), ~~ d \geq    d(\ghc_{k^{\star}}, k^{\star}, u, \varepsilon)   \vee 4M  ~~ \text{and} ~~ c \leq \frac{1}{\log d},  \label{prunesingle:hypers1}
}
and  $\rho_1, \rho_2 \geq 1$, where $d(\ghc_{k^{\star}}, k^{\star}, u, \varepsilon)$ is a constant depending on  $(\ghc_{k^{\star}}, k^{\star}, u, \varepsilon)$.  There exists a constant $K > 0$ that depends on $(C_1, C_2,  \Delta, k^{\star}, \cgp)$ such that if
\eq{
&  n \geq \frac{ K  M^{k^{\star}}  \log^2 \left( \frac{35nd}{M}  \right)  \log^{2 C_2} \left(18 n d^{u+1} \right)   \big( \rho_1 \log^{\rho_2} d  \big)^{k^{\star}} }{c^{2(k^{\star} - 1)}} \\[1.5ex]
&  M \geq \log(4 n  d^u) \vee \begin{cases} (\norm{\bs{v}}_{0} + 2) & q = 0 \\[1ex]
(2 - q)   \left[  \left( \norm{\bs{v}}^2_{q}  \vee k^{\star} 2^{\frac{2}{q}} \right)   \frac{q}{2} \left( \rho_1  \log^{\rho_2} d \right)^{k^{\star}} \right]^{\frac{q}{2-q}} & q \in (0,2)
 \end{cases} \label{prunesingle:hps}
}
with probability at least $1 - 4 d^{-u}$, Algorithm \ref{alg:pruningalg} returns    $\mathcal{J} \subseteq [d]$  such that
\eq{
\norm{\bs{v} \vert_{\cJ^c}}_2^2 \leq  K \frac{ \ghc^{ - \left( \frac{2}{k^{\star} - 1} \wedge 2  \right)}_{k^{\star}}}{\rho_1 \log^{\rho_2} d}.
}
\end{lemma}

\begin{proof}
We choose any $u \in \N$. We consider the intersection of the following events:
\begin{enumerate}[label = C.\arabic*]
\item \label{item:existb} There exists $j \in [m]$ such that $\bzj \geq 0$.
\item    \label{item:ytheta} $\sup_{\substack{J \subseteq [d] \\ \abs{J} = 2M}} \sup_{\pspace} \norm{\bs{Y}_{\theta} \vert_{\cJ}}_2^2 \leq  \frac{ K \tilde M \log^{2C_2} \left(  6nd^{u}\right) }{n}$
\item  \label{item:pruneconc} Proposition \ref{prop:gradl2diffconc} holds with $\delta = d^{-u}$.
\item  \label{item:alpconc} Proposition \ref{prop:alphaconc} holds with $\delta = d^{-u}$.
\end{enumerate}
It is easy to verify that the intersection of \eqref{item:existb}-\eqref{item:alpconc} holds with probability at least $1 - 4 d^{-u}$ when $d$ is larger than a constant depending on $(k^{\star},u,\epsilon)$.  We consider $k^{\star} = 1$ and $k^{\star} > 1$ cases separately.

\noindent
For $k^{\star} = 1,$  let $\tilde{\cJ}$ be the set of indices added in Line  \ref{algprune:firstHermite}.  For $j \in [m]$ with $\bzj \geq 0$, we have 
\eq{
\norm*{\frac{1}{2} \ghc_1 \bs{v} \vert_{\cJ^c}}_2^2  \labelrel\leq{prunemain:ineqq0}  \norm*{\rhc_1(\bzj) \ghc_1 \bs{v} \vert_{\tilde{\cJ}^c}}_2^2 
&  \labelrel={prunemain:eqq1} \norm{\grad_j R^{-}(\az, \te_d,  \bz)  -   \grad_j R^{-}(\az, \te_d,  \bz) \vert_{\tilde{\cJ}} }_2^2 \\
& \labelrel\leq{prunemain:ineqq2}  \norm{ \grad_j R^{-}(\az,  \te_d,  \bz)  - \widetilde \grad_j R^{-}_n( \te_d)}_2^2  \label{prunemain:arg1b}  
}
where we use $\cJ \supseteq \tilde{\cJ}$ and $\bzj \geq 0$ (see \eqref{eq:constantterms}) in  \eqref{prunemain:ineqq0},   $\grad_j R^{-}(\az,  \te_d, \bz)  = -\azj \rhc_1(\bzj) \ghc_1 \bs{v}$ (since $\bs{V}_{d*} = 0$) in    \eqref{prunemain:eqq1},   and that $\norm*{\bs{x} \vert_{\tilde{\cJ}^c}}_2 \leq   \norm{\bs{x} - \bs{y}\vert_{\tilde{\cJ}}}_2$ in \eqref{prunemain:ineqq2}.  By using  \eqref{item:pruneconc}  with $k^\star = 1$, we have
\eq{
\eqref{prunemain:arg1b}  \leq  \left\{\begin{aligned}\quad
&                 \frac{ K \tilde M \log^{2C_2} \left(  12 n d^{1+u} \right) }{n}    &   q = 0,  M \geq \norm{\bs{v}}_0 + 2   \\
&                \frac{ K \tilde M\log^{2 C_2} \left(12 n d^{1+u} \right) }{n}  + C_q  \cgp^2  \frac{  \left(   \tfrac{1}{1 - c^2} \right)^2  \left[  \norm{\bs{v}}_q^2 \vee  2^{\frac{2}{q}} \right]  }{M^{ \frac{2}{q} - 1}}  \hspace{-4em}   & q  \in (0,2). 
     \end{aligned}\right.  \label{prunemain:arg1e} 
}By  \eqref{prunesingle:hps},  the statement follows for $k^{\star} = 1$.

\noindent
For $k^{\star} > 1$ and even , we assume $d$ is high enough that
\eq{
c \leq \frac{1}{4} ~~ \text{and} ~~   \left[  \left( \frac{  \mrhc_{k^{\star}} \abs{ \ghc_{k^{\star}}}}{(k^{\star} - 1)!} \right)^\frac{2}{k^{\star} - 1}    -   8  \left(  \frac{\sqrt{2} c}{1 - c^2} \right)^{\frac{2}{k^{\star} -1}} \cgp^\frac{2}{k^{\star} - 1}   \right]  \geq  \left( \frac{1}{2} c_{k^{\star}}  \ghc^2_{k^{\star}} \right)^{\frac{1}{k^{\star} - 1}}, \label{prunesing:dlarge}
}
where $c_{k^{\star}}$ is the constant in  Proposition \ref{prop:alphaconc}.  Let 
\eq{
\bs{u} & \coloneqq 1/\sqrt{2m}  ~ ( \norm{\grad R^{+} ( \az,  \te_1,  \bz)}_F,  \cdots,  \norm{\grad R^{+} ( \az,  \te_d, \bz)}_F ) \\
\tilde{\bs{u}} & \coloneqq   1/\sqrt{2m} ~   ( \norm{\widetilde \grad R^{+} (  \te_1) }_F,   \norm{\widetilde \grad R^{+} (  \te_2) }_F,  \cdots,  \norm{\widetilde \grad R^{+} ( \te_d)}_F ).  \label{prunesing:defu}
}
In the following,  we will first bound $\sum_{j \in \mathcal{J}^c} \bs{u}_j^\frac{2}{k^{\star} - 1}$,  and then use  Proposition \ref{prop:pruningalgaux} with \eqref{prunesing:dlarge} to prove our statement.   Let  $\tilde{\cJ}$ be the set of indices added on  Line  \ref{algprune:prunepos}.   By using Lemma \ref{lem:prunelemma}, we can write
\eq{
\sum_{j \in \mathcal{J}^c} \bs{u}_j^{\frac{2}{k^{\star} - 1}} \leq  \norm{\bs{u} - \bs{u} \vert_{\widetilde{\cJ}}}_{\frac{2}{k^{\star} - 1}}^{\frac{2}{k^{\star} - 1}}    \label{prunesing:eq2b} 
& \leq   \norm{\bs{u} -  \tilde{\bs{u}} \tm }_{\frac{2}{k^{\star} - 1}}^{\frac{2}{k^{\star} - 1}} \\
& \leq 4 \norm{\bs{u} - \bs{u} \tm}_{\frac{2}{k^{\star} - 1}}^{\frac{2}{k^{\star} - 1}} + 5 \sup_{ \substack{ \mathcal{I} \subseteq [d] \\ \abs{ \mathcal{I} } = 2M }} \sum_{i \in \mathcal{I}} \abs{\bs{u}_i - \tilde{\bs{u}}_i}^{\frac{2}{k^{\star} - 1}}.     \label{prunesing:eq2}
}
Moreover, by  Corollary \ref{cor:residualevenodd} (with $N = k^{\star}-2$) and $c \leq 1/4$,
\eq{
\bs{u}_i^{\frac{2}{k^{\star} - 1}} = \norm*{(1/\sqrt{2m}) \grad R^{+}(\az, \te_i, \bz)}_F^{\frac{2}{k^{\star} - 1}}  & \leq  (1 + \sqrt{k^{\star}})^{\frac{2}{k^{\star} - 1}}  \cgp^{\frac{2}{k^{\star} - 1}}     \frac{c^2}{(1 - c^2)^{\frac{2}{k^{\star} - 1}}}   \abs{\bs{v}_i}^{2} \\
& \leq 12 \cgp^{\frac{2}{k^{\star} - 1}}     c^{2  } \abs{\bs{v}_i}^{2}, ~  \label{prunesing:boundforui}
}
where we use that  $(1 + \sqrt{k^{\star}})^{\frac{2}{k^{\star} - 1}}$  is non-increasing for $k^{\star}\geq 2$ in the last step.  By Lemma \ref{lem:donoho},  we have
\eq{
\norm{\bs{u} - \bs{u} \tm}_{\frac{2}{k^{\star} - 1}}^{\frac{2}{k^{\star} - 1}}  \leq  12 c^2  \cgp^{\frac{2}{k^{\star} - 1}}  \norm{\bs{v} - \bs{v} \tm}_{2}^{2} 
& \leq  12 c^2  \cgp^{\frac{2}{k^{\star} - 1}} \left\{\begin{aligned}\quad
&               0   &   q = 0,  M \geq \norm{\bs{v}}_0+ 2   \\
&                \frac{\left( 1 - \frac{q}{2} \right)^{\frac{2-q}{q}} \frac{q}{2} \norm{\bs{v}}_q^2 }{M^{\frac{2}{q} - 1} }  \hspace{-4em}   & q  \in (0,2) 
     \end{aligned}\right. \\
& \leq \frac{12 c^2 \cgp^{\frac{2}{k^{\star} - 1}}}{\rho_1 \log^{\rho_2} d},   \label{prunesing:boundforutm}
}
where we used   \eqref{prunesingle:hps}.  Moreover,   we have 
\eq{
\sup_{ \substack{ \mathcal{I} \subseteq [d] \\ \abs{ \mathcal{I} } = 2M }} \sum_{i \in \mathcal{I}} \abs{\bs{u}_i - \tilde{\bs{u}}_i}^{\frac{2}{k^{\star} - 1}} \leq \sup_{ \substack{ \mathcal{I} \subseteq [d] \\ \abs{ \mathcal{I} } = 2M }}   \sum_{i \in \mathcal{I}} (2m)^{\frac{- 1}{k^{\star} - 1}}      \norm{\widetilde \grad R^{+}_n(\te_i) - \grad R^{+}(\az, \te_i, \bz) }^{ \frac{2}{k^{\star} - 1} }_F, ~~~~~~~~  \label{prunesing:secondterma}
}
where by \eqref{item:pruneconc}, we have 
\eq{
& \forall i \in [d]; ~  (2m)^{-1} \norm{\widetilde \grad R^{+}_n(\te_i) - \grad R^{+}(\az, \te_i, \bz) }^2_F \\[1ex]
&\leq  \left\{\begin{aligned}\quad
&               \frac{ K \tilde{M} \log^{2C_2} \left( 12 n d^{1+ u}\right) }{n}   &   q = 0,  M \geq \norm{\bs{v}}_0 + 2   \\
&               \frac{K \tilde{M} \log^2 \left( \frac{35nd}{M}  \right) \log^{2C_2}  \left( 12 n d^{1+ u}\right) }{n}  +   \frac{ C_q \frac{\cgp^2  c^{2(k^{\star} - 1)}}{(1 - c^2)^2}    \abs{\bs{v}_i}^{2(k^{\star} - 1)}  \left[  \norm{\bs{v}}_q^2 \vee  k^{\star} 2^{\frac{2}{q}} \right]  }{M^{ \frac{2}{q} - 1}}  \hspace{-5.5em}   & q  \in (0,2). 
     \end{aligned}\right.   \\  \label{prunesing:secondtermb}
}Therefore,  by  \eqref{prunesingle:hps}, we have  $\eqref{prunesing:secondterma} \leq \frac{c^2 \tilde K}{\rho_1 \log^{\rho_2} d}$, where $\tilde K$ depends on $(C_1, C_2,  \Delta, k^{\star}, \cgp)$. By  \eqref{prunesing:eq2} and \eqref{prunesing:boundforutm}, the statement follows.
\end{proof}

\begin{lemma}[Multi-Index Setting]
\label{lem:pruningalgmainmultindex}
Consider the multi-index setting.   For $u \in \N$ and a small constant $\varepsilon > 0$,  let
\eq{
m = \Theta(d^\varepsilon), ~~ d \geq    d(\sigma_r(\bs{H}), u, \varepsilon)   \vee 4M  ~~ \text{and} ~~ c \leq \frac{1}{\log d},  \label{prunesingle:hypers1}
}
and  $\rho_1, \rho_2 \geq 1$.  There exists a  constant $K > 0$ that depends on $(C_1, C_2,  \Delta, r,   \cgp)$ such that if
\eq{
&  n \geq \frac{ K  M^{2}  \log^2 \left( \frac{35nd}{M}  \right)  \log^{2 C_2} \left(18 n d^{u+1} \right)   \big( \rho_1 \log^{\rho_2} d  \big) }{c^2} \\[1.5ex]
&  M \geq \log(4 n d^u) \vee \begin{cases} (\norm{\bs{V}}_{2,0} + 2) & q = 0 \\
(2 - q)   \left[  \left( \norm{\bs{V}}_{2,q}^2 \vee 2^{\frac{2}{q} + 1} \right)   \frac{q}{2} \left( \rho_1  \log^{\rho_2} d \right) \right]^{\frac{q}{2-q}} & q \in (0,2)
 \end{cases} \label{prunesingle:hps2}
}
with probability at least $1 - 4  d^{-u}$, Algorithm \ref{alg:pruningalg} returns    $\mathcal{J} \subseteq [d]$  such that
\eq{
\norm{\E[y \bs{x}]  \vert_{\cJ^c} }_2^2 \vee \norm{\bs{V}  \vert_{\cJ^c}}_F^2 \leq    \frac{K \sigma^{-2}_r(\bs{H}) }{\rho_1 \log^{\rho_2} d}.
}
\end{lemma}

\begin{proof}
We will follow the same arguments in the proof of Lemma \ref{lem:pruningalgmainsingleindex}.  We choose any $u \in \N$.  We consider the intersection of \eqref{item:existb}-\eqref{item:alpconc} above, 
which holds with probability at least $1 - 4 d^{-u}$. 

\noindent
For $\norm{\E[y \bs{x}]  \vert_{\cJ^c} }_2^2$,  let $\tilde{\cJ}$ be the set of indices added in Line  \ref{algprune:firstHermite}.  For $j \in [m]$ with $\bzj \geq 0$, we have 
\eq{
\norm*{\frac{1}{2}   \E[y \bs{x}]  \vert_{\cJ^c}  }_2^2  \labelrel\leq{prunemult:ineqq0}  \norm*{\rhc_1(\bzj)  \E[y \bs{x}] \vert_{\tilde{\cJ}^c}}_2^2 
&  \labelrel={prunemult:eqq1} \norm{\grad_j R^{-}(\az, \te_d,  \bz)  - \grad_j R^{-}(\az, \te_d,  \bz) \vert_{\tilde{\cJ}} }_2^2,  \\ \label{prunemult:eq99} 
}where we use $\cJ \supseteq \tilde{\cJ}$ and $\bzj \geq 0$  in  \eqref{prunemult:ineqq0} (see  \eqref{eq:constantterms}),   $\grad_j R^{-}(\az, \te_d,  \bz)  = - \azp_j \rhc_1(\bzj) \ghc_1 \bs{v}$ (since $\bs{V}_{d*} = 0$) in    \eqref{prunemult:eqq1}. By   \eqref{item:pruneconc}, we have
\eq{
\eqref{prunemult:eq99}  \leq  \left\{\begin{aligned}\quad
&                 \frac{ K \tilde M \log^{2C_2} \left(  12 n d^{1+u} \right) }{n}    &   q = 0,  M \geq \norm{\bs{V}}_{2,0} + 2   \\
&                \frac{ K \tilde M\log^{2 C_2} \left(12 n d^{1+u} \right) }{n}  + C_q  \cgp^2  \frac{  \left(   \tfrac{1}{1 - c^2} \right)^2  \left[  \norm{\bs{V}}_{2,q}^2 \vee  2^{\frac{2}{q}} \right]  }{M^{ \frac{2}{q} - 1}}  \hspace{-2em}   & q  \in (0,2). 
     \end{aligned}\right.
}
By  \eqref{prunesingle:hps2},  the statement follows for $\norm{\E[y \bs{x}]  \vert_{\cJ^c} }_2^2$.

\smallskip
For $\norm{\bs{V} \vert_{\cJ^c}}_F^2$,   we assume $d$ is high enough that
\eq{
c \leq \frac{1}{4} ~~ \text{and} ~~   \left[     \mrhc_{2}^2  \sigma^2_r(\bs{H})    -   16  \left(  \frac{c}{1 - c^2} \right)^{2} \cgp^2  \right] \geq  \frac{1}{2} c_{2} \sigma^2_r(\bs{H}) .  \label{prunemult:dlarge}
}
where $c_{2}$ is the constant in Proposition \ref{prop:alphaconc} for $k = 2$.
Let $\bs{u}$ and $\tilde{\bs{u}}$ be the vectors defined in  \eqref{prunesing:defu} and let
$\tilde{\cJ}$ be the set of indices added on  Line  \ref{algprune:prunepos}.   By following the arguments in   \eqref{prunesing:eq2b}-\eqref{prunesing:eq2} with $k^{\star} = 2,$  we can write
\eq{
\sum_{j \in \mathcal{J}^c} \bs{u}_j^{2}  \leq 4 \norm{\bs{u} - \bs{u} \tm}_2^2 + 5 \sup_{ \substack{ \mathcal{I} \subseteq [d] \\ \abs{ \mathcal{I} } = 2M }} \sum_{i \in \mathcal{I}} \abs{\bs{u}_i - \tilde{\bs{u}}_i}^2     \label{prunemult:eq2}
}
For $\bs{v}  \coloneqq (\norm{\bs{V}_{1*}}_2, \cdots, \norm{\bs{V}_{d*}}_2)$, by following the arguments in \eqref{prunesing:boundforui} and   \eqref{prunesing:boundforutm}, we can write that
\eq{
\norm{\bs{u} - \bs{u} \tm}_2^2  \leq  12 c^2  \cgp^{2}  \norm{\bs{v} - \bs{v} \tm}_{2}^{2} 
& \leq  12 c^2  \cgp^{2}  \left\{\begin{aligned}\quad
&       \!\!\!\!\!\!             0   &   q = 0,  M \geq \norm{\bs{V}}_{2,0} + 2   \\
&      \!\!\!\!\!\!            \frac{\left( 1 - \frac{q}{2} \right)^{\frac{2-q}{q}} \frac{q}{2}  \norm{\bs{V}}_{2,q}^2  }{M^{\frac{2}{q} - 1} }   \hspace{-2.5em}   & q  \in (0,2) 
     \end{aligned}\right.  \\  
& \leq \frac{6 c^2  \cgp^2}{\rho_1 \log^{\rho_2} d}.  \label{prunemult:boundforutm}
}
Moreover, by following the arguments in    \eqref{prunesing:secondterma} and \eqref{prunesing:secondtermb},  we can show that
\eq{
 \sup_{ \substack{ \mathcal{I} \subseteq [d] \\ \abs{ \mathcal{I} } = 2M }}   \sum_{i \in \mathcal{I}} \abs{\bs{u}_i - \tilde{\bs{u}}_i}^2  
 &  \leq
\left\{\begin{aligned}\quad
&   \!\!\!\!\!\!     \frac{   K  M^2  \log^2 \left( \frac{35nd}{M}  \right) \log^{2C_2}    \left( 12 n d^{1+ u}\right) }{n}    &   q = 0,  M \geq \norm{\bs{V}}_{2,0} + 2   \\
&    \!\!\!\!\!\!    \frac{   K  M^2 \log^{2} \left( \frac{35nd}{M}  \right) \log^{2C_2}    \left( 12 n d^{1+ u} \right) }{n} +      \frac{ r   C_q  \frac{\cgp^2c^{2}}{1 - c^2}    \left[  \norm{\bs{V}}_{2,q}^2 \vee 2^{\frac{2}{q} + 1} \right] }{M^{ \left( \frac{2}{q} - 1\right)}}    \hspace{-5em}   & q  \in (0,2) 
     \end{aligned}\right.  \\ 
& \leq    \frac{2c^2}{\rho_1 \log^{\rho_2} d} +  \frac{  32 r \cgp^2  c^2 }{\rho_1 \log^{\rho_2} d}    \label{prunemult:eq3}
}
By  the arguments between \eqref{prunemult:dlarge}-\eqref{prunemult:eq3},  the statement follows.
\end{proof}

\section{Feature Learning}
%
%

\subsection{Additional Notation and Terminology}
In the following, we will use SI for the single-index setting and MI for the multi-index setting.   In the following,  we assume $\abs{\cJ} \leq M$ and ignore the constants.  For SI, we consider a polynomial link function $\gt: \R \to \R$ such that $\gt(t) = \sum_{k \leq p} c_k t^k$.   For MI, we consider  a polynomial link function
$\gt: \R^r \to \R$ and $\tgt(\bs{z}) = \gt(\bs{z}) - \inner{\E[y \bs{x}]}{\bs{z}} = \sum_{k \leq p} \inner{\bm{ \tilde T_k }}{\bs{z}^{\otimes k}}$.  

Henceforth,  $\bs{w} \sim \cN(0, \ide{d})$ is a random vector independent of the remaining random variable unless otherwise stated.  Let $\wJ \coloneqq \frac{\bs{w} \vert_{\cJ}}{\norm{\bs{w} \vert_{\cJ}}_2}$. Let $vec(\bm{T})$ denotes the vectorized version of the tensor $\bm{T}$ and
\eq{
\sJ & \coloneqq \begin{cases}
\inner{\bs{v}}{\wJ}^{k^{\star} - 1} & \text{SI} \\
\bs{D} \bs{V}^\top \wJ   & \text{MI} 
\end{cases}  \\
z_k(\sJ) & \coloneqq \begin{cases}
0 & \sJ = 0 \\
c_k \E_{\bs{w}}[\sJ^{2k}]^{-1} \sJ^k & \text{SI and} ~ \sJ \neq 0 \\
\inner{vec(\bm{\tilde T_k})}{\E \left[ vec(\sJ^{\otimes k})  vec(\sJ^{\otimes k})^\top \right]^+ vec(\sJ^{\otimes k})}   & \text{MI and}  ~ \sJ \neq 0 
\end{cases} 
}
where $A^+$ denotes the pseudoinverse of $A$.  We will use
\eq{
\shc(\bzl) \coloneqq  \begin{cases}
\frac{\ghc_{k^{\star}} \rhc_{k^{\star}}(\bzl)}{(k^{\star} - 1)!} & \text{SI} \\
\rhc_2(\bzl)   & \text{MI} 
\end{cases}   ~~ \text{and} ~~
\ngt \coloneqq \sum_{l = 1}^N \indic{ \abs{ \shc(\bzl) } \geq \tau },
}
where $\tau$ will be specified later.   

\subsection{Auxiliary Results}

\begin{lemma}[{\cite[Lemma 9]{Damian2022NeuralNC}} with explicit constants]
\label{lem:damianetal}
Let $a \sim \text{Unif}( \{ -1, 1 \} )$ and $b \sim \cN(0,1)$. Then for any $k \geq 0$,  there exists $v_k(a,b)$ such that for $\abs{x} \leq 1$,
\eq{
\E \left[ v_k(a,b) \phi(a t + b)  \right] = t^k  ~~ \text{and} ~~ \sup_{a, b} \abs{v_k(a,b)} \leq  6 \sqrt{2} (k+1)^2.
}
\end{lemma}

\begin{proof}
By following the constants in \cite[Lemma 9]{Damian2022NeuralNC}, we have the statement.
\end{proof}

\begin{lemma}[{\cite[Lemma 21]{Damian2022NeuralNC}} with explicit constants]
\label{lem:damianetal2}
Let $\gt : \R^r \to \R$ be a polynomial of degree-$p$ such that $\E[\gt(\bs{z})^2] \leq 1.$  There exists symmetric $\bm{\tilde T_0}, \cdots, \bm{\tilde T_p}$ such that $\gt(\bs{z}) = \sum_{k = 0}^p \inner{\bm{\tilde T_k}}{\bs{z}^{\otimes k}}$ where
\eq{
\norm{\bm{\tilde T_k}}^2_F \leq    \frac{2 e^k}{k! }   (e \sqrt{r})^{\floor*{ \frac{p - k}{2} } }.
}
Consequently,  we have  $\sum_{k = 0}^p \norm{\bm{\tilde T_k}}_F (k+1)^2 \leq C (e \sqrt{r})^{\frac{p}{4}}$,  where $C > 0$ is a universal constant.
\end{lemma}

\begin{proof}
Let $ \gt(\bs{z}) = \sum_{j = 0}^p \frac{1}{j!} \inner{\bm{T_j}}{\bm{\Hek}}$.  Then,
\eq{
\bm{\tilde T_k} k! = \grad^k \gt(0) & = \sum_{j = 0}^{p - k}  \frac{1}{j!}  \bm{\grad^k T_{j+k}}[ \bm{\Hek}(0)]   \labelrel={dam2:eqq0} \sum_{\substack{j = 0 \\ j ~ even}}^{p - k}  \frac{(-1)^{j/2}(j - 1)!!}{j!}  \bm{\grad^k T_{j+k}}[sym( \ide{r}^{\otimes \frac{j}{2}} )]
}
where \eqref{dam2:eqq0} follows by Lemma \ref{lem:hermiteatzero} and since $\bm{\grad^k T_{j+k}}$ is symmetric by Lemma \ref{lem:gradtensor}.
Therefore,
\eq{
\norm{\bm{\tilde T_k} k! }_F 
\labelrel\leq{dam2:ineqq1}     \sum_{\substack{j = 0 \\ j ~ even}}^{p - k}  \frac{ (j - 1)!!}{j!}    \norm{  \bm{ T_{j+k}} }_F  \norm{sym( \ide{r}^{\otimes \frac{j}{2}} )}_F    
\labelrel\leq{dam2:ineqq2}    \sum_{\substack{j = 0 \\ j ~ even}}^{p - k}  \frac{ (j - 1)!!}{j!}  r^{\frac{j}{4}}    \norm{   \bm{T_{j+k}} }_F.  
}
where \eqref{dam2:ineqq1} follows Cauchy-Schwartz inequality and Lemma \ref{lem:gradtensor}, and    \eqref{dam2:ineqq2} follows  \cite[Lemma 3]{damian2023smoothing}.
Therefore,
\eq{
\norm{\bm{\tilde T_k} k! }_F^2 
\labelrel\leq{dam2:ineqq3}    \sum_{\substack{j = 0 \\ j ~ even}}^{p - k}  \frac{  \norm{  \bm{ T_{j+k}} }_F^2}{(j+k)!}  \sum_{\substack{j = 0 \\ j ~ even}}^p \left(  \frac{ (j - 1)!!}{j!}  \right)^2 r^{\frac{j}{2}} (j+k)!  
& \labelrel\leq{dam2:ineqq4}  \sum_{\substack{j = 0 \\ j ~ even}}^{p - k} \left(  \frac{ (j - 1)!!}{j!}  \right)^2 r^{\frac{j}{2}} (j+k)!   \\
& \labelrel\leq{dam2:ineqq5} k!  \sum_{\substack{j = 0 \\ j ~ even}}^{p - k} \binom{j+k}{k}  r^{\frac{j}{2}}.  \label{dam2:eq0}
}
where \eqref{dam2:ineqq3} follows from Cauchy-Schwartz inequality,  \eqref{dam2:ineqq4} follows $\E[\gt(\bs{z})^2] \leq 1$,   and \eqref{dam2:ineqq5} follows $(j-1)!!^2 \leq j!$. Therefore,
\eq{
\eqref{dam2:eq0} 
\labelrel\leq{dam2:ineqq6}  k! e^k  \sum_{\substack{j = 0 \\ j ~ even}}^{p - k}   (e^2 r)^{\frac{j}{2}} 
 =    k! e^k   \sum_{j = 0}^{\floor*{ \frac{p - k}{2} } }   (e \sqrt{ r} )^{j} 
 \labelrel\leq{dam2:ineqq7} 2    k! e^k    (e \sqrt{ r} )^{ \floor*{ \frac{p - k}{2} }}.
}
where \eqref{dam2:ineqq3} follows  $\binom{j+k}{k} \leq e^{j + k}$.  For the second part of the statement,  let $\sup_{k \geq 0} \frac{2 e^k (k+1)^4}{k!}  = C < \infty$ (as $k!$ grows faster than $e^k (k+1)^4$). We have
\eq{
\sum_{k = 0}^p \norm{\bm{\tilde T_k}}_F (k+1)^2   \leq \sum_{k = 0}^p  \left( \frac{2 e^k (k+1)^4}{k! } \right)^{1/2} (e \sqrt{r})^{\frac{p - k}{4}} 
\leq   C^{1/2}  \sum_{k = 0}^p   (e \sqrt{r})^{\frac{p - k}{4}}  
\leq \tilde C (e \sqrt{r})^{\frac{p}{4}}. 
}
\end{proof}

%
%

\begin{proposition}
\label{prop:condz}
We consider  MI (i.e.,  $\sJ = \bs{D} \bs{V}^\top \wJ$). For $k \in \N$ and $d \geq 2k$, we have 
\eq{
\inf_{ \substack{ \bm{T_k} : (\R^r)^{\otimes k} \to \R \\  \bm{T_k} ~ \text{is symmetric} \\ \norm{ \bm{T_k} }_F = 1  }}  \inner{ vec(\bm{T_k})}{\E_{\bs{w}}[ vec( \sJ^{\otimes k}) vec(\sJ^{\otimes k})^\top ] vec(\bm{T_k}) } \geq k! \frac{   \sigma_r^{2k}( \bs{V} \vert_{\cJ} \bs{D}) }{ \E \left[\norm{\bs{w} \vert_{\cJ}}_2^{2k} \right]}. 
}
\end{proposition}

\begin{proof}
Let  $\bm{T_k} : (\R^r)^{\otimes k} \to \R$ be a symmetric tensor with  $\norm{ \bm{T_k} }_F^2 = 1$ .  We have  
\eq{
\inner{ vec(\bm{T_k})}{\E[ vec( \sJ^{\otimes k}) vec(\sJ^{\otimes k})^\top] vec(\bm{T_k}) }  
=\E \left[\norm{\bs{w} \vert_{\cJ}}_2^{2k} \right]^{-1}  \E  \left[\inner{\bm{T_k}}{(\bs{D }\bs{V}^\top \bs{w} \vert_{\cJ})^{\otimes k}} ^2\right],  ~~~~ \label{condz:eq0}
}
where we use that $w/\norm{\bs{w}}_2$ and $\norm{\bs{w}}_2$ are independent.  Let $\bm{\hat T_k} : (\R^d)^{\otimes k} \to \R$ such that
\eq{
\bm{\hat T_k}[\bs{u}_1, \cdots, \bs{u}_k] = T_k[\bs{D} \bs{V}\vert_{\cJ} ^\top \bs{u}_1, \cdots, \bs{D} \bs{V}\vert_{\cJ} ^\top \bs{u}_k].
}
By using Lemma  \ref{lem:tensorfroblb} and \cite[Lemma 23]{Damian2022NeuralNC}, we have $$\eqref{condz:eq0} \geq k! \norm{\bm{ \hat T_k } }_F^2  \E \left[\norm{\bs{w} \vert_{\cJ}}_2^{2k} \right]^{-1} \geq k! \sigma_r^{2k}(\bs{V} \vert_{\cJ} \bs{D})  \E \left[\norm{\bs{w} \vert_{\cJ}}_2^{2k} \right]^{-1}.$$
\end{proof}


\begin{lemma}
\label{lem:taudef}
There exists $\tau> 0$ (that depends on $(k^{\star},  \ghc_{k^{\star}})$ for SI and universal for MI) such that for $b \sim \cN(0,1),$ we have
\eq{
& \mpr\left [\abs{ \shc(b)}  \geq \tau \right] \geq \frac{2}{3} ~~ \text{and} ~~ \mpr \left[  \frac{\ngt}{N} \geq \frac{1}{3}  \right] \geq 1 - \exp\left(  - \frac{2N}{9} \right). \label{biasagg:count}
}
\end{lemma}

\begin{proof}
In the following,  we will prove an anti-concentration result for $\rhc_k(b),$ $k \in \N$.  Note that by scaling the $k = k^{\star}$ case with  $\abs{\ghc_{k^{\star}}},$  the statement can be extended to SI.   MI immediately follows from the $k = 2$ case.

\noindent
For $k = 1$,  since $\rhc_k(b) \sim \text{Unif}[0,1]$, if we take $\tau =  1/3$, we have the first statement.  For $k = 2$,  since  $\rhc_k(b) = \frac{e^{\frac{-b^2}{2}}}{\sqrt{2\pi}}$, if we choose $\tau = \frac{1 }{e \sqrt{2\pi}}$,  we have 
\eq{
\mpr\left [ \rhc_k(b)  \geq \tau \right]  = \mpr\left [ \abs{b}  \leq \sqrt{2} \right]  \labelrel\geq{taudef:ineqq0} 1 - \frac{1}{e \sqrt{2}} \geq \frac{2}{3}. 
}
where we use  $\mpr[\abs{b} \geq t] \leq \tfrac{e^{-t^2/2}}{t}$ for  \eqref{taudef:ineqq0}.
For $k \geq 3$,  we have
\eq{
& \abs{ \rhc_k(b) }    \leq       \frac{1/ (e^2 \sqrt{2\pi})}{(2C)^{k-2}} \left( \frac{\varepsilon}{k-2} \right)^{k - 2}              \sqrt{  (k - 2)! }  \Rightarrow \abs{b} \geq 2 ~ \text{OR} ~   \abs{H_{e_{k -2}}  (- b)}      \leq       \left( \frac{\varepsilon/2C}{(k-2) } \right)^{k - 2}                \sqrt{ (k - 2)! },
}where $C$ is the constant appeared in \cite[Theorem 8]{Carbery2001DistributionalAL}.
Therefore, if we choose 
\eq{
\tau =        \frac{1/ (e^2 \sqrt{2\pi})}{(2C)^{k -2}}  \left( \frac{\varepsilon}{k -2} \right)^{k - 2} \frac{ \sqrt{ (k - 2)! } }{(k - 1)!}, 
}
by \cite[Theorem 8]{Carbery2001DistributionalAL},  we have
\eq{
\mpr\left [\abs{ \rhc_k(b)}  \leq \tau \right] 
 \leq  \mpr \left[ \abs{b} \geq 2 \right] +   \mpr \left[ H^2_{e_{k -2}}(- b) \leq  \frac{1}{C^{2k -4}}  \left( \frac{\varepsilon}{2k -4} \right)^{2k - 4}  (k - 2)!  \right]  
\leq \frac{1}{2 e^2} + \varepsilon.  
}
By choosing $\varepsilon = \tfrac{1}{6}$,  we have  the first part of the statement for $k \geq 3$ as well.   The second part  follows from Hoeffding's inequality and the result in first part.

\end{proof}

\subsubsection{Lemmas for Moments}

\begin{lemma}
\label{lem:errterm}
For any event $E$,
\eq{
& \text{SI}: ~~  \abs*{ \E_{\bs{w}} \left[ z_k(\wJ)  \sJ^k \inner{\bs{v}}{\bs{x}_i }^k \indic{E}   \right] }  \leq   \abs{c_k}  9^{k(k^{\star} - 1)} \abs{\inner{\bs{v}}{\bs{x}_i}}^{k} \mpr[E]^{1/2}  \\
& \text{MI}: ~  \abs*{ \E_{\bs{w}} \left[ z_k(\wJ)  \inner{ \sJ }{  \bs{V}^\top \bs{x}_i }^k \indic{E}   \right] } \leq  \frac{2^{k}}{(4k)^{1/4}}  \frac{\sigma_1^{k}( \bs{V} \vert_{\cJ} \bs{D}) }{\sigma_r^{k}( \bs{V} \vert_{\cJ} \bs{D}) } \norm{\bm{\tilde T_k}}_F \norm{\bs{V}^\top \bs{x}_i}_2^{k} \mpr[E]^{1/4} .
}
\end{lemma}

\begin{proof}
For SI:
\eq{
 \abs*{ \E_{\bs{w}} \left[ z_k(\wJ)  \sJ^k \inner{\bs{v}}{\bs{x}_i }^k \indic{E}   \right] } & \labelrel\leq{err:ineqq0}  \abs{c_k} \abs{ \inner{\bs{v}}{\bs{x}_i }}^k  \E_{\bs{w}} \left[ \left( \frac{\sJ^{2k}}{\E_{\bs{w}}[\sJ^{2k}]} \right)^2 \right]^{1/2} \mpr[E]^{1/2}   \\
& \labelrel\leq{err:ineqq1}   \abs{c_k} \abs{ \inner{\bs{v}}{\bs{x}_i }}^k   9^{k(k^{\star} - 1)} \mpr[E]^{1/2},
} 
where we used Cauchy-Schwartz inequality for \eqref{err:ineqq0} and Lemma \ref{lem:hypercontactivity} for  \eqref{err:ineqq1}. 

\noindent
For MI:  By using  Cauchy-Schwartz inequality,
\eq{
\E_{\bs{w}} \left[ z_k(\wJ)  \inner{ \sJ }{  \bs{V}^\top \bs{x}_i }^k  \indic{E}   \right] \leq \E_{\bs{w}} \left[ z^2_k(\wJ) \right]^{1/2} \E \left[  \inner{ \sJ }{  \bs{V}^\top \bs{x}_i }^{4k} \right]^{1/4} \mpr[E]^{1/4}. \label{errterm:eq0}
}
We have
\eq{
\E \left[  \inner{ \sJ }{  \bs{V}^\top \bs{x}_i }^{4k} \right]^{1/4}
& =  \norm{ ( \bs{H} \bs{x}_i ) \vert_{\cJ} }_2^{4k} (4k - 1)!!  \E \left[ \norm{\bs{w} \vert_{\cJ}}_2^{4k}  \right]^{-1}  \\
& \leq \sigma_1^{4k}( \bs{V} \vert_{\cJ} \bs{D})  \norm{\bs{V}^\top \bs{x}_i}_2^{4k}  (4k - 1)!!   \E_{\bs{w}} \left[ \norm{\bs{w} \vert_{\cJ}}_2^{4k}  \right]^{-1},  \label{errterm:eq1}
}
where we used $(\bs{H} \bs{x}_i)\vert_{\cJ} = \bs{V} \vert_{\cJ} \bs{D} \bs{V}^\top \bs{x}_i$ in the last step.  Moreover, we have
\eq{
& \E_{\bs{w}} \left[ z^2_k(\wJ) \right]  \\
& = \E_{\bs{w}} \Bigg[         \inner{vec( \bm{\tilde T_k} ) }{\E \left[ vec( \sJ^{\otimes k} )  vec(   \sJ^{\otimes k})^\top  \right]^{+}           vec(   \sJ^{\otimes k})}      \inner{ vec(   \sJ^{\otimes k})  }{\E \left[ vec( \sJ^{\otimes k} )  vec(   \sJ^{\otimes k})^\top  \right]^{+}         vec( \bm{\tilde T_k} )}        \Bigg] \\
& =   \inner{vec( \bm{\tilde T_k} ) }{\E \left[ vec( \sJ^{\otimes k} )   vec(   \sJ^{\otimes k})^\top   \right]^{+}   vec( \bm{\tilde T_k} )}  \\
& \leq  \frac{ \E \left[\norm{\bs{w} \vert_{\cJ}}_2^{2k} \right]}{k!   \sigma_r^{2k}( \bs{V} \vert_{\cJ} \bs{D})  } \label{errterm:eq1.5}
}
where we used  Proposition \ref{prop:condz} in the last line.
By using   \eqref{errterm:eq1} and   \eqref{errterm:eq1.5}, we have
\eq{
\eqref{errterm:eq0} 
\leq     \left( \frac{(4k - 1)!!}{k! k!} \right)^{1/4} \frac{\sigma_1^{k}( \bs{V} \vert_{\cJ} \bs{D}) }{\sigma_r^{k}( \bs{V} \vert_{\cJ} \bs{D}) } \norm{\bs{V}^\top \bs{x}_i}_2^{k}  \mpr[E]^{1/4}
\leq     \frac{2^{k}}{(4k)^{1/4}} \frac{\sigma_1^{k}( \bs{V} \vert_{\cJ} \bs{D}) }{\sigma_r^{k}( \bs{V} \vert_{\cJ} \bs{D}) } \norm{\bs{V}^\top \bs{x}_i}_2^{k}  \mpr[E]^{1/4},  \label{errterm:eq2}
}where we use Stirling's formula in the last step.
\end{proof}

\subsection{Approximation of the target}
We define
\eq{
h(\bs{w},\azp,\bop,\bzl) \coloneqq \sum_{k = 0}^p \frac{v_k(\azp,\bop)}{\eta^k \shc^k(\bzl)} z_k(\sJ) \indic{E},
}
where
\eq{  
E    \equiv     \left\{\begin{aligned}\quad
&    \!\!\!\!\!             \abs{\sJ} \leq \tfrac{1}{\eta \tau} ~~  \text{\small AND}  ~~  \norm{\bs{v} \vert_{\cJ^c}}_2^2    \leq   \tfrac{1}{4}   ~~  \text{\small AND}  ~~  \abs{\shc(\bzl)} \geq \tau    ~~ \text{\small AND} ~~  \max_{i \in [n]}   \eta \abs{ \shc(\bzl) \sJ \inner{\bs{v}}{\bs{x}_i}}   \leq   1      &  \text{SI}  \\
&    \!\!\!\!\!                \norm{\sJ}_2 \leq \tfrac{1}{\eta \tau}  ~  \text{\small AND}  ~    \norm{\bs{V} \vert_{\cJ^c}}_F^2   \leq   \tfrac{1}{4}  ~  \text{\small AND} ~   \abs{\shc(\bzl)}   \geq    \tau   ~  \text{\small AND}  ~   \max_{i \in [n]}    \eta \abs{ \shc(\bzl) \inner{\sJ}{     \bs{V}^\top \bs{x}_i}  } \leq   1     & \text{MI}
     \end{aligned}\right. 
}
\begin{lemma}
\label{lem:approx}
Let us have iid $\{\bzl \}_{l \in [N]}.$  We assume that: For SI,  $M \geq 2p(k^{\star} -1),$  $\ngt > 0$ and  $\norm{\bs{v} \vert_{\cJ^c}}_2^2 \leq \tfrac{1}{4}$.  For MI,  $M \geq 2p,$ $\ngt > 0$ and  $\norm{\bs{V} \vert_{\cJ^c}}_F^2 \leq \tfrac{1}{4}.$  Then,  there exists a constant $C_{k^{\star}} > 0$ depending on $k^{\star}$, and a universal constant $\tilde C > 0 $ such that the following holds:
\begin{itemize}
\item  For SI:
\eq{
&  (\rnu{1})  ~ \abs*{ \E_{(\bs{w},\azp,\bop)}    \left[ \frac{1}{\ngt} \sum_{l = 1}^N  h(  \bs{w}, \azp,\bop, \bzl)  \phi \left(  \azp \eta  \shc  \big(\bzl \big)  \sJ \inner{\bs{v}}{\bs{x}_i} + \bop \right)  \right]     -    \gt(\inner{\bs{v}}{\bs{x}_i}) }  \\
&  \hspace{4em} \leq     C_{k^{\star}} e^{\frac{p}{4}} \left( \max_{k \leq p} \abs{\inner{\bs{v}}{\bs{x}_i}}^k  \right)   \mpr_{\bs{w}} \left[ \abs{\sJ} \geq \frac{1}{\eta \tau} ~ \text{OR} ~ \max_{i \in [n]}   \abs{ \sJ \inner{\bs{v}}{\bs{x}_i} }  > \frac{1}{\eta k^{\star}} \right] ^{\frac{1}{2}} \label{lemapprox:sibias} \\
& (\rnu{2}) ~~ \abs{h( \bs{w}, \azp,\bop, \bzp) } \leq \tilde C e^{\frac{p}{4}} \max_{k \leq p} \frac{ M^{k(k^{\star} - 1)} }{ \eta^{2k} \tau^{2k} }.  \label{lemapprox:sihardbound}
}
\item For MI:
\eq{
& (\rnu{1}) ~ \abs*{ \E_{(\bs{w},\azp,\bop)}    \left[ \frac{1}{\ngt} \sum_{l = 1}^N  h(  \bs{w}, \azp,\bop, \bzl)  \phi \left(  \azp \eta  \shc  \big(\bzl \big)  \inner{\sJ}{\bs{V}^\top \bs{x}_i} + \bop \right)  \right]   -    \tgt(\bs{V}^\top \bs{x}_i) } \\ 
&  \hspace{1em} \leq    C_{k^{\star}} (e \sqrt{r})^{\frac{p}{4}}     \left(  \frac{\sigma_1(\bs{V} \vert_{\cJ} \bs{D})}{\sigma_r(\bs{V} \vert_{\cJ} \bs{D})} \right)^p      \left( \max_{k \leq p} \norm{\bs{V}^\top \bs{x}_i}_2^k  \right)      \mpr_{\bs{w}} \left[ \norm{\sJ}_2 \geq \frac{1}{\eta \tau}  ~ \text{OR} ~   \max_{i \in [n]}  \abs*{\inner{\sJ}{\bs{V}^\top \bs{x}_i}} > \frac{1}{\eta} \right]  ^{\frac{1}{4}} ~~~~ \label{lemapprox:mibias} \\
& (\rnu{2}) ~ \abs{h( \bs{w}, \azp,\bop, \bzp) } \leq  \tilde C (e\sqrt{r})^{\frac{p}{4}} \max_{k \leq p} \frac{ M^{k} }{ \eta^{2k} \tau^{2k} \sigma^{2k}_r(\bs{D}) }.  \label{lemapprox:mihardbound}
}
\end{itemize}
\end{lemma}

\begin{proof}
We start with SI.  Fix an $k \leq p$ and $l \in [N]$.  We have
\eq{
& \E_{(\bs{w},\azp,\bop)} \left[ \frac{1}{\ngt} \sum_{l = 1}^N  h(  \bs{w}, \azp,\bop, \bzl)  \phi \left(  \azp \eta  \shc  \big(\bzl \big)  \sJ \inner{\bs{v}}{\bs{x}_i} + \bop \right)  \right]  \\
& \labelrel={approx:eqq0}   \indic{\abs{\shc(\bzl)} \geq \tau} \E_{\bs{w}} \left[ \indic{E}  z_k(\sJ)  \sJ^k \inner{\bs{v}}{\bs{x}_i}^k \right]  \label{lemapprox:siarg0}   \\
& \labelrel={approx:eqq1}   \indic{\abs{\shc(\bzl)} \geq \tau} \left(  c_k \inner{\bs{v}}{\bs{x}_i}^k -   \E_{\bs{w}} \left[ \indic{E^c}  z_k(\sJ)  \sJ^k \inner{\bs{v}}{\bs{x}_i}^k \right]   \right) 
}
where we use Lemma   \ref{lem:damianetal} in \eqref{approx:eqq0} and the definition of $z_k$ and   $\norm{\bs{v} \vert_{\cJ}}_2^2 > 0$ in \eqref{approx:eqq1}.
Therefore, we have
\eq{
& \abs*{ \E_{(\bs{w},\azp,\bop)} \left[ \frac{1}{\ngt} \sum_{l = 1}^N  h(  \bs{w}, \azp,\bop, \bzl)  \phi \left(  \azp \eta  \shc  \big(\bzl \big)  \sJ \inner{\bs{v}}{\bs{x}_i} - \bop \right)  \right] -  \gt(\inner{\bs{v}}{\bs{x}_i}) } \\
& \leq \abs*{ \sum_{k = 0}^p \sum_{l = 1}^N  \frac{ \indic{\abs{\shc(\bzl)} \geq \tau}}{\ngt} \E_{\bs{w}} \left[ \indic{E^c}   z_k(\sJ)  \sJ^k \inner{\bs{v}}{\bs{x}_i}^k \right] } \\
& \labelrel\leq{approx:eqq2} \left( \max_{k \leq p} \abs{\inner{\bs{v}}{\bs{x}_i}}^{k} \right)   \mpr[E^c]^{1/2}   \sum_{k = 0}^p   \abs{c_k}  9^{k(k^{\star} - 1)}
}
where we use Lemma \ref{lem:errterm} for  \eqref{approx:eqq2}.  
By Lemma \ref{lem:damianetal2}, we have
\eq{
\sum_{k = 0}^p   \abs{c_k}  9^{k(k^{\star} - 1)} \labelrel\leq{approx:ineqq3}\sum_{k = 0}^p \frac{\sqrt{2}  9^{k k^{\star}}}{\sqrt{k!}} e^{\frac{p - k}{4}} 
& \leq C e^{\frac{9^{2k^{\star}}}{2}}  e^{\frac{p}{4}}.  \label{lemapprox:siarg1} 
}
where \eqref{approx:ineqq3} follows $9 \geq \sqrt{e}$.
By observing that $\abs{\shc(\bzl)} \leq k^{\star}$ and $E^c \Rightarrow    \max_{i \in [n]}  \abs*{\sJ \inner{\bs{v}}{\bs{x}_i}} > \tfrac{1}{\eta k^{\star}} ~OR~   \abs{\sJ} \geq \frac{1}{\eta \tau}$, we have  \eqref{lemapprox:sibias}.   For  \eqref{lemapprox:sihardbound}, by Lemma \ref{lem:damianetal}, we have
\eq{
\frac{\abs{v_k(\azp,\bop)}}{\eta^k \shc^k(\bzp)} \leq \frac{6 \sqrt{2} (k+1)^2}{\eta^k \tau^k}.  \label{lemapprox:vbound} 
}
Moreover,
\eq{
\abs{z_k(\sJ)} \indic{E}  \labelrel\leq{approx:ineqq4} \frac{\abs{c_k} }{\eta^k \tau^k} \frac{1}{\E_{\bs{w}} \left[ \sJ^{2k} \right]}   
\labelrel\leq{approx:ineqq5} \frac{\abs{c_k} }{\eta^k \tau^k}\frac{4^{k (k^{\star} - 1)} M^{k(k^{\star} - 1)} }{ \big( 2k (k^{\star} - 1) \big) !!}  
\labelrel\leq{approx:ineqq6}  \frac{  e^2 \abs{c_k} }{\eta^k \tau^k}  M^{k(k^{\star} - 1)}, 
}
where we use  $E \Rightarrow \abs{\sJ} \leq \tfrac{1}{\eta \tau}$ for \eqref{approx:ineqq4},    $\norm{\bs{v} \vert_{\cJ^c}}_2^2 \leq \tfrac{1}{4}$ and $M \geq 2p (k^{\star} - 1)$ for  \eqref{approx:ineqq5}, and    $\tfrac{4^{k (k^{\star} - 1)}}{\big( 2k (k^{\star} - 1) \big) !!} = \tfrac{2^{k(k^{\star}-1)}}{\big( k (k^{\star} - 1) \big) !} \leq e^2$ for   \eqref{approx:ineqq6}.
Therefore, 
\eq{
\abs{h( \bs{w}, \azp,\bop, \bzp) }   \leq  \sum_{k = 0}^p \frac{M^{k(k^{\star} - 1)}}{\eta^{2k} \tau^{2k}} 6  e^2 \sqrt{2} (k+1)^2 \abs{c_k}  
\labelrel\leq{approx:ineqq7}  \tilde C e^{\frac{p}{4}}  \max_{k \leq p}   \frac{M^{k(k^{\star} - 1)}}{\eta^{2k} \tau^{2k}},
}
where we used Lemma \ref{lem:damianetal2} for \eqref{approx:ineqq7}.
For MI, by adjusting the arguments  between \eqref{lemapprox:siarg0}-\eqref{lemapprox:siarg1} by using the  bounds for MI proven above, we can obtain  \eqref{lemapprox:mibias}. For \eqref{lemapprox:mihardbound}, we observe that 
\eq{
\abs{ z_k(\sJ)} \indic{E} \labelrel\leq{approx:ineqq8}   \norm{\sJ}_2^k \norm*{\E \left[ vec( \sJ^{\otimes k} )  vec(   \sJ^{\otimes k})^\top  \right]^{+}    vec(\bm{\tilde T_k})}_2 \indic{E} 
& \labelrel\leq{approx:ineqq9}  \frac{1}{\eta^k \tau^k} \frac{\E \left[ \norm{\bs{w} \vert_{\cJ}}_2^{2k} \right]}{k! \sigma_r^{2k}(\bs{V} \vert_{\cJ} \bs{D})} \norm{\bm{\tilde T_k}}_F \indic{E}   \\
& \labelrel\leq{approx:ineqq10}     \frac{e^4}{\eta^k \tau^k} \frac{ M^k }{\sigma_r^{2k}(\bs{D})} \norm{\bm{\tilde T_k}}_F \indic{E}  
}
where we used Cauchy Schwartz inequality for \eqref{approx:ineqq8},   Proposition \ref{prop:condz} and $E \Rightarrow  \norm{\sJ}_2 \leq \frac{1}{\eta\tau}$ for   \eqref{approx:ineqq9},  and  $E \Rightarrow  \norm{\bs{V} \vert_{\cJ^c}}_F \leq \frac{1}{2}$, $M \geq 2p$, and $\tfrac{4^k}{k!} \leq e^4$ for \eqref{approx:ineqq10}.
By \eqref{lemapprox:vbound} and Lemma \ref{lem:damianetal2},  we have
\eq{
\abs{h( \bs{w}, \azp,\bop, \bzp) } \leq  \sum_{k \leq p}  \frac{M^k 6\sqrt{2} e^4  (k+1)^2}{\eta^{2k} \tau^{2k}  \sigma_r^{2k}(\bs{D})}  \norm{\bm{\tilde T_k}}_F 
\leq \tilde C (e \sqrt{r})^{\frac{p}{4}} \max_{k \leq p}  \frac{M^k}{\eta^{2k} \tau^{2k}  \sigma_r^{2k}(\bs{D})}.  
}
\end{proof}

\subsection{Empirical Approximation}
For the following theorem, we introduce:
\eq{
\norm{X}_{\psi_2} \coloneqq \inf \left \lbrace t > 0 ~ \vert ~ \E_{(\bs{w}, \az,\bo)} \left[  \exp \left( \frac{X^2}{t^2} \right)  \right] \leq 2 \right \rbrace.
}

For the following, let us assume that  we have  i.i.d. $\{  (\bs{w}_j, \azj, \boj, \bzj)\}_{j \in [m]}$ and for $B, N \in \N$, let $m = B \cdot N$.   We will double index parameters as $\bs{w}_{jl} = \bs{w}_{(j-1)N + l}$,  $j \in [B]$ and $l \in [N]$.
Recall that
\eq{
h(  \bs{w}, \azp,\bop, \bzl ) \coloneqq    \sum_{k = 0}^p  \frac{v_{k}(\azp,\bop)}{\eta^k \shc^k(\bzl)} z_k(\sJ)   \indic{E}  \label{eq:defh2}
} 
We let
\eq{
Y_{jl} \coloneqq \begin{cases} 
h(  \bs{w}_{jl} , \azp_{jl},\bop_{jl}, \bzp_{jl})  \phi \left(  \azp_{jl}  \eta \shc \big( \bzp_{jl} \big) \inner{\bs{v}}{\frac{\bs{w}_{jl} \vert_{\cJ}}{\norm{\bs{w}_{jl} \vert_{\cJ}}_2}}^{k^{\star} - 1} \inner{\bs{v}}{\bs{x}_i}  + \bop_{jl} \right)  & \text{SI} \\[1.5ex]
h(  \bs{w}_{jl} , \azp_{jl},\bop_{jl}, \bzp_{jl})  \phi \left(  \azp_{jl}  \eta \shc \big(  \bzp_{jl} \big)  \inner{\bs{D} \bs{V}^\top  \frac{\bs{w}_{jl} \vert_{\cJ}}{\norm{\bs{w}_{jl} \vert_{\cJ}}_2} }{\bs{V}^\top \bs{x}_i} + \bop_{jl} \right)  & \text{MI} 
\end{cases}
}
Moreover let  $Y_j    \coloneqq \frac{1}{\ngt} \sum_{l = 1}^N Y_{jl}$  and $\ngt_j \coloneqq \sum_{l = 1}^N \indic{\abs{\shc(\bzp_{jl})} \geq \tau}$.
We have the following statement:
\begin{lemma}
\label{lem:empsg}
 We assume that: For SI,  $M \geq 2p(k^{\star} -1)$,  and $\ngt_j > N/3$.  For MI:  $M \geq 2p$, and $\ngt_j > N/3$.  Then,  there exists a universal constant $\tilde C > 0$ such that
\eq{
\norm{Y_j - \E_{(\bs{w},\az,\bo)}[Y_j]}_{\psi_2} \leq \tilde C  \begin{cases}
\frac{ e^{\frac{p}{4}}}{\sqrt{N}} \max_{k \leq p} \frac{ M^{k(k^{\star} - 1)} }{ \eta^{2k} \tau^{2k} } & \text{SI} \\
\frac{(e\sqrt{r})^{\frac{p}{4}}}{\sqrt{N}} \max_{k \leq p} \frac{ M^{k} }{ \eta^{2k} \tau^{2k} \sigma^{2k}_r(\bs{D}) }  & \text{MI}. 
\end{cases}
}
\end{lemma}

\begin{proof}
For both SI and MI,  there exists a universal $C > 0$ such that we have
\eq{
\norm*{Y_j - \E_{(\bs{w},\azp,\bop)}[Y_j]}_{\psi_2} ^2 = \norm*{\frac{1}{\ngt} \sum_{l = 1}^N Y_{jl}- \E_{(\bs{w},\az,\bo)}[Y_{jl}]}_{\psi_2}^2 \leq C \sum_{l = 1}^N   \norm*{ Y_{jl} }_{\psi_2}^2.  \label{empsg:eq0}
}
Since $\phi(t)^ 2 \leq t^2$, for SI, we have
\eq{
\norm{ Y_{jl} }_{\psi_2} & \leq \norm*{ h(  \bs{w}_{jl} , \azp_{jl},\bop_{jl}, \bzp_{jl})   \left(  \azp_{jl}  \eta \shc \big( \bzp_{jl} \big) \inner{\bs{v}}{\frac{\bs{w}_{jl} \vert_{\cJ}}{\norm{w_{jl} \vert_{\cJ}}_2}}^{k^{\star} - 1} \inner{\bs{v}}{\bs{x}_i}  +  \bop_{jl} \right) }_{\psi_2} \\
& \labelrel\leq{empsg:ineqq0}  \tilde C e^{\frac{p}{4}} \max_{k \leq p} \frac{ M^{k(k^{\star} - 1)} }{ \eta^{2k} \tau^{2k} },
}
where \eqref{empsg:ineqq0} follows  by the definition of $E$   and $\norm{\bop_{jl}}_{\psi_2} \leq 3$.
For MI, we have
\eq{
  \norm*{  Y_{jl}  }_{\psi_2} & \labelrel={empsg:eqq1}   \norm*{   h(  \bs{w}_{jl} , \azp_{jl},\bop_{jl}, \bzp_{jl})  \phi \left(  \az_{jl}   \eta \shc \big( \bzp_{jl} \big)   \inner{\bs{D} \bs{V}^\top  \frac{\bs{w}_{jl} \vert_{\cJ}}{\norm{\bs{w}_{jl} \vert_{\cJ}}_2} }{\bs{V}^\top \bs{x}_i} + \bop_{jl} \right)   }_{\psi_2}  \\
& \labelrel\leq{empsg:ineqq2}   \tilde C (e\sqrt{r})^{\frac{p}{4}} \max_{k \leq p} \frac{ M^{k} }{ \eta^{2k} \tau^{2k} \sigma^{2k}_r(\bs{D}) }, 
}
where \eqref{empsg:eqq1} follows from $\phi(t)^ 2 \leq t^2$,  \eqref{empsg:ineqq2} follows  by the definition of $E$   and $\norm{\bo}_{\psi_2} \leq 3$.
By \eqref{empsg:eq0}  and $\ngt_j > N/3$, the statement follows.
\end{proof}

\bigskip
Let $poly(\cdot)$ a polynomial respectively,   depending on  $(p, k^{\star},\ghc_{k^{\star}})$ for SI,  and $\left( p, r, \sigma_1(\bs{D})/\sigma_r(\bs{D}) \right)$ for MI,  which will be defined later (see \eqref{corempapprox:defpoly}).  We define the following event:
\eq{
\widetilde E \equiv 
\begin{cases}
\abs*{   \frac{1}{B} \sum_{j = 1}^B Y_j  - \gt (\inner{\bs{v}}{\bs{x}_i}) } \geq   \frac{ poly \left[  \log n,  \log d^u \right] \log^{\frac{1}{2}} \left(\frac{2n}{\delta}  \right)}{\sqrt{m} } + \frac{1}{n} & \text{SI} \\[1ex]
\abs*{   \frac{1}{B} \sum_{j = 1}^B Y_j  - \tgt (\bs{V}^\top \bs{x}_i) } \geq   \frac{ poly \left[  \log n,  \log d^u \right] \log^{\frac{1}{2}} \left(\frac{2n}{\delta}  \right)}{\sqrt{m} } + \frac{1}{n} & \text{MI}
\end{cases}
}
\begin{lemma}
\label{lem:empapprox}
There exists a constant $C > 0$ depending on $(k^{\star}, \ghc_{k^{\star}})$ for SI and $r$ for MI such that if we have

\bigskip \noindent 
\begin{minipage}{.5\textwidth}
For SI:
\begin{enumerate}[leftmargin = *]
\item $\max_{i \in [n]} \abs{\inner{\bs{v}}{\bs{x}_i}} \leq \sqrt{3} \sqrt{1 + \log(4nd^u)}$.
\item $\eta = \frac{1}{C} \frac{1}{ \tau   \sqrt{1 + \log(4n d^u)}} \left( \frac{M}{1 + \log(P)}  \right)^{\frac{k^{\star} - 1}{2}}$ \\
where $P = n^2  \left[ C \left(1  + \log \left( 4n d^u \right) \right)   \right]^{p} $.
\item  $M \geq 2p(k^{\star} - 1)  \vee 16 \log \left( P \right) $
\item  $ \norm{\bs{v} \vert_{\cJ^c}}_2^2 \leq 1/4 $
\item $  \ngt_j \geq N/3 ~ \text{for all}~ j \in [B]$
\end{enumerate}
\end{minipage}%
\begin{minipage}{0.5\textwidth}
For MI:
\begin{enumerate}[leftmargin = *]
\item $\max_{i \in [n]} \norm{\bs{V}^\top \bs{x}_i} \leq \sqrt{3} \sqrt{r + \log(4nd^u)}$.
\item $\eta = \frac{1}{C} \frac{1}{ \tau \sigma_1(D)   \sqrt{r + \log(4n d^u)}} \left( \frac{M}{r + \log(P)}  \right)^{\frac{1}{2}}$ \\
where $P = n^4  \left[ C \left(r  + \log \left( 4n d^u \right) \right)   \right]^{2p} $.
\item  $M \geq 2p  \vee 16 \log \left( P \right) $
\item  $ \norm{\bs{V} \vert_{\cJ^c}}_F^2 \leq 1/4 $
\item $  \ngt_j \geq N/3 ~ \text{for all}~ j \in [B]$
\end{enumerate}
\end{minipage}

\bigskip \noindent 
then, the following holds:

\smallskip \noindent 
\begin{minipage}{.5\textwidth}
\begin{itemize}[leftmargin = *]
\item $\max_{k \leq p} \tfrac{   e^{\frac{p}{4}} M^{k(k^{\star} - 1)} }{ \eta^{2k} \tau^{2k} } \\  \leq      C^{2p} e^{\frac{p}{4}}   (1  + \log(4nd^u) )^{p} (1 + \log(P) )^{p(k^{\star} - 1)}$
\item $\mpr_{(\bs{w},\azp,\bop)}[\widetilde E] \leq \delta$
\end{itemize}
\end{minipage}%
\begin{minipage}{0.5\textwidth}
\begin{itemize}[leftmargin = *]
\item $\max_{k \leq p} \tfrac{   (e\sqrt{r})^{\frac{p}{4}} M^{k} }{ \eta^{2k} \tau^{2k} \sigma^{2k}_r(\bs{D}) } \\  \leq  \!    C^{2p} (e\sqrt{r})^{\frac{p}{4}}  \left( \tfrac{\sigma_1(\bs{D})}{\sigma_r(\bs{D})} \right)^{2p}   \!\!\!     (r  + \log(4nd^u) )^{p} (1 + \log(P) )^{p}$
\item $\mpr_{(\bs{w},\azp,\bop)}[\widetilde E] \leq \delta$
\end{itemize}
\end{minipage}
\end{lemma}

\begin{proof}
For SI, we have
\eq{
\max_{k \leq p} \tfrac{   e^{\frac{p}{4}} M^{k(k^{\star} - 1)} }{ \eta^{2k} \tau^{2k} }  & =  e^{\frac{p}{4}}  \left(  \max_{k \leq p}     C^k (1 + \log(4n d^u ) )^{k} (1 + \log(P) )^{k (k^{\star} - 1)}     \right)  \\
&=  C^{2p} e^{\frac{p}{4}  }  (1  + \log(4nd^u) )^{p} (1 + \log(P) )^{p}   \label{corempapprox:sieq0}
}
For MI, we have
\eq{
 (e\sqrt{r})^{\frac{p}{4}} \left(  \max_{k \leq p} \frac{  M^{k} }{ \eta^{2k} \tau^{2k} \sigma^{2k}_r(\bs{D}) } \right)   
 & =    (e\sqrt{r})^{\frac{p}{4}}  \left( \frac{\sigma_1(\bs{D})}{\sigma_r(\bs{D})} \right)^{2p}  \left(  \max_{k \leq p}     C^{2k} (r  + \log(4n d^u) )^{k} (1 + \log(P) )^{k}     \right)   \\
&= C^{2p} (e\sqrt{r})^{\frac{p}{4}}  \left( \frac{\sigma_1(\bs{D})}{\sigma_r(\bs{D})} \right)^{2p}      (r  + \log(4nd^u) )^{p} (r+ \log(P) )^{p}   \label{corempapprox:mieq0}
}Let
\eq{
poly(\log n, \log d^u)  \geq 
\begin{cases}
C^p e^{\frac{p}{4}  }  (1  + \log(4nd^u) )^{p} (1 + \log(P) )^{p}   & \text{SI} \\
C^p (e\sqrt{r})^{\frac{p}{4}}  \left( \frac{\sigma_1(\bs{D})}{\sigma_r(\bs{D})} \right)^{2p}      (r  + \log(4nd^u) )^{p} (1 + \log(P) )^{p}  & \text{MI}.
\end{cases}
\label{corempapprox:defpoly}
}
By Lemma \ref{lem:empsg}, for  both SI and MI, we have
\eq{
\mpr_{(\bs{w}, \az,\bo)} \Bigg[ \abs*{ \frac{1}{B} \sum_{j = 1}^B Y_j  -  \E_{(\bs{w},\azp,\bop)} \left[ Y_j \right]} \geq    \underbrace{ poly(\log n, \log d^u)  \sqrt{\frac{\log(2/\delta)}{m}} }_{\coloneqq A_1}  \Bigg] \leq \delta. \label{empapprox:eq0}
}
By Lemma \ref{lem:approx}, we have
\eq{
& \text{SI}: ~ \abs*{ \E_{(\bs{w},\ao,\bo)} \left[ Y_j \right] - \gt (\inner{\bs{v}}{\bs{x}_i}) } \\
& \hspace{1em}    \leq \underbrace{ C_{k^{\star}} e^{\frac{p}{4}}   \left( \max_{k \leq p} \abs{\inner{\bs{v}}{\bs{x}_i}}_2^k  \right)  \mpr_{\bs{w}} \left[ \abs{\sJ} \geq \frac{1}{\eta \tau}  ~ \text{OR} ~   \max_{i \in [n]}  \abs*{\sJ \inner{\bs{v}}{\bs{x}_i}} > \frac{1}{ \eta k^{\star}} \right]^{\frac{1}{2}} }_{ \coloneqq A_2}  \\
& \text{MI}: ~ \abs*{ \E_{(\bs{w},\az,\bo)} \left[ Y_j \right] - \tgt (\bs{V}^\top \bs{x}_i) }  \\
&  \hspace{0.5em}  \leq\underbrace{   C_{k^{\star}}  (e \sqrt{r})^{\frac{p}{4}}      \left(    \frac{\sigma_1(\bs{V} \vert_{\cJ} \bs{D})}{\sigma_r(\bs{V} \vert_{\cJ} \bs{D})}  \right)^p     \left( \max_{k \leq p}    \norm{\bs{V}^\top \bs{x}_i}_2^k  \right)    \mpr_{\bs{w}}    \left[  \norm{\sJ}_2 \geq \frac{1}{\eta \tau}   \text{OR}    \max_{i \in [n]}  \abs*{\inner{\sJ}{\bs{V}^\top \bs{x}_i}} > \frac{1}{\eta}   \right]^{\frac{1}{4}}    }_{ \coloneqq A_2} 
}
Therefore, for both SI and MI, we have
\eq{
 \mpr_{(\bs{w}, \az,\bo)} \left[ \abs*{ \frac{1}{B} \sum_{j = 1}^B Y_j    - \tgt (\bs{V}^\top \bs{x}_i)} \geq   A_1 + A_2   \right] \leq \delta \label{empapprox:eq1}
}
For SI,  by Lemmas \ref{lem:highprobprobbound1} and \ref{lem:highprobprobbound2}, we have 
\eq{
\mpr \left[ \abs{\sJ} \geq \frac{1}{\eta \tau} \right] \labelrel\leq{coremp:ineqq0} \frac{2}{P} ~~ \text{and} ~~
\mpr \left[ \max_{i \in [n]} \abs{\inner{\bs{v}}{\bs{x}_i}}\abs{\sJ} \geq \frac{1}{\eta k^{\star}} \right]  \labelrel\leq{coremp:ineqq1}  \frac{2}{P},  \label{corempapprox:eq1}
}
where we choose $C \geq 1 \vee  \frac{k^{\star} \sqrt{3}}{\tau} 6^{\frac{k^{\star} -1}{2}}$ for \eqref{coremp:ineqq0} and  \eqref{coremp:ineqq1}.
Therefore,  by choosing $C \geq 3 \sqrt{e} (2C_{k^{\star}})^{2/p}$,  we have
\eq{
A_2 & \leq  2  C_{k^{\star}}   e^{\frac{p}{4}}  \left( \sqrt{3} \sqrt{  1+  \log( 4n d^u)}  \right)^p  \frac{1}{\sqrt{P}}  \leq \frac{1}{n} . \label{corempapprox:eq2}
}
For MI, the same argument with its corresponding bounds applies.
\end{proof}
\subsection{Concentration Bound for a Desirable Event}
\begin{corollary}
\label{cor:fl-existence}
We fix $u \in \N$.  For any $\varepsilon > 0$,  if
\eq{
m = \Theta(d^\varepsilon), ~~ d \geq O(M)  ~~ \text{and} ~~ c = \frac{1}{\log d},  \label{fl-singleindex:hypers1}
}
$n$ and $M$ are chosen as in Lemmas \ref{lem:pruningalgmainsingleindex} and \ref{lem:pruningalgmainmultindex} for SI and MI respectively, and
\eq{
\eta = \frac{1}{\tau C} \begin{cases}
\frac{1}{\sqrt{1+\log(4nd^u)}} \left(  \frac{M}{1 + \log(P)}  \right)^{\frac{k^{\star} -1}{2}}  & \text{SI} \\
\frac{1/\sigma_1(H)}{\sqrt{r+\log(4nd^u)}} \left(  \frac{M}{r + \log(P)}  \right)  & \text{MI} 
\end{cases}
~~ \text{where} ~ ~
P = \begin{cases}
n^2  \left[ C \left(1  + \log \left( 4n d^u \right) \right)   \right]^{p},  & \text{SI} \\
n^4  \left[ C \left(r  + \log \left( 4n d^u \right) \right)   \right]^{2p},  & \text{MI} 
\end{cases} \\ \label{eq:etaval}
}
and $C$ is the constant appeared in Lemma \ref{lem:empapprox},  we have with probability at least $1 - (16 + 6 m) d^{-u}$, the intersection of the
\begin{enumerate}[label=C.\arabic*]
\item \label{cond:boundW}$\max_{j \in [2m]}  \norm{ \bs{W}^{(1)}_{j*} }_2 \leq  \tilde O(1)$
\item   \label{cond:boundmu} $\norm{ \hmu \vert_{\cJ} }_2 \leq  1 +  O\left(\frac{1}{\sqrt{M}}\right)$
\item  \label{cond:boundb} $\norm{\bo}_2^2 \leq 4m$  and  $\norm{\bo}_4^4 \leq 6m$ and $\norm{\bo}_{\infty} \leq  \tilde O(1)$
\item \label{cond:existsa} There exists $\bm{\hat a} \in \R^{2m}$ such that 
\eq{
\norm{\bm{\hat a}}_2^2 \leq \begin{cases}  
O \left( \frac{ (1  + \log(4nd^u) )^{2p} (1 + \log(P) )^{2p(k^{\star} - 1)}}{m} \right)  & \text{SI} \\
O \left(  \frac{ \left(r + \log(4n d^u) \right)^{2p}  \left( r + \log(P) \right)^{2p}}{m} \right)  & \text{MI}, 
\end{cases}
}
and
\item   \label{cond:empiricalrisk}    $\frac{1}{n} \sum_{i =1 }^n \left(  y_i  - \hat y(\bs{x}_i; (\bm{\hat a}, \bs{W}^{(1)}, \bo)) \right)^2 \leq   \Delta \E [ \epsilon^2 ] +  \tilde O \left(  \frac{1}{m} +  \frac{1}{\sqrt{n}}+ \frac{1}{M} \right) \\[1.5ex] \hspace*{18em} + \begin{cases}
 O \left( \tfrac{1  + \log(4nd^u) )^{2p} (1 + \log(P) )^{2p(k^{\star} - 1)}}{\rho_1 \log^{\rho_2} d}  \right) & \text{SI} \\
 O \left( \tfrac{r  + \log(4nd^u) )^{2p} (r + \log(P) )^{2p}}{\rho_1 \log^{\rho_2} d}  \right)   & \text{MI} 
\end{cases} $
\end{enumerate}
where $O$ suppresses constants, and $\tilde O$ suppresses constants and $\text{Poly} \left[  \log n,  \log d  \right]$  depending  on the problem parameters \footnote{Specifically, $(k^{\star}, \ghc_{k^{\star}}, u,  p, \varepsilon, \alpha, C_1, C_2, \cgp,  \Delta)$ for SI,  $(\sigma_1(H), \sigma_r(H), u,  p, \varepsilon, \alpha, C_1, C_2, \cgp, r, \Delta)$ for MI.}.
\end{corollary}
Let $\tau >0$ be the values defined in Lemma \ref{lem:taudef},  $N = \floor{  \sqrt{m} }$, and let
\eq{
\bm{\hat a}_j \coloneqq \frac{ h(\bs{W}_{j*}, \azj,\bzj,\boj) }{B \ngt_j}.
}
Moreover let
\eq{
& \tilde y_i \coloneqq  \begin{cases}
\sum_{j = 1}^{2m} \bm{\hat a}_j \phi \left( \azj  \inner{\bs{v}}{ \bs{W}^{(0)}_{j*}}^{k^{\star} - 1}  \inner{\bs{v}}{\bs{x}_i} -  \boj  \right)   & \text{SI} \\
 \inner{\hmu \vert_{\cJ}}{\bs{x}_i} + \sum_{j = 1}^{2m} \bm{\hat a}_j \phi \left( \azj  \eta \shc(\bzj) \inner{H \bs{W}^{(0)}_{j*}}{  \bs{x}_i} -  \boj  \right)  & \text{MI} 
\end{cases} \\[1ex]
& \hat y_i \coloneqq \begin{cases}
\inner{\bm{\hat a}}{\phi ( \bs{W}^{(1)} \bs{x}_i + \bo  )}  & \text{SI} \\
\inner{\hmu \vert_{\cJ}}{\bs{x}_i} +    \inner{\bm{\hat a}}{\phi ( \bs{W}^{(1)} \bs{x}_i + \bo  )} & \text{MI}. 
\end{cases}
}
We consider the intersection of the following events:
\begin{enumerate}[label=E.\arabic*]
\item \label{event:ngt} $\ngt_j \geq N/3 ~ \text{for all}~ j \in [B]$
\item \label{event:gradbound} For SI Proposition \ref{prop:emprgradientconcsingle},  for MI Proposition \ref{prop:emprgradientconc}  holds for all $j \in [2m]$ with $\delta = d^{-u}$
\item   \label{event:pruningres}  For SI:  $\norm{\bs{v} \vert_{\cJ^c}}_2^2 \leq O \left(   \frac{1}{\rho_1 \log^{\rho_2} d} \right)$.  For MI: $\norm{ \E[y \bs{x}] \vert_{\cJ^c}}_2^2 \vee \norm{\bs{V} \vert_{\cJ^c}}_2^2 \leq   O \left( \tfrac{1}{\rho_1 \log^{\rho_2} d} \right)$.
\item  \label{event:maxinner}  We have   
\eq{ \max_{i \in [n]} \abs{\inner{\bs{v}}{\bs{x}_i}} \leq \sqrt{3} \sqrt{1 + \log(4n d^u)} ~ \text{and} ~  \max_{i \in [n]} \norm{\bs{V}^\top \bs{x}_i}_2 \leq     \sqrt{3} \sqrt{r+ \log(4n d^u)}, 
}
for SI and MI respectively.
\item \label{event:secondlayerbound} $\norm{\bm{\hat a}}_2^2 \leq \begin{cases}  
O \left( \tfrac{ (1  + \log(4nd^u) )^{2p} (1 + \log(P) )^{2p(k^{\star} - 1)}}{m} \right)  & \text{SI} \\[1ex]
O \left(  \tfrac{ \left(r + \log(4n d^u) \right)^{2p}  \left( r + \log(P) \right)^{2p}}{m} \right)  & \text{MI}, 
\end{cases} $
\item \label{event:pruningbias} $\tfrac{1}{n} \sum_{i = 1}^n \left( \tilde y_i - \hat y_i \right)^2 \leq  \tilde O \left( \tfrac{1}{M}   \right) +   \begin{cases}  
O \left( \tfrac{ (1  + \log(4nd^u) )^{2p} (1 + \log(P) )^{2p(k^{\star} - 1)} }{\rho_1 \log^{\rho_2} d} \right)   & \text{SI} \\[1ex]
O \left(  \tfrac{\left(r + \log(4n d^u) \right)^{2p}  \left( r + \log(P) \right)^{2p-1}}{\rho_1 \log^{\rho_2} d  } \right) & \text{MI}, 
\end{cases} $
\item   \label{event:errorfirsthermite} For MI:  $\frac{1}{n} \sum_{i = 1}^n (\inner{\E[y \bs{x}]}{\bs{x}_i} - \inner{\hmu \vert_{\cJ}}{\bs{x}_i} )^2 \leq  O \left( \frac{1}{\rho_1 \log^{\rho_2} d} + \frac{1}{M} \right)$
\end{enumerate}

\begin{lemma}
\label{lem:desired}
With the choice of parameters in Corollary \ref{cor:fl-existence},  the intersection  of \eqref{event:ngt}-\eqref{event:errorfirsthermite} holds with probability  at least $1 - (11 + 4m)d^{-u}$.
\end{lemma}

\begin{proof}
Since $N = \floor{ \sqrt{m} }$,  by using Lemma \ref{lem:taudef} and union bound, we can show that  \eqref{event:ngt} holds with probability at least $1 - \Theta(d^{\varepsilon/2} ) \exp \left( -  \Theta(d^{\varepsilon/2} )  \right) \geq 1 - d^{u}$ for large enough $d$ depending on $(u,  \varepsilon)$.  Since with a sufficiently large constant factor,  $M$ satisfies the condition in Propositions \ref{prop:emprgradientconcsingle} and \ref{prop:emprgradientconc}, we have  \eqref{event:gradbound} holds with probability at least $1 - 2m d^{-u}$.   By Lemmas \ref{lem:pruningalgmainsingleindex}, \ref{lem:pruningalgmainmultindex}  and the choice of parameters,   we can show that \eqref{event:pruningres} holds  with probability at least $1 - 4 d^{-u}$.   By Corollary \ref{cor:laurentmassart2} we have that  \eqref{event:maxinner} holds with probability at least $1 - d^{-u}$.  

\noindent
For \eqref{event:secondlayerbound},  by Lemmas \ref{lem:approx} and \ref{lem:empapprox},   we have
\eq{
\abs{ \bm{\hat a}_j } \leq  \begin{cases}
O \left( \frac{N}{\ngt_j} \frac{\tilde C}{m}    \max_{k \leq p} \frac{  M^{k(k^{\star}-1)} }{ \eta^{2k} \tau^{2k} }   \right), & \text{SI} \\
O \left( \frac{N}{\ngt_j} \frac{\tilde C}{m}    \max_{k \leq p} \frac{  M^{k} }{ \eta^{2k} \tau^{2k}  \sigma_r^{2k}(H)}   \right),  \hspace{-2em} & \text{MI}
\end{cases} 
\leq
 \begin{cases}
 O \left( \frac{ (1  + \log(4nd^u) )^{p} (1 + \log(P) )^{p(k^{\star} - 1)} }{m}  \right), & \text{SI} \\
O \left( \frac{ (r  + \log(4nd^u) )^{p} (r + \log(P) )^{p} }{m}  \right) \hspace{-2em} & \text{MI}
\end{cases}
}
Hence,  \eqref{event:secondlayerbound} follows.   For the following, we additionally consider the intersection of the following events:
\begin{enumerate}[label=$\widetilde{\text{E}}$.\arabic*, leftmargin=2em]
\item  \label{event:yth} Lemma \ref{lem:supbound} holds for $\phi(t) = t$ with $\delta = d^{- u}$.
\item \label{event:sth} Lemma \ref{lem:supopnorm} holds for $\phi(t) = t$ with $\delta = d^{- u}$.
\item  \label{event:hpp}  Lemma \ref{lem:highprobprobbound2} holds for all  $\bs{W}_{j^*}^{(0)}$,  $j \in [2m]$, with  $\delta = d^{- u}$.
\item \label{event:reluadd1} For SI,  Lemma \ref{lem:reluadditivetermsaux} holds for $\mathcal{A} = \left \lbrace \frac{\bs{v} \vert_{\cJ^c}}{\norm{\bs{v} \vert_{\cJ^c}}_2}, ~ \abs{\cJ} \leq M \right \rbrace$ with  $\delta = d^{- u}$.  
\item  \label{event:reluadd2} For MI,   Lemma \ref{lem:reluadditivetermsaux} holds for  $\mathcal{A} = \left \lbrace \tfrac{ \E[y \bs{x}] \vert_{\cJ^c}}{\norm*{  \E[y \bs{x}] \vert_{\cJ^c} }_2}, ~ \abs{\cJ} \leq M \right \rbrace$ and  conditioned on $\bs{W}$ (see \eqref{eq:initdist}),  holds for
$\mathcal{A} = \left \lbrace  \tfrac{ \bs{H} \vert_{\cJ^c \times \cJ} \Wz_{j*}}{\norm*{ \bs{H}  \vert_{\cJ^c \times \cJ}  \Wz_{j*}}_2}, ~  \abs{\cJ} \leq M \right \rbrace$
each with  $\delta = d^{- u}$.
\end{enumerate}
Note that the intersection of the given events holds with probability at least $1 - 5 d^{-u} - 2m d^{-u}$.  For \eqref{event:pruningbias}, we observe that  $\bs{W}_{j*}^{(1)} = \eta \azj g(\Wz_{j*} ,  \bzj) \vert_{\cJ}$ , where $g$ is defined in \eqref{empgrad:epmgraddef}.  By Cauchy-Schwartz and triangle inequalities, we have
\eq{
& \frac{1}{n} \sum_{i = 1}^n \left( \tilde y_i - \hat y_i \right)^2 \\
& \leq    2 \eta^2  \norm{\bm{\hat a}}_2^2
\begin{cases}
 \sum_{j = 1}^{2m}   \frac{1}{n}  \sum_{i = 1}^n   \left(     \inner{ g(\Wz_{j*} ,  \bzj) \vert_{\cJ} - \shc(\bzj) \inner{\bs{v} }{ \bs{W}^{(0)}_{j*} }^{k^{\star} - 1}   \bs{v} \vert_{\cJ}}{\bs{x}_i}        \right)^2       & \text{SI}  \\ 
\qquad  + \sum_{j = 1}^{2m}    \frac{1}{n}  \sum_{i = 1}^n   \left( \shc(\bzj)  \inner{\bs{v} }{ \bs{W}^{(0)}_{j*} }^{k^{\star} - 1}     \inner{\bs{v} \vert_{\cJ^c}}{\bs{x}_i}   \right)^2  & \\[3ex]
\sum_{j = 1}^{2m}    \frac{1}{n}  \sum_{i = 1}^n   \left(     \inner{ g(\Wz_{j*} ,  \bzj) \vert_{\cJ} - \shc(\bzj)  \bs{H} \cJJ \bs{W}^{(0)}_{j*}}{\bs{x}_i}        \right)^2  & \text{MI}     \\
\qquad +  \sum_{j = 1}^{2m}    \frac{1}{n}  \sum_{i = 1}^n   \left( \shc(\bzj)  \inner{\bs{H} \cJcJ \bs{W}^{(0)}_{j*} }{\bs{x}_i}   \right)^2    &
\end{cases} ~~~~~~~ \label{eq:eq99999}
}
Hence, 
\eq{
\eqref{eq:eq99999} & \labelrel\leq{des:ineqq0} 4m \eta^2  \norm{\bm{\hat a}}_2^2
\begin{cases}
 O \left(   \frac{M \log^2 \left( \frac{24dn}{M} \right) \log^{2C_2} \left( 12nd^u \right)}{n}   + \left( \frac{1+\log(4d^u)}{M} \right)^{k^{\star}} \right)      & \text{SI}  \\ 
\qquad  +   O \left( \frac{ (1  + \log(4nd^u) )^{2p} (1 + \log(P) )^{2p(k^{\star} - 1)} }{\rho_1 \log^{\rho_2} d} \right)     & \\[3ex]
 O  \left(  \frac{M \log^2 \left( \frac{35dn}{M} \right) \log^{2C_2} \left( 18nd^u \right)}{n}   + \left( \frac{r +\log(4d^u)}{M} \right)^2 \right) & \text{MI}     \\
\qquad +  O \left( \frac{ (r + \log(4nd^u) )^{2p} (r + \log(P))^{2p}}{ \rho_1 \log^{\rho_2} d} \right)    &
\end{cases} \\[2ex]
& \labelrel\leq{des:ineqq1}  \tilde O \left( \tfrac{1}{M}   \right) +   \begin{cases}  
O \left( \frac{ (1  + \log(4nd^u) )^{2p} (1 + \log(P) )^{2p(k^{\star} - 1)} }{\rho_1 \log^{\rho_2} d} \right) & \text{SI} \\[1ex]
O \left(  \tfrac{\left(r + \log(4n d^u) \right)^{2p}  \left( r + \log(P) \right)^{2p-1}}{\rho_1 \log^{\rho_2} d  } \right) & \text{MI}, 
\end{cases}
}
where we use \eqref{event:gradbound}, and  \eqref{event:sth}-\eqref{event:reluadd2}   for \eqref{des:ineqq0} and  \eqref{event:secondlayerbound} and \eqref{eq:etaval} for  \eqref{des:ineqq1}.  Lastly,
\eq{
\frac{1}{n} \sum_{i = 1}^n (\inner{\E[y \bs{x}]}{\bs{x}_i} - \inner{\hmu \vert_{\cJ}}{\bs{x}_i} )^2  
&\leq 2 \norm*{\frac{1}{n} \sum_{i = 1}^n \bs{x}_i \bs{x}_i^\top \cJJ}_2    \norm{ ( \hmu  -\E[y \bs{x}] )\vert_{\cJ}}_2^2 +   \frac{2}{n} \sum_{i = 1}^n \inner{\E[y \bs{x}] \vert_{\cJ^c}}{\bs{x}_i}^2 \\
& \labelrel\leq{des:ineqq2}  O \left(  \frac{M \log^2 \left( \tfrac{24dn}{M} \right) \log^{2C^2}(6nd^u)}{n} +  \norm{\E[y \bs{x}]\vert_{\cJ^c}}_2^2 \right), 
}
where we used  \eqref{event:yth}- \eqref{event:sth} for   \eqref{des:ineqq2}.  By    \eqref{event:pruningres},     \eqref{event:errorfirsthermite} follows.
\end{proof}

\begin{proof}[Proof of Corollary \ref{cor:fl-existence}]
We assume   the intersection  of \eqref{event:ngt}-\eqref{event:errorfirsthermite} and   \eqref{event:yth}- \eqref{event:reluadd2} holds.  By recalling that $\bs{W}_{j*}^{(1)} = \azj \eta g \left( \Wz_{j*} ,  \bzj \right),$  we have
\eq{
\norm{ \bs{W}_{j*}^{(1)} }_2    =   \eta \norm*{ g \left( \Wz_{j*} ,  \bzj \right) }_2  
&   \labelrel={fl:eqq0}    \eta  \begin{cases}
O \left( \left( \frac{1 + \log(4d^u)}{M}  \right)^{\frac{k^{\star} - 1}{2}} +  \sqrt{ \frac{M \log^2 \left(\frac{24dn}{M}  \right) \log^{2C_2} \left( 12nd^u  \right)   }{n}} \right) & \text{SI} \\[3ex]
O \left( \left( \frac{r  + \log(4d^u)}{M}  \right)^{\frac{1}{2}} +  \sqrt{ \frac{M \log^2 \left(\frac{35dn}{M}  \right) \log^{2C_2} \left( 18nd^u  \right)   }{n}}   \right)       & \text{MI}
\end{cases}  \\
& \leq \tilde O(1),
}
where we  use  \eqref{event:gradbound} in  \eqref{fl:eqq0}.   

For  \eqref{cond:boundmu}, for SI $\hmu = 0$,  therefore,  the statement is trivial in this case.  For MI,  by \eqref{event:yth}, we can write
\eq{
\norm{\hmu \vert_{\cJ}}  \leq \norm{(\hmu - \E[y \bs{x}] )\vert_{\cJ}}_2 + \norm{\E[y \bs{x}] \vert_{\cJ}}_2  
& \labelrel\leq{fl:ineqq1} 1 +   O \left( \sqrt{ \frac{M \log^2 \left( \frac{24dn}{M}  \right) \log^{2C_2} \left( 6nd^u \right)}{n} } \right)
}
where \eqref{fl:ineqq1}  follows since  $\norm{\E[y \bs{x}]}_2 \leq 1$.

For  \eqref{cond:boundb}, by using Lemma \ref{lem:laurentmassart},  we have with probability $1 - d^{-u}$, for $d$ is large enough
\eq{
\norm{\bo}_2^2 \leq 2m + 2 \sqrt{2 m \log d^u}  + 2 \log d^u  \leq 3m.  
}
Moreover,  by Lemma \ref{lem:hypercontactivity}, we observe that $\E \left[  \left(\frac{1}{2m} \sum_{j = 1}^{2m} b_j^4 - 3 \right)^p \right]^{1/p} \leq \tfrac{p^2 \E[b_1^8]}{\sqrt{m}}$. Therefore, with probability  $1 - d^{-u}$, for $d$ is large enough
\eq{
\frac{1}{2m} \sum_{j = 1}^{2m} b_j^4 - 3  \leq \frac{e \log^2 d^u \E [b_1^8]}{\sqrt{m}}  \Rightarrow \norm{\bo}_4^4 \leq 7m   
}
Moreover, by using standard Gaussian concentration with union bound, we have with probability $1 - 2m d^{-u}$,  $\norm{\bo}_{\infty} \leq \sqrt{\log(d^u)}$.  \eqref{cond:existsa} directly follows from  \eqref{event:secondlayerbound}.  

\noindent
For   \eqref{cond:empiricalrisk}  in SI, we have
\eq{
 \frac{1}{n} \sum_{i =1 }^n ( y_i  - \hat y(\bs{x}_i; (\bm{\hat a}, \bs{W}^{(1)}, \bo)))^2&  \leq   \frac{1}{n} \sum_{i =1 }^n ( \gt(\inner{\bs{v}}{\bs{x}_i})  - \hat y_i)^2  +   \frac{\sqrt{\Delta}}{n} \sum_{i =1 }^n ( \gt(\inner{\bs{v}}{\bs{x}_i})  - \hat  y_i) \epsilon_i \\
 & +     \frac{ \Delta }{n} \sum_{i =1 }^n \epsilon_i^2
}
 By using $\delta = d^{-u}$ in Lemma \ref{lem:empapprox} and  \eqref{event:pruningbias}, we have with probability at least $1 - d^{-u}$
\eq{
\frac{1}{n} \sum_{i =1 }^n ( \gt(\inner{\bs{v}}{\bs{x}_i})  - \hat y_i)^2 &  \leq    \frac{2}{n} \sum_{i =1 }^n ( \gt(\inner{\bs{v}}{\bs{x}_i})  - \tilde y_i)^2  +    \frac{2}{n} \sum_{i =1 }^n (  \tilde y_i -  \hat y_i)^2   \\
& \leq  \tilde O \left(  \frac{1}{m} + \frac{1}{n} + \frac{1}{M} \right) +    O \left( \tfrac{ (1  + \log(4nd^u) )^{2p} (1 + \log(P) )^{2p(k^{\star} - 1)} }{\rho_1 \log^{\rho_2} d} \right).
}
Since $\epsilon_i$ has $1$-Subgaussian norm, we have with probability at least $1 - 2 d^{-u}$,
\eq{
& \frac{\sqrt{\Delta}}{n} \sum_{i =1 }^n ( \gt(\inner{\bs{v}}{\bs{x}_i})  - \tilde y_i) \epsilon_i \leq {\sqrt{ \frac{\Delta \log(2d^u)}n}} \left( \frac{1}{n} \sum_{i =1 }^n ( \gt(\inner{\bs{v}}{\bs{x}_i})  - \hat y_i)^2 \right)^{1/2}  \\
& \frac{1}{n} \sum_{i =1 }^n  \epsilon_i^2 - \E \epsilon_i^2 \leq \tilde O \left( \frac{1}{\sqrt{n}} \right).  \label{eq:noisterms}
}
Therefore,    \eqref{cond:empiricalrisk}  follows for SI.  For MI,
\eq{
 \frac{1}{n} \sum_{i =1 }^n ( y_i  - \hat y(\bs{x}_i; (\bm{\hat a}, \bs{W}^{(1)}, \bo)))^2&  \leq   \frac{1}{n} \sum_{i =1 }^n ( \gt( \bs{V}^\top \bs{x}_i)  - \hat y_i)^2  +   \frac{\sqrt{\Delta}}{n} \sum_{i =1 }^n ( \gt(\bs{V}^\top \bs{x}_i)  - \hat  y_i) \epsilon_i +     \frac{ \Delta }{n} \sum_{i =1 }^n \epsilon_i^2
}
We observe that
\eq{
 (\gt(\bs{V}^\top \bs{x}_i) - \hat y_i)^2 &  \leq 2   (\gt(\bs{V}^\top \bs{x}_i) - \tilde y_i)^2 + 2 (\tilde y_i - \hat y_i)^2  \\
& \leq 4  \left(\tgt(\bs{V}^\top \bs{x}_i) -  \sum_{j = 1}^{2m} \bm{\hat a}_j \phi \left( \azj  \eta \shc(\bzj) \inner{H \bs{W}^{(0)}_{j*}}{  \bs{x}_i} -  \boj  \right) \right)^2 \\
& \quad + 4   \left(\inner{\E[y \bs{x}]}{\bs{x}_i} - \inner{\hmu \vert_{\cJ}}{\bs{x}_i} \right)^2 +  2 (\tilde y_i - \hat y_i)^2
}
Therefore, by using $\delta = d^{-u}$ in Lemma \ref{lem:empapprox} and by \eqref{event:pruningbias} and \eqref{event:errorfirsthermite},  we have with probability $1 - d^{-u}$
\eq{
\frac{1}{n}\sum_{i = 1}^n (\gt(\bs{V}^\top \bs{x}_i) - \hat y_i)^2 \leq O \left(  \frac{\left(r + \log(4n d^u) \right)^{2p}  \left( r + \log(P) \right)^{2p}}{\rho_1 \log^{\rho_2} d  } \right) + \tilde O \left( \frac{1}{m} + \frac{1}{M}  + \frac{1}{n} \right).
}
By the same argument in  \eqref{eq:noisterms},   \eqref{cond:empiricalrisk}  holds for MI as well.
\end{proof}

\subsection{Main Result}
\begin{theorem}[Restatement of Theorems \ref{thm:single} and \ref{thm:multi}]
\label{thm:populationrisk}
Under the parameter choice given in Corollary \ref{cor:fl-existence},
 for $\lambda_t = \frac{  m }{  \rho_1 \log^{\rho_2} d}$,  $\eta_t = \frac{1}{  \tilde O(m) + \lambda}$ and $T = \tilde O \left(\rho_1 \log^{\rho_2} d \right),$  Algorithm \ref{alg:onestepgd}  guarantees that with probability at least $1 - (18 + 6m) d^{-u}$, we have
\eq{
\E_{(\bs{x}, y)} \left[ \left( y - \hat y( \bs{x}; (\bs{a}^{(T)}, \bs{W}^{(1)}, b^{(1)}) ) \right)^2 \right] & \leq      \Delta \E[\epsilon^2] +  \tilde O \left(  \frac{1}{m}  + \frac{1}{M}   +  \sqrt{ \frac{M \log \left( \frac{35d}{M} \right)}{n}}  \right)   \\ 
& \quad +
\begin{cases}
 O \left( \tfrac{1  + \log(4nd^u) )^{2p} (1 + \log(P) )^{2p(k^{\star} - 1)}}{\rho_1 \log^{\rho_2} d}  \right) & \text{SI} \\
 O \left( \tfrac{r  + \log(4nd^u) )^{2p} (r + \log(P) )^{2p}}{\rho_1 \log^{\rho_2} d}  \right)   & \text{MI} 
\end{cases}
}
where $O$ suppresses constants, and $\tilde O$ suppresses constants and $\text{Poly} \left[  \log n,  \log d  \right]$  depending  on the problem parameters.
\end{theorem}

\begin{proof}
In the following, we assume that  \eqref{cond:boundW}-\eqref{cond:empiricalrisk} in  Corollary \ref{cor:fl-existence} hold.   We will prove the statement for SI and will sketch the proof for MI, since the arguments are the same except a few minor steps.  Recall that $R_n( (\bs{a},\bs{W},\bs{b}) ) = \frac{1}{2n} \sum_{i = 1}^n \left(y_i -  \inner{\bs{a}}{\phi(\bs{W} \bs{x}_i + \bs{b}) } \right)^2$.  
We consider
\eq{
\bs{a}^{*} \coloneqq \min_{\bs{a} \in \R^{2m} }  R_n( (\bs{a}, \Wo,  \bo) ) + \lambda \frac{\norm{\bs{a}}_2^2}{2} ~~ \text{where} ~~ \lambda = \frac{m}{  \rho_1 \log^{\rho_2} d}.\label{poprisk:minimizer}
} 
We observe that
\eq{
& \frac{\lambda \norm{\bs{a}^*}_2^2}{2} \leq  R_n( (\bm{\hat a}, \Wo, \bo) ) + \lambda \frac{\norm{\bm{\hat a}}_2^2}{2}  \Rightarrow  \\
&  \norm{\bs{a}^*}_2^2 \leq \frac{2}{\lambda}   R_n( (\bm{\hat a}, \Wo, \bo) )  +  \norm{\bm{\hat a}}_2^2 \leq    O \left(  \frac{ \left(1 + \log(4n d^u) \right)^{2p}  \left( 1 + \log(P) \right)^{2p(k^{\star} - 1)}}{m}  \right),  ~~~~~~~  \label{poprisk:normbound} 
}
and
\begin{align}
&  R_n(  (\bs{a}^*, \Wo \bo) ) \leq  R_n( (\bm{\hat a}, \Wo, \bo) ) + \lambda \frac{\norm{\bm{\hat a}}_2^2}{2}   \Rightarrow   \\
&    R_n(  (\bs{a}^*, \Wo, \bo) ) \leq   \Delta \E [ \epsilon^2 ] +   O \left(  \tfrac{ \left(1 + \log(4n d^u) \right)^{2p}  \left( 1 + \log(P) \right)^{2p(k^{\star} - 1)}}{\rho_1 \log^{\rho_2} d} \right)+  \tilde O \left(  \frac{1}{m} +  \frac{1}{\sqrt{n}}+ \frac{1}{M} \right) ~~~~ \label{poprisk:empriskbound}
\end{align}
Moreover, we observe that
\eq{
& \grad_a^2  R_n( (\bs{a},  \Wo, \bo) ) = \lambda \ide{2m} + \frac{1}{n} \sum_{i = 1}^n \phi(\Wo \bs{x}_i + \bo)  \phi(\Wo \bs{x}_i + \bo)^\top\\
& \Rightarrow  \norm{\grad_a^2  R_n( (\bs{a}, \Wo, \bo) )}_2 \leq  \lambda + \frac{1}{n}  \sum_{i = 1}^n \norm*{  \phi(\Wo \bs{x}_i + \bo) }_2^2
}
We have 
\eq{
  \frac{1}{n}  \sum_{i = 1}^n \norm*{  \phi(\Wo \bs{x}_i + \bo) }_2^2  & \leq \frac{1}{n}  \sum_{i = 1}^n \norm*{ \bs{W}^{(1)} \bs{x}_i + \bo  }_2^2  \\
& \leq  2   \sum_{j = 1}^{2m}  \norm{ \bs{W}^{(1)}_{j*} }_2^2 \norm*{ \frac{1}{n} \sum_{i = 1}^n \bs{x}_i \bs{x}_i^\top \cJJ}_2 +   2   \sum_{j = 1}^{2m}  (\boj)^2   
\labelrel\leq{main:ineqq0}  \tilde O(m).   \label{poprisk:smoothnesconstant}
}
where we use   \eqref{cond:boundW} and  \eqref{cond:boundb} for  \eqref{main:ineqq0} .

Therefore,   \eqref{poprisk:minimizer} is a $\lambda$-strongly convex and $ \big(  \tilde O(m) + \lambda \big)$- smooth problem.  By using $\eta_t = \frac{1}{ \tilde O(m) + \lambda }$, we can approximate to $a^*$ by $\frac{1}{nm}$ in $T =  \tilde O (  \rho_1 \log^{\rho_2} d) \log(nm) = \tilde O (  \rho_1 \log^{\rho_2} d)  $ iteration of gradient descent, i.e.,  $\norm{a^{(T)}- a^*}_2^2 \leq \frac{1}{nm}$ \cite[Theorem 3.10]{Bubeck2014ConvexOA}.   We have
\eq{
& \E_{(\bs{x}, y)} \left[ \left( y - \hat y( \bs{x}; (\bs{a}^{(T)}, \bs{W}^{(1)}, \bo) ) \right)^2 \right]     \\
& \leq   \E_{(\bs{x}, y)}    \left[ \left(   y  - \hat y( \bs{x}; (\bs{a}^{*}, \Wo, \bo) ) \right)^2 \right]  \\
&  + 2  \E       \left[ \left(   y    -   \hat y( \bs{x}; (\bs{a}^{*}, \Wo, \bo) ) \right)^2 \right]^{\frac{1}{2}}          \E_{\bs{x}}      \left[ \left(    \hat y( \bs{x}; (\bs{a}^{*}, \Wo, \bo) )   -   \hat y ( \bs{x}; (\bs{a}^{(T)}, \Wo, \bo) )    \right)^2 \right]^{\frac{1}{2}}  \\
&  +   \E_{\bs{x}} \left[ \left(  \hat y( \bs{x}; (\bs{a}^{*}, \Wo, \bo) ) - \hat y ( \bs{x}; (\bs{a}^{(T)}, \Wo, \bo) ) \right)^2 \right]. \label{poprisk:eqinit}
}
For the last term, 
\eq{
 \E_{\bs{x}} \Big[ \Big(  \hat y( \bs{x}; (\bs{a}^{*}, \Wo, \bo) )  - \hat y ( \bs{x}; (\bs{a}^{(T)}, \Wo,  \bo) ) \Big)^2 \Big] 
 & \leq \norm{\bs{a}^* - \bs{a}^{(T)}}_2^2 \E_{\bs{x}} \left[  \norm*{  \phi(\Wo \bs{x} + \bo) }_2^2 \right] \\ 
&  \leq   \norm{\bs{a}^* - \bs{a}^{(T)}}_2^2 \sum_{j = 1}^{2m} \norm{  \Wo_{j*}}_2^2 + (  \boj )^2 \\
&  \leq   \tilde O \left( 1/n\right). \label{poprisk:eq0}
}For the first term,  for $C > 0$ and the event $E_{C} \equiv   \abs*{  \gt( \bs{V}^\top \bs{x})- \hat y( \bs{x}; (\bs{a}^{*}, \Wo, \bo) )  } > C$,  we have
\eq{
  \E_{(\bs{x}, y)} \left[ \left(  y- \hat y( \bs{x}; (\bs{a}^{*}, \bs{W}^{(1)}, \bo) ) \right)^2 \right]  
 & \leq  \E \left[ \left( y  - \hat y( \bs{x}; (\bs{a}^{*}, \Wo, \bo) ) \right)^2 \wedge C^2  \right]  \\
  &+  \E \left[ \left(  y - \hat y( \bs{x}; (\bs{a}^{*}, \Wo, \bo) ) \right)^2 \indic{E_{C}} \right].  
}
Here,
\eq{
& \E_{(\bs{x}, y)} \left[ \left(  y - \hat y( \bs{x}; (\bs{a}^{*}, \Wo, \bo) ) \right)^2 \indic{E_{C}} \right] \\
&  \leq   \left( \E  [ y^4 ]^{1/4} + \E [\hat y (\bs{x}; (\bs{a}^{*},  \Wo, \bo) )^4]^{\frac{1}{4}} \right)^2 \mpr_{\bs{x}}     \left[   \abs{    \gt( \bs{V}^\top \bs{x})  - \hat y( \bs{x}; (\bs{a}^{*}, \Wo, \bo) ) } >  C  \right]^{\frac{1}{2}} \\ 
&  \leq   \tilde O(1) \mpr_{\bs{x}} \left[   \abs{   \gt( \bs{V}^\top \bs{x}) - \hat y( \bs{x}; (\bs{a}^{*}, \Wo, \bo) ) } > C \right]^{\frac{1}{2}}, \label{poprisk:arg1}
}
 where we use Lemma \ref{lem:relunetmoments}, and $\norm{\bs{a}^*}^2_2 \leq \tilde O(1/m)$, $\norm{\bo}_2^2 \leq 4m$,  and $\norm{\Wo_{j*}}_2 \leq \tilde O(1)$ in the last line.   By choosing 
\eq{
 C   \coloneqq    \norm{\bs{a}^*}_2 \sqrt{\norm{\bo}_2^2 + \norm{\Wo}_F^2} + ( \norm{\bs{a}^*}_2 \norm{\Wo}_F )  \sqrt{2 \log(4n)}   +   3 C_1  (2e  \log 6n)^{C_2} \leq \tilde O \left(  1 \right),   
}
by Lemma \ref{lem:relunettailbound}, we have  $\eqref{poprisk:arg1} \leq   \tilde O \left( 1/\sqrt{n} \right)$.  On the other hand,  by \eqref{poprisk:normbound} and \eqref{poprisk:empriskbound},  we have with probability at least $1 - d^{-u}$,
\eq{
& \E_{(\bs{x}, y)} \left[ \left(y - \hat y( \bs{x}; (\bs{a}^{*}, \Wo \bo) ) \right)^2 \wedge C^2  \right]  \\
 &  \leq  \Delta \E [ \epsilon^2 ] +  O \left(  \frac{ \left(1 + \log(4n d^u) \right)^{2p}  \left(1 + \log(P) \right)^{2p(k^{\star} - 1)}}{\rho_1 \log^{\rho_2} d} \right)+  \tilde O \left(  \frac{1}{m}  + \frac{1}{M}   +  \sqrt{ \frac{M \log \left( \frac{6d}{M} \right)}{n}}  \right). ~~~~~ \label{poprisk:arg3}
}
By  \eqref{poprisk:eqinit}-\eqref{poprisk:arg3}, the statement follows for SI.

\smallskip
For MI,  we observe that the setting is identical except that here we have $\hmu \vert_{\cJ}$.   By observing that $\norm{\hmu \vert_{\cJ}}_2 \leq \tilde O(1)$  (by  \eqref{cond:boundmu}  in Corollary \ref{cor:fl-existence}),  we can adjust the steps between \eqref{poprisk:eqinit}-\eqref{poprisk:arg3} to prove the statement for MI.
\end{proof}

\section{Lower bounds for CSQ methods}
Correlational Statistical Query (CSQ) algorithms are a family learners that can access data using queries $h : \R^d \to \R$ with $\E_{\bs{x}}[h(\bs{x})^2] \leq 1$ and returns $\E_{(\bs{x},y)}[h(\bs{x})y]$ within an error margin $\tau$.  In our setting,  since $y = \gt(\bs{V}^\top \bs{x})+ \sqrt{ \Delta} \epsilon$, where $\epsilon$ is independent zero-mean noise, the query returns a value in $\E_{\bs{x}}[h(\bs{x})\gt(\bs{V}^\top \bs{x})] + [- \tau, + \tau]$.    An instance of a CSQ algorithm is gradient descent on the population square loss with added noise in the gradients.  In this part, we give a lower bound on the CSQ complexity of learning a function in 
\eq{
\cF_{r,k} \coloneqq \left \lbrace \bs{x} \to \frac{1}{\sqrt{rk!}} \sum_{j = 1}^r \Hek(\inner{\bs{V}_{*j}}{\bs{x}}) ~ \vert ~ \bs{V}\in \R^{d \times r}, ~  \bs{V}^\top  \bs{V} = \ide{r}, ~ \norm{\bs{V}}_{2,q}^q \leq r^{\frac{q}{2}} d^{\alpha \left( 1 - \frac{q}{2} \right)}   \right \rbrace,   \label{eq:frkapp}
}
when $\bs{x} \sim \cN(0, \ide{d})$.  Here,  $\Hek$ denotes the $k$th Hermite polynomial (see Definition \ref{def:hermite}), and we use the convention $\norm{\bs{V}}_{2,0}^0 \coloneqq \norm{\bs{V}}_{2,0}$.

For notational convenience,   in the following,  ``$d$ is large enough'' means that $d \geq d^{*}(r,q,\alpha,k)$, where $d^*(r,q,\alpha,k)$ is a constant depending on the problem parameters $(r,q,\alpha,k)$.  Without loss of generality, we can assume all $d^{*}$'s are the same since if not, we can take their maximum.  We will use $\gtrsim$, $\lesssim$, and $\Omega(\cdot)$,   to suppress constants depending on  $(r,q,\alpha,k)$ in inequalities and lower bounds.   We will use $\widetilde O(\cdot)$ to suppress the aforementioned constants and the logarithmic terms in $d$ in upper bounds.\footnote{Here, one might be concerned by the possibility of trivial bounds when $q = 0$.  Although,  our notation does not exclude such problematic cases,  we will use our notation for the sake of readability as such problematic cases do not appear in our proof.}  The main theorem of this section is as follows:
\begin{theorem}[Restatement of Theorem \ref{thm:csqlbtext}]
\label{thm:csqlb}
Consider $\cF_{r,k}$ with some $q \in [0,2)$ and $\alpha \in (0,1)$.  If $d$ is large enough,  any CSQ algorithm for $\cF_{r,k}$ that guarantees error $\varepsilon = \Omega(1)$ requires either queries of accuracy, i.e.,  $\tau = \widetilde O\left(  d^{\left(\alpha \wedge \frac{1}{2} \right)  \frac{-k}{2}} \right)$ or super-polynomially many queries in $d$.
\end{theorem}

To prove our lower bound,  we will use the argument in \cite[Lemma 2]{Damian2022NeuralNC},  for which we need to create a large family of functions with a small average correlation.  With the following lemma,  we  construct such a function class.

\begin{lemma}
\label{lem:csqdim2}
Let $q \in [0,2)$,  $\alpha \in (0,1)$,  $r \in \N$.   When $d $ is large enough,   for any $c, k \geq 1$,  we can find a set of orthonormal matrices $\cV \subseteq \R^{d \times r}$ such that
\begin{itemize} 
\item $\abs{\mathcal{V} } \gtrsim \exp \left(  \Omega(d^\alpha)  \right) \wedge c^r d^k   $, 
\item $\max_{\bs{V} \in \mathcal{V} }\norm{\bs{V}}_{2,q}^q \leq r^{\frac{q}{2}} d^{ \alpha \left( 1 - \frac{q}{2}  \right) }$,
\item $\max_{ \substack{\bs{V}^{(1)}, \bs{V}^{(2)} \in \mathcal{V}  \\ \bs{V}^{(1)} \neq \bs{V}^{(2)}}}  \frac{1}{r} \sum_{i,  j = 1}^r \abs*{\inner{\bs{V}^{(1)}_{*i}}{\bs{V}^{(2)}_{*j}}}^k \lesssim   \frac{    \log^k (c d^k)}{  d^{k \left(\alpha \wedge \frac{1}{2} \right)  }} $  .
\end{itemize}
\end{lemma}

\begin{proof}
Let $\tilde d = \floor*{\frac{d}{r}}$ and $s = \floor*{ \frac{2d^{\alpha}}{3^{\frac{2}{2-q}} r} }.$  When $d$ is large enough,  $\frac{\tilde d}{2} \geq s \geq 64.  $ Hence,  by Corollary \ref{cor:csqdim},  we can find a set $\mathcal{U} \subseteq S^{\tilde d -1}$ such that 
\begin{itemize} 
\item $\abs{\mathcal{U} } \geq \frac{1}{3} \min \{ e^{\frac{s}{16}},   c r^k \tilde d^k \}  \geq  \tfrac{1}{6} \min \left \lbrace \exp \left[ \frac{d^{\alpha}/16}{ 3^{\frac{2}{2-q}} r} \right],   c d^{k} \right \rbrace,$  where the second inequality holds when $d$ is large enough.
\item $\max_{\bs{x} \in \mathcal{U} }\norm{\bs{x}}_q^q \leq \frac{r^{\frac{q}{2}} d^{ \alpha \left(1 - \frac{q}{2} \right)}}{r}$,
\item $\max_{ \substack{\bs{x}, \bs{y} \in \mathcal{U}  \\ \bs{x} \neq \bs{y}}} \abs{ \inner{\bs{x}}{\bs{y}} } \leq 8 C e  \frac{\log (c r^k \tilde d^k)}{\min \{ \sqrt{\tilde d}, s \}}\leq 16 C e  3^{\frac{2}{2-q}} r \frac{\log (c d^k)}{\min \left \lbrace d^{1/2} ,  d^\alpha \right \rbrace} $, where the second inequality holds when $d$ is large enough.
\end{itemize}

Hence, we can partition $\mathcal{U}$ into r equally sized mutually exclusive sets, and for using a vector from each set,  we can form a set of orthonormal matrices $\mathcal{V} \subset \R^{d \times r}$ such that
\begin{itemize}[leftmargin=*]
\item $\abs{\mathcal{V} } \geq \frac{1}{(6r)^r} \min \left \lbrace \exp \left[ \frac{ d^{\alpha}/ 16}{3^{\frac{2}{2-q}} } \right],   c^r d^{rk} \right \rbrace.$
\item $\max_{\bs{V} \in \mathcal{V} }\norm{\bs{V}}_{2,q}^q \leq r^{\frac{q}{2}} d^{\alpha \left(1 - \frac{q}{2} \right)}$,
\item $\max_{ \substack{\bs{V}^{(1)}, \bs{V}^{(2)} \in \mathcal{V}  \\ \bs{V}^{(1)} \neq \bs{V}^{(2)}}} \frac{1}{r} \sum_{i, j = 1}^r \abs*{\inner{\bs{V}^{(1)}_{*i}}{\bs{V}^{(2)}_{*j}}}^k  \leq   \frac{ (16 r C e)^k  3^{\frac{2k}{2-q}}   \log^k(c d^k)}{\min \left \lbrace d^{k/2} ,  d^{\alpha k} \right \rbrace} $.
\end{itemize}
\end{proof}

\subsubsection*{Proof of Theorem \ref{thm:csqlb}}
 
\begin{proof}[Proof of Theorem \ref{thm:csqlb}]
Let $Q$ represents the number of queries.  We consider polynomial queries,  i.e.,    $Q \leq d^C$ for some $C \in \N$.  Let $h_{e_k} \coloneqq \frac{1}{\sqrt{k!}} \Hek$ be the normalized $kth$ Hermite polynomial.    By Lemma \ref{lem:csqdim2},  we can construct the following function class which is a subset of $\cF_{r,k}$:
\eq{
\cF_q \coloneqq    \left \lbrace \frac{1}{\sqrt{r}} \sum_{j =1}^r h_{e_k}(\inner{\bs{V}_{*j}}{\bs{x}}) ~  \big \lvert ~ V \in \cV   \right \rbrace~\text{and} ~ \bs{x}  \sim \cN(0,\ide{d}),
}
where $\norm{\bs{V}}_{2,q}^q \leq r^{\frac{q}{2}} d^{\alpha (1 - \frac{q}{2})},$  for $\alpha \in (0,1)$,   $\abs{\cV} \geq  \Omega \left( \exp \left(  \Omega(d^\alpha)  \right) \wedge d^C d^k   \right)$,  where we used $c = d^C.$ We observe that for any different $f, \tilde f \in \cV$,   we have
\eq{
\E[f(\bs{x})^2] = 1 ~~ \text{and} ~~  \E[f(\bs{x}) \tilde f(\bs{x})] \leq \varepsilon \lesssim   \frac{    \log^{k} (d)}{  d^{k \left(\alpha \wedge \frac{1}{2} \right)  }} 
}
Therefore, by  \cite[Lemma 2]{Damian2022NeuralNC}, to get a population loss  $\E [(f(\bs{x}) - f^*(\bs{x}))^2] \leq  2 - 2\varepsilon$
\eq{
\tau^2 \lesssim \frac{d^C}{\exp \left(  \Omega(d^\alpha)  \right) \wedge d^{Cr} d^k   }  +    \frac{    \log^k (d)}{  d^{k \left(\alpha \wedge \frac{1}{2} \right)  }}  &  \lesssim   \frac{    \log^k (d)}{  d^{k \left(\alpha \wedge \frac{1}{2} \right)  }}  \label{thmlb:eq0}
}
where we use  $d^{Cr +k} \leq \exp(\Omega(d^\alpha))$ for $d$ is large enough in the first line.  We observe that for $d$ large enough,  $\varepsilon \leq 1$.
By taking the square root of both sides in \eqref{thmlb:eq0},  we obtain the statement.
\end{proof}

\subsection{Lemmas for Lower Bounds}

\subsubsection{Preliminaries}

In this section,  we will use Rosenthal-Buckholder inequality and Chernoff-Hoeffding bound given as follows.
\begin{lemma}[{\cite[Theorem 5.2]{Pinelis1994}} (and see {\cite[Lemma 22]{damian2023smoothing}})]
\label{lem:rosenthal}
\hfill \\
Let $\{ Y_i \}_{i=0}^n$ be a martingale with martingale difference sequence $\{ X_i \}_{i=1}^n$ where $X_i = Y_i - Y_{i-1}$.  Let
\eq{
\langle Y_n \rangle= \sum_{i=1}^n \E [ \abs{X_i}^2 \lvert \cF_{i-1}]
}
denote the predictable quadratic variation. Then, there exists an absolute constant $C$ such that for all $p \geq 2$
\eq{
\norm{Y_n}_p \leq C \left[ \sqrt{p} \norm{ \langle Y_n \rangle^{1/2} }_p + p n^{1/p} \max_i \norm{X_i}_p \right]. 
}
\end{lemma}

\begin{lemma}[Chernoff-Hoeffding Bound]
\label{lem:chernoff}
Let $X_1, \cdots, X_n \sim_{iid} Ber(p)$, where $p \in (0, \frac{1}{2}]$  We have
\eq{
\mpr \left[ \left \lvert \frac{1}{n} \sum_{i=1}^n (X_i - p) \right \rvert \geq \frac{p}{2} \right] \leq 2 \exp \left(  \frac{- p n}{16} \right).
}
\end{lemma}

\subsubsection{Lemmas for Lower Bounds}
For the following,  we define a probability distribution  $P_s$, parametrized by $s \in [d]$,  as follows:  For $\bs{x} \coloneqq (\bs{x}_1, \cdots, \bs{x}_d)^\top$,  
\eq{
\bs{x} \sim P_s ~ \text{if}  ~
\bs{x}_i \sim_{iid} 
\begin{cases}
\frac{1}{\sqrt{s}} & \text{wp} \: \frac{s}{2d} \\
\frac{- 1}{\sqrt{s}} & \text{wp} \: \frac{s}{2d} \\
0 & \text{wp} \: 1 - \frac{s}{d}
\end{cases}, ~~ \text{for}~ i = 1, \cdots, d. \label{eq:prodef}
}

\begin{lemma}
\label{lem:innerproduct}
Let $\bs{x}, \bs{y} \sim_{iid} P_s$.   For $s \in [d]$ and $p \geq 2$,  we have
\eq{
\mpr \left[ \abs{\inner{\bs{x}}{\bs{y}}} \geq C e \left( \sqrt{\frac{p}{d}} + \frac{p}{\sqrt{d}} \left( \frac{s^2}{d}  \right)^{\frac{1}{p} - \frac{1}{2}}  \right) \right] \leq e^{-p}.  \label{eq:probbound}
}
\end{lemma}

\begin{proof}
For any $i \in [d],$ note that  $\E [ \bs{x}_i] = 0$ and $\E \left[ \abs{\bs{x}_i}^p  \right] = \frac{s}{d} s^{ - p/2}$,.  Therefore, by independence, we have  $\E \left[ \abs{\bs{x}_i \bs{y}_i}^p  \right] =  s^{2-p}/d^2$.
By following the notation in Lemma \ref{lem:rosenthal},  we let  $Y_0 \coloneqq 0 ~ \text{and} ~ Y_d  \coloneqq \sum_{i=1}^d \bs{x}_i \bs{y}_i$,
where $X_i = Y_{i} - Y_{i-1} = \bs{x}_i \bs{y}_i$.  We have  $\norm{X_i}_p =  \E \left[ \abs{\bs{x}_i \bs{y}_i}^p  \right]^{1/p} = s^{2/p - 1} d^{-2/p}$,  and by the independence of $\bs{x}$ and $y$,  $\langle Y_d \rangle = 1/d$.  Hence, by Lemma \ref{lem:rosenthal}, for $p \geq 2$,
\eq{
\norm{Y_d}_p \leq C \left[ \sqrt{ \frac{p}{d}} + \frac{p}{\sqrt{d}} \left( \frac{s^2}{d}  \right)^{\frac{1}{p} - \frac{1}{2}}   \right].
}
The statement follows by Markov's inequality.
\end{proof}

\begin{corollary}
\label{cor:innerproduct}
By Lemma \ref{lem:innerproduct},  for $s \in [d]$ and $p \geq 2$,  we have
\eq{
\mpr \left[ \abs{\inner{\bs{x}}{\bs{y}}} \geq 2 C e \frac{p}{\min \{ \sqrt{d}, s \}}  \right] \leq e^{-p}.
}
\end{corollary}

\begin{proof}
The statement immediately follows from  \eqref{eq:probbound}.
\end{proof}

\begin{lemma}
\label{lem:nonzero}
Let $\bs{x} \sim P_s$.  For $d \geq 2s$, we have  $\mpr \left[  \big \lvert \norm{ \bs{x} }_0 - s \big \rvert \geq \tfrac{s}{2} \right] \leq 2 e^{\frac{- s}{16}}$. 
\end{lemma}

\begin{proof}
Note that $\mathbbm{1}_{\bs{x}_i \neq 0} \sim Ber(\frac{s}{d})$ and $\norm{\bs{x}}_0 = \sum_{i=1}^d \mathbbm{1}_{\bs{x}_i \neq 0}$.  Since $d \geq 2s$,  by using Lemma \ref{lem:chernoff}, we have
\eq{
\mpr \left[  \left \lvert \frac{1}{d} \sum_{i=1}^d \left( \mathbbm{1}_{\bs{x}_i \neq 0} - \frac{s}{d}  \right) \right \rvert \geq  \frac{s}{2d} \right] \leq 2 e^{\frac{- s}{16}},
}
which is equivalent to the statement.
\end{proof}

\begin{lemma}
\label{lem:csqdim}
Fix any $q \in [0,2)$.  For any $s \leq \frac{d}{2}$,  let  $\bs{x}^{(1)}, \cdots, \bs{x}^{(n)} \sim_{iid} P_{s}$.  For any $c, k \geq 1$, we  let
\eq{
\varepsilon \coloneqq 8 C e  \frac{\log (c d^k)}{\min \{ \sqrt{d}, s \}}.
}
For $s \geq 5$,  we have
\eq{
\mpr \left [ \max_{i \in [n]} ~  \norm*{ \frac{\bs{x}^{(i)}}{\norm{\bs{x}^{(i)}}_2} }_q^q \leq  3 \left( \tfrac{s}{2} \right)^{\frac{2-q}{2}}   ~\text{AND} ~   \max_{ \substack{i, j \in [n] \\ i \neq j}} \abs*{ \inner{ \frac{\bs{x}^{(i)}}{\norm{\bs{x}^{(i)}}_2} } { \frac{\bs{x}^{(j)}}{\norm{\bs{x}^{(j)}}_2} } }  \leq \varepsilon \right]  \geq 1 - 2 n e^{\frac{- s }{16}} -   \frac{n^2}{c^2	d^ {2k}}.
}
\end{lemma}

\begin{proof}
We observe that
\eq{
&  \max_{i \in [n]} ~ \abs{  \norm{\bs{x}^{(i)}}_0 -   s  } \leq \frac{ s}{2}   ~~\text{AND} ~~ \max_{ \substack{i, j \in [n] \\ i \neq j}}  \abs*{\inner{\bs{x}^{(i)}}{ \bs{x}^{(j)}}} \leq \frac{\varepsilon}{2}  \label{lemcsq:target}  \\
& \Rightarrow    \max_{i \in [n]} ~ \abs{  \norm{\bs{x}^{(i)}}_0 -   s  } \leq \frac{s}{2}   ~\text{AND} ~ \max_{ \substack{i, j \in [n] \\ i \neq j}}  \abs*{ \inner{ \frac{\bs{x}^{(i)}}{\norm{\bs{x}^{(i)}}_2} } { \frac{\bs{x}^{(j)}}{\norm{\bs{x}^{(j)}}_2} } }  \leq \varepsilon \\
& \Rightarrow  \max_{i \in [n]} ~  \norm*{ \frac{\bs{x}^{(i)}}{\norm{\bs{x}^{(i)}}_2} }_q^q \leq 2^{\frac{q}{2} - 1} 3 s^{\frac{2-q}{2}}   ~~\text{AND} ~~ \max_{ \substack{i, j \in [n] \\ i \neq j}}  \abs*{ \inner{ \frac{\bs{x}^{(i)}}{\norm{\bs{x}^{(i)}}_2} } { \frac{\bs{x}^{(j)}}{\norm{\bs{x}^{(j)}}_2} } }  \leq \varepsilon 
}
where the second line holds since  $\norm{\bs{x}^{(i)}}_0 \geq s/2$ implies $ \norm{\bs{x}^{(i)}}_2^2 \geq 1/2$ and  the last statement holds since  $ 3s/2 \geq \norm{\bs{x}^{(i)}}_0 \geq s/2$ implies  $ \norm{\bs{x}^{(i)}}_2 \geq 1/\sqrt{2}$ and  $\norm{\bs{x}^{(i)}}_q^q \leq \frac{3}{2} s^{\frac{2-q}{2}}$.
In the following, we will lower bound \eqref{lemcsq:target}.  Since $d \geq 2 s$, by Lemma \ref{lem:nonzero}, we have
\eq{
\mpr \left [ \max_{i \in [n]} \abs{ \norm{\bs{x}^{(i)}}_0 -  s} > \frac{  s}{2} \right]&  \leq \sum_{i \in [n]}  \mpr \left [ \abs{ \norm{\bs{x}^{(i)}}_0 -  s} \geq \frac{   s}{2} \right]  \leq 2 n \exp \left( \frac{- s}{16}  \right). \label{lemcsq:interm1} 
}
Moreover, for any $i \neq j \in [n]$,
\eq{
\mpr \left [ \abs*{\inner{\bs{x}^{(i)}}{\bs{x}^{(j)}}} \geq \frac{\varepsilon}{2} \right] & = \mpr \left [ \abs*{\inner{\bs{x}^{(i)}}{\bs{x}^{(j)}}} \geq 4 C e \frac{\log(c d^k)}{\min \{ \sqrt{d}, s \}}  \right] 
\leq \frac{1}{c^2 d^{2k} }. 
}
where the last step follows  Corollary \ref{cor:innerproduct}, since for $s \geq 5$,   we have $d \geq 10$ and $\log(cd^k) \geq 2$ for $c, k \geq 1$.
Therefore,
\eq{
\mpr \left [  \max_{ \substack{i, j \in [n] \\ i \neq j}}  \abs*{\inner{\bs{x}^{(i)}}{\bs{x}^{(j)}}} > \frac{\varepsilon}{2} \right]  \leq \frac{n^2}{c^2 d^{2k}}. \label{lemcsq:interm2} 
}
By lower bounding \eqref{lemcsq:target} with \eqref{lemcsq:interm1} and \eqref{lemcsq:interm2}, we obtain the result.
\end{proof}

\begin{corollary}
\label{cor:csqdim}
For any $q \in [0,2)$ and $64 \leq s \leq \frac{d}{2}$  and $k, c \geq 1$,  there exists a set $\mathcal{U} \subseteq S^{d-1}$ such that 
\begin{itemize}
\item $\abs{\mathcal{U} } \geq \frac{1}{3} \min \{ e^{\frac{s}{16}},   c d^k \}$, 
\item $\max_{\bs{x} \in \mathcal{U} }\norm{\bs{x}}_q^q \leq 3 \left( \frac{s}{2} \right)^{\frac{2-q}{2}}$,
\item $\max_{ \substack{x, y \in \mathcal{U}  \\ \bs{x} \neq \bs{y}}} \abs{ \inner{\bs{x}}{\bs{y}} } \leq \varepsilon$, where $\varepsilon$ is defined in Lemma \ref{lem:csqdim}.
\end{itemize}
\end{corollary}

\begin{proof}
Consider Lemma \ref{lem:csqdim} with $q \in [0,2)$, $5 \leq s \leq \frac{d}{2}$, $k, c \geq 1$,   and $n  = \ceil{ \frac{1}{3} \min \{ e^{\frac{s}{16}},  c d^{ k} \}}$. We observe that the probability of the event in  Lemma \ref{lem:csqdim} is nonzero.
Hence, there exists such $\mathcal{U}$ as a subset of the normalized versions of the support of $P_s$.
\end{proof}

\section{Miscellaneous}

\subsection{Laurent-Massart Lemma and Its Corollaries} 
\begin{lemma}[Laurent-Massart Lemma]
\label{lem:laurentmassart}
Let $X$ be a chi-square with $N$ degrees of freedom. For any $t > 0$,
\eq{
(\rnu{1})~~ \mpr \left[   X - N \geq 2 \sqrt{N t} + 2 t \right] \leq e^{-t} ~~ \text{and} ~~
(\rnu{2})~~ \mpr \left[   X - N \leq - 2 \sqrt{N t}\right] \leq e^{-t}.
}
\end{lemma}

\begin{corollary}
\label{cor:laurentmassart1}
Let $\bs{w} \sim \cN(0,  \ide{d}).$  For $d \geq 16 \log(1/\delta),$  we have with probability at least $1 - \delta$,  $\norm{\bs{w}}_2^2 \geq \frac{d}{2}$.
\end{corollary}

\begin{proof}
By Lemma \ref{lem:laurentmassart}, with probability at least $1 - \delta$,  for $d \geq 16  \log(1/\delta)$,  $\norm{\bs{w}}_2^2=  \sum_{i = 1}^d \bs{w}_i^2 \geq d - 2\sqrt{d \log(1/\delta)} \geq \frac{d}{2}$.
\end{proof}

\begin{corollary}
\label{cor:laurentmassart2}
For  $r \leq d_1 \wedge d_2$, let $\bs{A} \in \R^{d_1 \times d_2}$  be a rank-$r$ matrix.  For $\bs{w} \sim \cN(0,  \ide{d_2})$, we have
\eq{
\mpr \left[ \norm{\bs{A} \bs{w}}^2_2 \geq 3 \norm{\bs{A}}^2_2 (r + \log(1/\delta))  \right] \leq \delta.
}
\end{corollary}

\begin{proof}
Since $\bs{A}$ is rank-$r$,  by using SVD, we can write that $\bs{A} = \bs{U} \bm{\Sigma} \bs{L}^\top$  where $\bs{U} \in \R^{d_1 \times r}$ and $\bs{L} \in \R^{d_2 \times r}$ are orthonormal, $\bm{\Sigma} \in \R^{r \times r}$ is diagonal.  For $\tilde{\bs{w}} \coloneqq \bs{L}^\top \bs{w}$,  we have  $\norm{\bs{A}\bs{x}}^2_2 =^d \norm{\bm{\Sigma} \tilde{\bs{w}}}^2_2 \leq \norm{\bs{A}}^2_2 \norm{\tilde{\bs{w}}}_2^2$.  By using Lemma \ref{lem:laurentmassart}, we have with probability at least $1 - \delta$,  $\norm{\bs{A}}^2_2 \norm{\tilde{\bs{w}}}_2^2 \leq  \norm{\bs{A}}^2_2 (r + 2 \sqrt{r \log(1/\delta)} + 2 \log(1/\delta) ).$  By observing that   $\left( r + 2 \sqrt{r \log(3/\delta)} + 2 \log(3/\delta) \right)  \leq 3(r + \log(3/\delta))$, we prove the statement.
\end{proof}

\begin{lemma}
\label{lem:highprobprobbound1}
Suppose we have $\{c_1, \cdots, c_r\} \subset \R$ and  an orthonormal $\{ \bs{v}_1, \cdots, \bs{v}_r\} \subset \R^{d}$.  For $k \in \N$ and $\delta \in (0,1]$,  if    $\max_{i \in [n]} \norm{\bs{V}^\top \bs{x}_i}_2 \leq C_{\cD}$  and  $M \geq 16 \log(2/\delta)$ hold,  then
\eq{
\mpr_{\bs{w}} \left[   \max_{i \in [n]}   \abs*{\sum_{l = 1}^r c_l \inner{\bs{v}_l}{\bs{x}_{i}} \inner{\bs{v}_l}{\wJ }^{k - 1} } >  C_{\cD}  \max_{l \leq r} \abs{c_l}  \left( \frac{ 6 \left( r+  \log(2/\delta) \right) }{M} \right)^{\frac{k - 1}{2}}  ~ \Bigg \vert ~ \{ (\bs{x}_i, y_i) \}_{i = 1}^n \right]\leq \delta.
}
\end{lemma}

\begin{proof}
By assumption, we have
\eq{
\max_{i \in [n]} \abs*{\sum_{l = 1}^r c_l \inner{\bs{v}_l}{\bs{x}_{i}} \inner{\bs{v}_l}{\wJ }^{k - 1} } 
\labelrel\leq{hpp1:ineqq0} C_{\cD}  \max_{l \leq r} \abs{c_l}   \left(  \sum_{l = 1}^r  \inner{\bs{v}_l}{\wJ }^{2}  \right)^{\frac{(k - 1)}{2}}.  \label{hphp1:eq0}
}
where  \eqref{hpp1:ineqq0} follows that $\norm{\bs{v}}_p \geq  \norm{\bs{v}}_q$ for $1 \leq p \leq q \leq \infty$.
On the other hand, by Corollaries \ref{cor:laurentmassart1} and  \ref{cor:laurentmassart2},  we have with probability at least $1 - \delta$,
\eq{
\sum_{l = 1}^r  \inner{\bs{v}_l}{\wJ }^{2} & =   \sum_{l = 1}^r  \frac{ \inner{\bs{v}_l \vert_{\cJ}}{\bs{w}}^{2} }{\norm{\bs{w} \vert_{\cJ}}^2_2} 
 \leq \frac{3(r + \log(2/\delta))}{M/2} 
 = \frac{6(r + \log(2/\delta))}{M}.  \label{hphp1:eq1}
}
\end{proof}

\begin{lemma}
\label{lem:highprobprobbound2}
We have for $\delta \in (0,1]$ and $M \geq 16 \log(2/\delta)$,
\eq{
\mpr_{\bs{w}} \left[ \left( \sum_{l = 1}^r c^2_l \inner{\bs{v}_l}{\wJ}^{2(k-1)} \right)^{\frac{1}{2}} >  6^{\frac{k-1}{2}} \max_{l \leq r} \abs{c_l}   \left( \frac{ r + \log(2/\delta)}{M} \right)^{\frac{k-1}{2}}  \right] \leq \delta.
}
\end{lemma}

\begin{proof}
We have  $ \left( \sum_{l = 1}^r c^2_l \inner{\bs{v}_l}{\wJ}^{2(k-1)} \right)^{\frac{1}{2}} \leq   \max_{l \leq r} \abs{c_l}   \left( \sum_{l = 1}^r \inner{\bs{v}_l}{\wJ}^{2(k-1)} \right)^{\frac{1}{2}}$.  The statement follows the argument in  \eqref{hphp1:eq0} and \eqref{hphp1:eq1}.
\end{proof}

\begin{lemma}
\label{lem:reluadditivetermsaux}
Let $\mathcal{A} \subset \R^{d_1 \times d}$ such that  for any $\bs{A} \in \mathcal{A}$,  $\norm{\bs{A}}_2 \leq 1$ and $rank(\bs{A}) \leq r$.
For $\bs{x}_1, \cdots, \bs{x}_n \sim_{iid} \cN(0,\ide{d})$, we have with probability $1 - \delta$,
\eq{
\sup_{\bs{A} \in \mathcal{A}} \norm*{\frac{1}{n} \sum_{i = 1}^n \bs{A} \bs{x}_i \bs{x}_i^\top \bs{A}^\top - \bs{A} \bs{A}^\top   }_2 \leq \sqrt{ \frac{r}{n} } + \sqrt{ \frac{2 \log(2/\delta)}{n}} + \sqrt{ \frac{  2 \log  \abs{\mathcal{A}} }{n} } 
}
\end{lemma}

\begin{proof}
Let's fix a $\bs{A} \in \mathcal{A}$.  By SVD, we can write $\bs{A} = \bs{U} \bm{\Sigma} \bs{L}^\top$, where $\bs{U},\bs{L} \in \R^{d \times r}$ are orthonormal and $\bm{\Sigma} \in  \R^{r \times 
r}$ is diagonal.  For $\tilde{\bs{x}}_i \coloneqq \bs{L}^\top \bs{x}_i,$ since  $\norm{\bs{A}}_2 = 1$,  we have   $$\norm*{\frac{1}{n} \sum_{i = 1}^n  \bs{A} \bs{x}_i 
\bs{x}_i^\top \bs{A}^\top - \bs{A} \bs{A}^\top   }_2\leq  \norm*{  \tfrac{1}{n} \sum_{i = 1}^n  \tilde{\bs{x}}_i \tilde{\bs{x}}_i^\top - \ide{r} }_2.$$ By 
\cite[Corollary 5.35]{Vershynin2010IntroductionTT}, for a fixed $\cJ \in \cH$, we have with probability at least $1 - \delta$,    $\norm*{  \tfrac{1}{n} \sum_{i 
= 1}^n  \tilde{\bs{x} }_i\tilde{\bs{x}}_i^\top - \ide{r}  }_2 \leq \sqrt{ \tfrac{r}{n} } + \sqrt{\tfrac{ 2 \log(2/\delta) }{n}}$.   By union bound and that  $
\sqrt{a+b}\leq \sqrt{a} + \sqrt{b}$ for $a,b > 0$, the statement follows.
\end{proof}
\subsection{Lemmas for Bounding Polynomials of Gaussian Random Vectors}

\begin{lemma}[Moments of Gaussian Vector]
\label{lem:gaussianvectormoments}
For $\bs{x} \sim \cN(0,  \ide{d})$, we have $\E [ \norm{\bs{x}}_2^{2k}] = d (d+ 2) \cdots (d + 2k - 2)$.
For $d \geq 2k$, we have $\E [ \norm{\bs{x}}_2^{2k}]^{-1} \geq 2^{-k} d^{-k}$.
\end{lemma}

\begin{lemma}[Hypercontractivity]
\label{lem:hypercontactivity}
Let $P_k : \R^d \to \R$ be a polynomial of degree-$k$.  For $q \geq 2$, we have $
\E_{\bs{x} \sim \cN(0,\ide{d})} \left[ P_k(\bs{x})^q \right]^{1/q} \leq (q-1)^{k/2}  \E_{\bs{x} \sim \cN(0, \ide{d})} \left[ P_k(\bs{x})^2 \right]^{1/2} .$
\end{lemma}
\noindent
In the following, we will state some consequences of Lemmas \ref{lem:gaussianvectormoments} and \ref{lem:hypercontactivity}.

\begin{corollary}
\label{cor:momenthyper}
For $\bs{z} \sim \cN(0,  \ide{r})$ and $p \geq 2$,  $\E[(1+ \norm{\bs{z}}_2^2)^{p}]^{\frac{1}{p}}  \leq (p - 1) (r+2)$.
\end{corollary}

\begin{proof}
By Lemma \ref{lem:gaussianvectormoments} and \ref{lem:hypercontactivity},   $\E[(1+ \norm{\bs{z}}_2^2)^{p}]^{\frac{1}{p}} \leq (p - 1) \E[(1+ \norm{\bs{z}}^2)^{2}]^{\frac{1}{2}}   \leq (p - 1) (r+2)$.
\end{proof}

\begin{proposition}
\label{prop:polyconc}
For $\bs{z} \sim \cN(0, \ide{r})$ and $C > 0$, 
$\mpr \left[   (1+ \norm{\bs{z}}^2_2)^C \geq  u^C (r+2)^C   \right] \leq \exp \left( \frac{- u}{e}  \right)$,  for $u \geq 2e$.
\end{proposition}

\begin{proof}
By Corollary \ref{cor:momenthyper},   we have for $p \geq 2$  that  $\mpr \left[   (1+ \norm{\bs{z}}^2_2)^C \geq  u^C (r+2)^C   \right]  \leq  p^p u^{-p}$.   By using  $p^* = \frac{u}{e}$ and $u \geq 2e,$ we have the statement.
\end{proof}

\begin{corollary}
\label{cor:concg}
By Proposition \ref{prop:polyconc},   $ \mpr_{\bs{z} \sim \cN(0, \ide{r})} \left[   \abs{\gt(\bs{z})} \geq  C_1  u^{C_2} (r+2)^{C_2}   \right] \leq  \exp \left( \frac{- u}{e}  \right)$,   for $u \geq 2e$.
\end{corollary}

\begin{proposition}
\label{prop:concfory}
We have for $u \geq 2e$,   $\mpr \left[  \abs{y} \geq  C_1 (r +2)^{C_2} u^{C_2} + \sqrt{\Delta/e} u^{\frac{1}{2}}   \right] \leq 3 \exp \left( \tfrac{- u}{e}  \right)$.
\end{proposition}

\begin{proof}
By  $\abs{y}  \leq   \abs{\gt(\bs{V}^\top \bs{x})} + \sqrt{\Delta} \abs{\epsilon}$,  Corollary \ref{cor:concg}, $\mpr \left[ \abs{\epsilon}  > t \right]   \leq    2 e^{-t^2}$, the statement follows.
\end{proof}

\begin{proposition}
\label{prop:biasbound}
For $R = C_1 (r + 2)^{C_2} u^{C_2} + \sqrt{\Delta/e } u^{\frac{1}{2}}$ and $u \geq 2e$, we have
\eq{
\sup_{ \substack{ \bs{w}, \bs{v} \in S^{d-1} \\ b \in \R}} \abs*{  \E \left[ y \indic{\abs{y} > R} \inner{\bs{v}}{\bs{x}} \phi^\prime(\inner{\bs{w}}{\bs{x}} + b)  \right] } \leq 6^{\frac{3}{4}} \exp \left( \tfrac{- u}{2e}  \right) \Bigg( C_1^4   (4 C_2)^{4C_2}  (r+2)^{4 C_2} +  2 \Delta^2  \Bigg)^{\frac{1}{4}}.
}
\end{proposition}

\begin{proof}
Choose arbitrary $\bs{w}, \bs{v} \in S^{d-1}$ and $b \in \R$.  By using Cauchy-Schwartz inequality, we have
\eq{
\abs*{ \E \left[ y \indic{\abs{y} > R} \inner{u}{\bs{x}} \phi^\prime(\inner{\bs{w}}{\bs{x}} + b)  \right] } 
& \leq  \mpr \left[  \abs{y} \geq R   \right]^{\frac{1}{2}} \E[y^4]^{\frac{1}{4}} \E[\abs{ \inner{u}{\bs{x}} \phi^\prime(\inner{\bs{w}}{\bs{x}} + b)}^4]^{\frac{1}{4}} \\
&  \leq 3^{\frac{3}{4}}  \exp \left( \frac{- u}{2e}  \right)   \E[y^4]^{\frac{1}{4}},  \label{biasbound:eq0}
}
where we use   $\abs{\phi^\prime} \leq 1$ and Proposition \ref{prop:concfory} in \eqref{biasbound:eq0}.
We observe that
\eq{
\E[y^4]   
& \leq 2^3 (  \E \left[ ( \gt(\bs{V}^\top \bs{x})^4 \right] +\Delta^2 \E \left[ \epsilon^4  \right]  )   \\
&  \labelrel\leq{propbiasbound:ineqq0}  2^3 (  \E \left[ ( \gt(\bs{V}^\top \bs{x})^4 \right] +  2 \Delta^2  )  
 \labelrel\leq{propbiasbound:ineqq1}  2^3 \Big( C_1^4  (4 C_2)^{4C_2}  (r+2)^{4 C_2} +  2 \Delta^2  \Big).    \label{biasbound:eq1}
} 
where  \eqref{propbiasbound:ineqq0} follows from the tail inequality for $\epsilon$, and  \eqref{propbiasbound:ineqq1} follows  from Corollary \ref{cor:momenthyper}  since $C_2 \geq 1/2$.
By using \eqref{biasbound:eq1} in \eqref{biasbound:eq0} , we have the statement.
\end{proof}

\subsection{Magnitude Pruning}
\begin{lemma}
\label{lem:prunelemma}
For $\bs{u} \in \R^d$, let  $\mathcal{I}_{u}$ denotes the index set that includes the largest $M$ entries of $u$ and let $\bs{u} \tm$ denote the vector $\bs{u}$ with everything except $M$ largest coefficients set $0$.  For any $\bs{v} \in \R^d$ and $q \in (0,2]$, we have
\eq{
 (4^{(q-1) \vee 0}  +1) \sum_{i \in \mathcal{I}_{u} \cup \mathcal{I}_{v}} \abs{\bs{u}_i - \bs{v}_i}^q  \geq \norm{\bs{u} \tm - \bs{v}}_q^q -  4^{(q-1) \vee 0} \norm{\bs{v} - \bs{v} \tm}_q^q.
}
\end{lemma}

\begin{proof}
Without loss of generality, we can assume $\abs{\bs{v}_1} \geq   \abs{\bs{v}_2} \geq  \abs{\bs{v}_3} \cdots \geq  \abs{\bs{v}_d}$.  We have
\eq{
\norm{\bs{u} \tm - \bs{v}}_q^q =         \sum_{i \in \mathcal{I}_{u} \cap [M]}          \abs{\bs{u}_i - \bs{v}_i}^q  + \sum_{i \in \mathcal{I}_{u} - [M]}           \abs{\bs{u}_i - \bs{v}_i}^q   + \sum_{i \in [M] - \mathcal{I}_{u} }           \abs{\bs{v}_i}^q    +  \sum_{i \in [d] - \left( \mathcal{I}_{u} \cup [M] \right) }           \abs{\bs{v}_i}^q. ~~~~~~~~  \label{eq:normsquared}
}
If $\mathcal{I}_u = [M]$,  the statement follows by Proposition \ref{prop:lqtri}.  Therefore, suppose $\mathcal{I}_u \neq [M]$.  Let  $[M] - \mathcal{I}_u \coloneqq \{ j_1, \cdots, j_\kappa \}$ and  $\mathcal{I}_u - [M] \coloneqq \{ l_1, \cdots, l_\kappa \}$.  For some  $\iota = 1, \cdots,  \kappa$, we  get
\eq{
\abs{\bs{v}_{j_\iota}}^q  
=  \abs{\bs{v}_{j_\iota}  \pm \bs{u}_{j_\iota}}^q 
& \labelrel\leq{pr:ineqq0} 2^{(q-1) \vee 0}  \abs{\bs{v}_{j_\iota} - \bs{u}_{j_\iota}}^q +  2^{(q-1) \vee 0} \abs{\bs{u}_{l_\iota}}^q   \\  
& \labelrel\leq{pr:ineqq1}   2^{(q-1) \vee 0}   \abs{\bs{v}_{j_\iota} -  \bs{u}_{j_\iota}}^q +  4^{(q-1) \vee 0}     \abs{\bs{v}_{l_\iota} -  \bs{u}_{l_\iota}}^q +  4^{(q-1) \vee 0}   \abs{\bs{v}_{l_\iota}}^q,   \label{eq:boundjalpha}
}
where in \eqref{pr:ineqq0}, we use  Proposition \ref{prop:lqtri} and   $\abs{\bs{u}_{j_\iota}} \leq  \abs{\bs{u}_{l_\iota}}$, $j_{\iota} \in \mathcal{I}_u$,   and   Proposition \ref{prop:lqtri} for \eqref{pr:ineqq1}.  
By using \eqref{eq:boundjalpha} for $\iota = 1, \cdots,  \kappa$, we  get
\eq{
\eqref{eq:normsquared} 
& \labelrel\leq{pr:ineqq2}        \!\!\!\!     \sum_{i \in \mathcal{I}_u \cap [M]} \abs{\bs{u}_i - \bs{v}_i}^q  + (4^{(q-1) \vee 0}  +1)   \!\!\!\!              \sum_{i \in \mathcal{I}_u - [M]}  \abs{\bs{u}_i - \bs{v}_i}^q  + 2^{(q-1) \vee 0} \!\!\!\!  \!   \sum_{i \in [M] - \mathcal{I}_u }  \abs{\bs{u}_i - \bs{v}_i}^q    +   4^{(q-1) \vee 0} \!\!\!\! \!           \sum_{i \in [d] -  [M] }  \abs{\bs{v}_i}^q   \\
& \leq    (4^{(q-1) \vee 0}  +1) \sum_{i \in \mathcal{I}_u \cup [M]} \abs{\bs{u}_i - \bs{v}_i}^q  +  4^{(q-1) \vee 0}    \sum_{i \in [d] - [M] }  \abs{\bs{v}_i}^q, \label{eq:boundnormsquared}
}where  \eqref{pr:ineqq2} follows $\left(\mathcal{I}_u - [M] \right) ~ \cup ~ \left( [d] - \left( \mathcal{I}_u \cup [M] \right) \right) = [d] - [M] $.
By \eqref{eq:boundnormsquared}, the statement follows.
\end{proof}

\begin{lemma}
\label{lem:donoho}
Let $q \in (0,2)$ and $v \in \R^{d}$.   We have $\norm*{\bs{v} - \bs{v} \tm}_2 \leq \Big( \left( 1 - \frac{q}{2} \right)^{\frac{2-q}{q}} \frac{q}{2}   \Big)^{1/2} \norm{v}_q M^{\frac{-1}{q} + \frac{1}{2}}$,  for  $M = 1, 2 , \cdots, d$.
\end{lemma}

\begin{proof}
Without loss of generality, we assume $\abs{\bs{v}_1} \geq  \abs{\bs{v}_2} \geq \cdots \geq \abs{\bs{v}_d}$. Then, we have
\eq{
\norm*{\bs{v} - \bs{v} \tm}_2^2 = \sum_{i = M+1}^d \bs{v}_i^2 \leq \abs{\bs{v}_{M+1}}^{2 - q} \sum_{i = M+ 1}^d \abs{\bs{v}_i}^q. \label{eq:clippedbound}
}
Let   $\sum_{i = M+ 1}^d \abs{\bs{v}_i}^q = r$ and  $\sum_{i = 1}^d \abs{\bs{v}_i}^q = R$.  Then, we have 
\eq{
R - r = \sum_{i = 1}^M \abs{\bs{v}_i}^q   \geq   M \abs{\bs{v}_{M+1}}^q & ~  \Rightarrow ~ \abs{\bs{v}_{M+1}}^{2-q} \leq (R - r)^{\frac{2-q}{q}}  M^{ - \frac{2-q}{q}} ~ \Rightarrow ~ \eqref{eq:clippedbound} \leq (R - r)^{\frac{2-q}{q}} r  M^{ - \frac{2-q}{q}}.
}
The statement follows from $\max_{r \in [0, R]} (R - r)^{\frac{2-q}{q}} r \leq  \left( 1 - \tfrac{q}{2} \right)^{\frac{2-q}{q}} \tfrac{q}{2} R^{\frac{2}{q}}$.
\end{proof}

\subsection{Elementary Results}

\begin{corollary}
\label{cor:size}
For any $M \in [d]$ and $\epsilon > 0$,  let $\cN^{\epsilon}_{M} \subseteq S_M^{d-1}$ be the minimal $\epsilon$-cover. We have $\abs{\cN^{\epsilon}_M } \leq \binom{d}{M} \left(1+ \frac{2}{\epsilon} \right)^M$. 
\end{corollary}

\begin{proof}
By \cite[Corollary 4.2.13]{Vershynin2018HighDimensionalP}, we know that  the minimal $\epsilon$-cover of the unit sphere, i.e.,  $\cN^{\epsilon}  \subseteq S^{d-1}$, satisfies $\abs{\cN^{\epsilon}} \leq  (1+2/\epsilon)^d$. Then, by choosing $M$ subsets of $S^{d-1}$ and taking the union of $\epsilon$-covers restricted on the chosen indices, we can construct an $\epsilon$-cover for  $S_M^{d-1}$. Therefore, the statement follows.
\end{proof}

\begin{proposition}
\label{prop:lqtri}
For any $q \in (0, \infty],$ we have $\abs{a+b}^q \leq 2^{(q-1) \vee 0} (\abs{a}^q + \abs{b}^q)$.
\end{proposition}

\begin{proof}
Without loss of generality, let's assume $\abs{b} \geq \abs{a}.$   For $q \in (0,1],$ we have  $\abs{a+b}^q  \leq (\abs{a} + \abs{b})^q  \leq  \abs{a}^q +q \abs{a}^{q-1} \abs{b} \leq  \abs{a}^q +\abs{b}^q $,  where we use  that $x \to x^q$ is concave in the second inequality.
For $q >1,$ we have $\abs{a+b}^q \leq (\abs{a} + \abs{b})^q    \leq 2^{q-1} (\abs{a}^q + \abs{b}^q)$ where we use Jensen's inequality in the last step.
\end{proof}

\begin{lemma}
\label{lem:gaussmgf}
Let $\cosh(t) \coloneqq \tfrac{e^t + e^{-t}}{2}$.  For $Z \sim \cN(0,1)$, we have
\eq{
(\rnu{1}) ~ ~ \E[\cosh(\lambda Z^2)] \leq \exp \left( 4 \lambda^2 \right) , ~   \abs{ \lambda } \leq \tfrac{1}{2\sqrt{2}}  ~~ \text{and} ~~
(\rnu{2}) ~\E[\exp(\lambda^2 Z^2)] \leq \exp \left( 2 \lambda^2 \right), ~    \abs{ \lambda } \leq \tfrac{1}{2}.
}
\end{lemma}

\begin{proof}
Since $\abs{ \lambda } \leq \frac{1}{2\sqrt{2}}$, we have  $\E \left[ \exp ( \lambda Z^2)  \right] = \tfrac{1}{\sqrt{1 - 2\lambda}}$ and $\E \left[ \exp (- \lambda Z^2)  \right] = \tfrac{1}{\sqrt{1 + 2\lambda}}.$ 
Therefore,
\eq{
\E[\cosh(\lambda Z^2)]  = \frac{1}{2} \left( \frac{\sqrt{1 - 2\lambda} + \sqrt{1 + 2\lambda}}{\sqrt{1 - 4 \lambda^2}}  \right)  \leq \frac{1}{\sqrt{1 - 4 \lambda^2}}   \labelrel\leq{cosh:ineqq0}    \exp(4 \lambda^2) \label{cosh:eq0}
}
where \eqref{cosh:ineqq0} follows  $\tfrac{1}{1 - t} \leq \exp(2 t$) for $\abs{t} \leq 1/2$. The second statement also follows the same argument.
\end{proof}

\subsection{Lemmas for Feature Learning}

\begin{proposition}
\label{prop:relunetradamachercompl}
For $m \in \N$,  $M \in [d]$ and $(\bs{a},\bs{W},\bs{b}, \bs{u}) \in \R^m \times \R^{d \times m} \times \R^m \times \R^d$,  let
\eq{
  \Theta \coloneqq  \Big \lbrace (\bs{a},\bs{W},\bs{b}, \bs{u})   ~ \big \vert ~ \norm{\bs{a}}_2 \leq \frac{r_a}{\sqrt{m}}, ~ \norm{\bs{b}}_{\infty} \leq r_b,  ~&  \norm{\bs{u}}_2 \leq r_u ,  ~\norm{\bs{W}_{j*}}_2 \leq r_W, ~ \\ 
&  \norm{\bs{u}}_0 \leq M,  ~ \norm{\bs{W}_{j*}}_0 \leq M,  ~ j \in [m]  \Big \rbrace.  
}
and for some $\tau > 0$,  let $\cG   \coloneqq   \left \lbrace    (\bs{x}, y)   \to    \big (y - \inner{\bs{u}}{\bs{x}} -  \inner{\bs{a}}{\phi(\bs{W}^\top \bs{x} + \bs{b})} \big)^2   \wedge \tau^2 ~ \vert ~ (\bs{a},\bs{W},\bs{b}, \bs{u}) \in  \Theta \right \rbrace$ and
let $\mathcal{R}(\cG)$ denote the Rademacher complexity of $\cG$. Then, with $\bs{x} \sim \cN(0,  \ide{d}),$ we have
\eq{
 \mathcal{R}(\cG) \leq 4 \tau C   \left( ( r_a  r_W + r_u ) \sqrt{ \frac{M \log \left( \frac{6d}{M} \right)}{n}} +   \frac{r_a r_b}{\sqrt{n}} \right) 
}
where $n$ is number of samples and $C > 0$ is a universal constant.
\end{proposition}

\begin{proof}
Let $\cF \coloneqq \left \lbrace (\bs{x}, y) \to \inner{\bs{u}}{\bs{x}} +  \inner{\bs{a}}{\phi(\bs{W}^\top \bs{x} + \bs{b})}    ~ \vert ~ (\bs{a},\bs{W},\bs{b}, \bs{u}) \in  \Theta \right \rbrace$.  By Talagrand's contraction principle, we have  $\mathcal{R}(\cG) \leq 2   \tau \mathcal{R}(\cF).$ Hence,  in the following, we will bound  $\mathcal{R}(\cF).$ Indeed, let $(\varepsilon_i)_{i \in [n]}$ be a sequence of i.i.d Radamacher random variables.   Then, we have
\eq{
\mathcal{R}(\cF)  & = \E \left[ \sup_{(\bs{a},\bs{W},\bs{b}, \bs{u})  } \frac{1}{n} \sum_{i = 1 }^n \varepsilon_i \left(  \inner{\bs{u}}{\bs{x}_i} + \inner{\bs{a}}{\phi(\bs{W}^\top \bs{x}_i + \bs{b})} \right)  \right]  \\
& \leq  \E \left[ \sup_{(\bs{a},\bs{W},\bs{b})  } \frac{1}{n} \sum_{i = 1 }^n \varepsilon_i    \inner{\bs{a}}{\phi(\bs{W}^\top \bs{x}_i + \bs{b})}  \right]  + \E \left[ \sup_{ \bs{u}  } \frac{1}{n} \sum_{i = 1 }^n \varepsilon_i \inner{\bs{u}}{\bs{x}_i}   \right] \\
& \leq   \E \left[ \sup_{(\bs{a},\bs{W},\bs{b})  } \frac{1}{n} \sum_{i = 1 }^n \varepsilon_i    \inner{\bs{a}}{\phi(\bs{W}^\top \bs{x}_i + \bs{b})}  \right]  + C r_u \sqrt{ \frac{M \log \left( \frac{6d}{M} \right)}{n}}  \label{radamacher:eq0}
}
where we use \cite[Exercise 10.3.8]{Vershynin2018HighDimensionalP} in the last line.  To bound the first term, we have
\eq{
\E \left[ \sup_{(\bs{a},\bs{W},\bs{b})  } \frac{1}{n} \sum_{i = 1 }^n \varepsilon_i    \inner{\bs{a}}{\phi(\bs{W}^\top \bs{x}_i + \bs{b})}  \right]
& \leq \frac{r_a}{\sqrt{m}}  \E \left[  \sup_{(\bs{a},\bs{W},\bs{b})  }  \norm*{ \frac{1}{n} \sum_{i = 1 }^n \varepsilon_i    \phi(\bs{W}^\top \bs{x}_i + \bs{b}) }_2  \right]  \\
& \leq  r_a  \E \left[ \sup_{(\bs{a},\bs{W},\bs{b})  }  \norm*{ \frac{1}{n} \sum_{i = 1 }^n \varepsilon_i    \phi(\bs{W}^\top \bs{x}_i + \bs{b}) }_\infty \right] \\
& \leq  2 r_a   \E \Big[ \sup_{  \substack{ \norm{\bs{w}}_2 \leq r_W \\  \norm{\bs{w}}_0 \leq M \\ \abs{b} \leq r_b  }  }    \abs*{ \frac{1}{n} \sum_{i = 1 }^n \varepsilon_i  ( \inner{\bs{w}}{\bs{x}_i} +  b) } \Big]  \label{radamacher:eq1}
} 
where we use Cauchy Schwartz inequality in the first line, and the contraction lemma in the last line (note that $\phi(0)=0$ and it is $1$-Lipschitz). Then, since the set we take supremum over is symmetric, we have
\eq{
\eqref{radamacher:eq1} {\small = 2 r_a   \E \Big[ \sup_{  \substack{ \norm{\bs{w}}_2 \leq r_W \\  \norm{\bs{w}}_0 \leq M \\ \abs{b} \leq r_b  }  }      \frac{1}{n} \sum_{i = 1 }^n \varepsilon_i  ( \inner{\bs{w}}{\bs{x}_i} +b)  \Big]   } & \leq {\small 2  r_a  r_W  \E \Big[ \sup_{ \substack{ \norm{\bs{w}}_2 \leq 1 \\  \norm{\bs{w}}_0 \leq M} }      \inner{\bs{w}}{  \frac{1}{n} \sum_{i = 1 }^n \varepsilon_i  \bs{x}_i} \Big]  + 2 r_a r_b \E \left[\abs*{ \frac{1}{n} \sum_{i = 1 }^n \varepsilon_i  } \right]} \\
& \leq    2 C  r_a  r_W   \sqrt{ \frac{M \log \left( \frac{6d}{M} \right)}{n}}  + 2  r _a r_b \frac{1}{\sqrt{n}}  \label{radamacher:eq2}
}
where we use \cite[Exercise 10.3.8]{Vershynin2018HighDimensionalP}  in the last line.   By \eqref{radamacher:eq0} and   \eqref{radamacher:eq2}, the statement follows. 
\end{proof}

\begin{lemma}
\label{lem:relunetmoments}
For  fixed $(\bs{a},\bs{W},\bs{b}) \in \R^m \times \R^{d \times m} \times \R^m$,  let $\hat y (\bs{x}; (\bs{a},\bs{W},\bs{b})) \coloneqq \bs{a}^\top \phi(\bs{W}^\top \bs{x} + \bs{b})$.  For $x \sim \cN(0,\ide{d})$, we have the following:
\begin{enumerate}
\item $\E_{\bs{x}} [\hat y(\bs{x}; (\bs{a},\bs{W},\bs{b}) )^2] \leq \norm{\bs{a}}_2^2 \left(  \norm{\bs{b}}_2^2 + \norm{\bs{W}}_F^2 \right)$
\item $\E_{\bs{x}} [\hat y (\bs{x}; (\bs{a},\bs{W},\bs{b}) )^4] \leq \norm{\bs{a}}_2^4 m \sum_{j = 1}^m \left( 3 \norm{\bs{W}_{j*}}_2^4 + 6\norm{\bs{W}_{j*}}_2^2 \bs{b}_j^2 + \bs{b}_j^4  \right)$
\end{enumerate}
\end{lemma}

\begin{proof}
For the first item,   by using Cauchy Schwartz inequality and that $\phi(t) \leq \abs{t}$,  we have
\eq{
\E[\hat y(\bs{x}; (\bs{a},\bs{W},\bs{b}) )^2] = \E \left[\inner{\bs{a}}{\phi(\bs{W}^\top \bs{x} + \bs{b})}^2 \right] &\leq \norm{\bs{a}}_2^2 \E \left[\norm{\bs{W}^\top \bs{x} + \bs{b}}^2_2  \right] =   \norm{\bs{a}}_2^2 \left(  \norm{\bs{b}}_2^2 + \norm{\bs{W}}_F^2 \right).
}
For the second item,  by using the same arguments, 
\eq{
\E[\hat y (\bs{x}; (\bs{a},\bs{W},\bs{b}) )^2] =  \norm{\bs{a}}_2^4 \E \left[\norm{\bs{W}^\top \bs{x} + \bs{b}}^4_2  \right] &  \labelrel\leq{rel:ineqq0}   \norm{\bs{a}}_2^4 m \sum_{j = 1}^m \E \left[ (\inner{\bs{W}_{j*}}{\bs{x}} + \bs{b}_j)^4 \right]  \\
&=    \norm{\bs{a}}_2^4 m \sum_{j = 1}^m  \left( 3 \norm{\bs{W}_{j*}}_2^4 + 6\norm{\bs{W}_{j*}}_2^2 \bs{b}_j^2 + \bs{b}_j^4  \right)
}
where we use $\norm{\bs{v}}_4 \leq m^{1/4} \norm{\bs{v}}_2$ for $\bs{v} \in \R^m$ for \eqref{rel:ineqq0}. 
\end{proof}

\begin{lemma}
\label{lem:relunettailbound}
For fixed $(\bs{a},\bs{W},\bs{b})  \in \R^m \times \R^{d \times m} \times \R^m$,   and $\bs{u} \in \R^d$,  let $\hat y(\bs{x}; (\bs{a},\bs{W},\bs{b}) ) \coloneqq \bs{a}^\top \phi(\bs{W}^\top \bs{x} + \bs{b}) + \bs{u}^\top \bs{x}$.     For $\bs{x} \sim \cN(0,  \ide{d})$,  we have with probability at least $1 - \delta$,
\eq{
 \abs{ \hat y(\bs{x}; (\bs{a},\bs{W},\bs{b}) ) - \gt(\bs{V}^\top \bs{x}) }&  \leq  \norm{\bs{a}}_2 \sqrt{\norm{\bs{b}}_2^2 + \norm{\bs{W}}_F^2}  + ( \norm{\bs{a}}_2 \norm{\bs{W}}_F + \norm{\bs{u}}_2 ) \sqrt{2 \log(4/\delta)}  \\
&  + C_1 (r+ 2) (2e)^{C_2} \log^{C_2} (6 / \delta).
}
\end{lemma}

\begin{proof}
We first observe that
\eq{
 \abs{ \hat y(\bs{x}; (\bs{a},\bs{W},\bs{b}) ) - \gt(\bs{V}^\top \bs{x}) } & =  \abs{ \hat y(\bs{x}; (\bs{a},\bs{W},\bs{b}) )  - \E[\hat y(\bs{x}; (\bs{a},\bs{W},\bs{b}) ) ]} \\
 & + \abs{\E[\hat y(\bs{x}; (\bs{a},\bs{W},\bs{b}) ) ]} + \abs{\gt(\bs{V}^\top \bs{x}) } \\
&  \leq  \abs{ \hat y(\bs{x}; (\bs{a},\bs{W},\bs{b}) )  - \E[\hat y(x; (\bs{a},\bs{W},\bs{b}) ) ]} +   \abs{\gt(\bs{V}^\top \bs{x}) } \\
&+   \norm{\bs{a}}_2 \left(  \norm{\bs{b}}_2^2 + \norm{\bs{W}}_F^2 \right)^{1/2}.  \label{reluconc:eq0}
 }
Moreover,  since $\phi$ is $1$-Lipschitz that $\bs{x} \to  \hat y(\bs{x}; (\bs{a},\bs{W},\bs{b}) )$ is $\norm{\bs{a}}_2 \norm{\bs{W}}_F + \norm{\bs{u}}_2$ - Lipschitz.  Then, by using Gaussian Lipschitz concentration inequality (see \cite[Theorem 5.2.2]{Vershynin2018HighDimensionalP}) and Corollary \ref{cor:concg},  we obtain the statement.
\end{proof}

%

%
%

\end{document}